\documentclass[11pt]{article}
	
	\newcommand{\blind}{0}
	
    \newcommand{\Real}{\mathbb R}
    
    \newcommand{\NatInt}{\mathbb N}

    \newcommand{\CalO}{\mathcal O}
    \newcommand{\CalG}{\mathcal G}

    \newcommand{\CalN}{\mathcal N}
    \newcommand{\CalT}{\mathcal T}
    \newcommand{\BX}{\bold X}
    \newcommand{\Bx}{\bold x}
    \newcommand{\By}{\bold y}
    \newcommand{\BY}{\bold Y}
    \newcommand{\BZ}{\bold Z}
    
    \newcommand{\Bm}{\bold m}
    \newcommand{\Bf}{\bold f}
    \newcommand{\Bi}{\bold i}
    \newcommand{\BU}{\bold U}
    \newcommand{\BD}{\bold D}
     \newcommand{\Bu}{\bold u}
    \newcommand{\Bl}{\bold l}

    \newcommand{\E}{\mathbb E}
     \newcommand{\polylog}{\text{polylog}}
    \newcommand{\prob}{\mathbb P}
    \newcommand{\Btheta}{\boldsymbol{\theta}}
     \newcommand{\BTheta}{\boldsymbol{\Theta}}

    \newcommand{\Bphi}{\boldsymbol{\phi}}
    \newcommand{\Brho}{\boldsymbol{\rho}}
    \newcommand{\Bmu}{\boldsymbol{\mu}}
    \newcommand{\BSigma}{\boldsymbol{\Sigma}}

    \makeatletter
    \renewcommand\section{\@startsection {section}{1}{\z@}%
                                       {-3.5ex \@plus -1ex \@minus -.2ex}%
                                       {2.3ex \@plus.2ex}%
                                       {\normalfont\fontfamily{phv}\fontsize{16}{19}\bfseries}}
    \renewcommand\subsection{\@startsection{subsection}{2}{\z@}%
                                         {-3.25ex\@plus -1ex \@minus -.2ex}%
                                         {1.5ex \@plus .2ex}%
                                         {\normalfont\fontfamily{phv}\fontsize{14}{17}\bfseries}}
    \renewcommand\subsubsection{\@startsection{subsubsection}{3}{\z@}%
                                        {-3.25ex\@plus -1ex \@minus -.2ex}%
                                         {1.5ex \@plus .2ex}%
                                         {\normalfont\normalsize\fontfamily{phv}\fontsize{14}{17}\selectfont}}
    \makeatother
	
	\usepackage{amsmath}
	\usepackage{graphicx}
	\usepackage{enumerate}
	\usepackage{natbib} 
	\usepackage{url} 
	
	\usepackage{hyperref}       
    \usepackage{url}            
    \usepackage{booktabs}       
    \usepackage{amsfonts}       
    \usepackage{nicefrac}       
    \usepackage{microtype}      
    \usepackage{graphicx}
    \usepackage[space]{grffile}
    \usepackage{latexsym}
    \usepackage{textcomp}
    \usepackage{longtable}
    \usepackage{tabulary}
    \usepackage{booktabs,array,multirow}
    \usepackage{amsfonts,amsmath,amssymb}
    \usepackage{url}
    \usepackage{amsmath,amsfonts,amssymb}
    \usepackage{mathtools}
    \usepackage{bm}
    \usepackage[mathscr]{euscript}
    \usepackage{enumitem}
    \usepackage{xcolor}
    \usepackage{graphicx}
    \usepackage{multirow}
    \usepackage{booktabs}
    \usepackage{microtype}
    \usepackage{fix-cm}
    \usepackage{colortbl}
    \usepackage{tikz}
    \usepackage{pbox}
    \usepackage{xcolor}
    \usepackage{etoolbox}
    \usepackage{appendix}
    \usepackage{scalerel}
    \usepackage[ruled,vlined,linesnumbered]{algorithm2e}
    \usepackage[english]{babel}
    \usepackage{amsthm}
    \usepackage{gensymb}

    \newtheorem{theorem}{Theorem}
    \newtheorem{example}{Example}
    \newtheorem{lemma}{Lemma}

    \newtheorem{definition}{Definition}

\begin{document}
		
		\def\spacingset#1{\renewcommand{\baselinestretch}%
			{#1}\small\normalsize} \spacingset{1}
		
		\if0\blind
		{
			\title{\bf A Sparse Expansion For Deep Gaussian Processes }
			\author{Liang Ding$^{a}$, Rui Tuo$^{b}$ and Shahin Shahrampour$^c$ \\
                $^a$ School of Data Science, Fudan University, Shanghai, China\\
			$^b$ Industrial \& Systems Engineering, Texas A\&M University, College Station, TX \\
             $^c$ Mechanical \& Industrial Engineering, Northeastern University, Boston, MA }
			\date{}
			\maketitle
		} \fi
		
		\if1\blind
		{

            \title{\bf A Sparse Expansion For Deep Gaussian Processes}
			\author{Author information is purposely removed for double-blind review}
			\date{}
			\maketitle

		} \fi
		\bigskip
		
	\begin{abstract}
{In this work, we use Deep Gaussian Processes (DGPs) as statistical surrogates for stochastic processes with complex distributions}. Conventional inferential methods for DGP models can suffer from high computational complexity as they require large-scale operations with kernel matrices for training and inference. In this work, we propose an efficient scheme for accurate inference and efficient training based on a range of Gaussian Processes, called the \emph{Tensor Markov Gaussian Processes} (TMGP). We construct an induced approximation  of TMGP referred to as the \emph{hierarchical expansion}.  Next, we develop a deep TMGP (DTMGP) model as the composition of multiple hierarchical expansion of TMGPs. The proposed DTMGP model has the following properties: (1) the outputs of each activation function are deterministic while the weights are chosen independently from standard Gaussian distribution; (2) in training or prediction, only $\mathcal{O}({\polylog}(M))$ (out of $M$) activation functions have non-zero outputs, which significantly boosts the computational efficiency. Our numerical experiments on synthetic models and real datasets show the superior computational efficiency of DTMGP over existing DGP models.
	\end{abstract}
			
	\noindent%

	\spacingset{1.5} 

\section{Introduction} \label{sec:intro}
{
This work is partially motivated by surrogate modeling of stochastic computer simulations \citep{AnkenmanNelsonStaum10,plumlee2014building}. The ultimate goal in this field is to approximate an underlying stochastic process, defined by a complex stochastic computer simulation code, using a surrogate stochastic process that is computationally efficient.
Among existing statistical surrogates, Deep Gaussian Processes (DGP) proposed by \cite{Damianou13} has recently become a popular choice owing to its exceptional performance and flexibility in fitting the complex distributions from the stochastic processes in real applications \citep{radaideh2020surrogate,sauer2020active}.}

DGPs are multi-layer  compositions of multi-variate Gaussian Processes (GP). The deep probabilistic structures of DGPs allows a Bayesian formulation to model complex stochastic systems, such as  computer vision, natural language  processing, etc..  However, the cost of DPGs' expressiveness and flexibility is the extreme difficulties in training and inference, because the propagation of randomness throughout the layers of DGPs is nonlinear and correlated among layers. Existing approaches for overcoming these difficulties can be separated into two classes -- one class focuses on approximating the Bayesian formulations of DGPs, such as variational inference (VI) \citep{Blei_2017}, expectation propagation (EP) \citep{Bui16}, or Vecchia approximation \citep{VecchiaDGP}, while the other class considers simplified representations of DGPs' structures, such as inducing-variable approximation \citep{hensman2014nested} and random Fourier feature expansion \citep{Cutajar17}. 

Simplifications based on inducing variables and random Fourier features \citep{rahimi2008random} partially address the computational issue of DGPs. However, they are still hard to implement. This is  because these simplifications of DGPs are deep Neural Networks (DNN) with kernel based activations, correlated random weights and bias parameters, which is  equivalent to Bayesian Neural Networks (BNN) \citep{mackay1992practical,neal1996bayesian} with dense and highly correlated structures.  Training and inference of these simplifications are at least as hard as existing BNN models.

 In this work, we { focus on introducing an accurate and efficient simplification of DGPs.} We propose a sparse reduced-rank approximation, referred as \emph{Hierarchical Expansion}, for DGPs. Our expansion is specialized for DPGs which are the compositions of multi-variate tensor Markov GPs (TMGP) \citep{ding2020sample}, so we call this expanded DGP deep TMGP (DTMGP). One of our main contributions lies in constructing a sparse representation of DGPs. There are only a poly-logarithmic number of activations in our model with non-zero outputs each time we run the model. Because of this sparse property, training of DTMGPs are much faster and easier compared with existing DGP models, avoiding commonly seen numerical issues in training deep models. Another contribution is to show that hierarchical expansion is also highly accurate. Unlike many other GPs/DGPs sparse approximations, 
 the difference between sample paths generated by any DGP and those generated by its hierarchical expansion is small. This property allows DTMGPs to model complex stochastic systems with satisfactory performance as shown in our numerical experiments.

{ We use VI to train DTMGP on data generated from simulations and real data and compare it with other existing DGP models. We conduct simulation studies by running experiments on an artificial random field and simulator of a stochastic activity network \citep{Simopt}. We also use the  RC-49 data set, which is publicly available on the web \url{https://paperswithcode.com/dataset/rc-49}, to demonstrate the performance of the proposed approach. In these numerical studies, we find that the training process of DTMGP is more stable, and the training loss converges faster compared with its competitors. Besides, we compare the similarities between instances from the underlying systems and instances generated from competing statistical surrogate models. The results show that DTMGP outperforms the alternative DGPs in all experiments.}

The remainder of this article is organized as follows. We will review the related literature in Section \ref{sec:literature_review}. We introduce general DGP models and an approximation of DGPs called induced approximation in Section  \ref{sec:DGP-model}. The methodology and detailed implementation of DTMGP are introduced in Section  \ref{sec:methodology}. Simulation studies are given in Section \ref{sec:numerical} and experiments on real data are given in \ref{sec:RC-49} . Concluding remarks are made in Section \ref{sec:conclusion}. The Appendix includes the required mathematical tools and technical proofs.

\section{Literature Review} \label{sec:literature_review}
The state-of-the-art deep learning techniques have brought probabilistic modeling with deep neural network structure in popularity. Based on the concept of deep belief network (DBN) \citep{Hinton06},  \cite{Damianou13}  generalized the Restricted Boltzmann Machine \citep{Hinton10}, which is a DBN with binary output, to deep Gaussian Processes based on Gaussian process mapping. DGPs have been a popular tool in several applications. Their prowess has been demonstrated on many classification tasks \citep{Damianou13,fei2018active,Yang21}. Compared with traditional DNNs,   the flexibility in uncertainty quantification of DGPs makes them ideal candidate for surrogate modeling \citep{radaideh2020surrogate,sauer2020active}. DGPs are commonly used statistical surrogates in many applications such as Bayesian optimization \citep{hebbal2021bayesian}, calibration \citep{marmin2022deep}, multi-fidelity analysis \citep{ko2021deep}, healthcare \citep{li2021deep}, and etc.

However, training and inference of DGPs are difficult. Current attempts can be separated into two classes. One focuses on designing more efficient and accurate training algorithm while the other one on constructing DGP architectures with sparsity.  Efficient training and inference algorithms  include expectation propagation \citep{Bui16},  doubly stochastic variational inference \citep{Salimbeni17}, stochastic gradient Hamiltonian Monte Carlo \citep{Havasi18}, elliptical slice sampling \citep{sauer2020active}, Vecchia approximation \citep{vecchia1988estimation, VecchiaDGP,katzfuss2021general}, and etc. Reformulation of DGPs to simplified models is an alternative approach. DGPs are reformulated as an variational model in \cite{Tran16}.  Low-rank approximations for GPs in \cite{banerjee2008gaussian,cressie2008fixed} are also extended to DGPs.   \cite{hensman2014nested,dai2016variational} extend the  inducing-variable method to DGPs. \cite{Cutajar17} uses random Fourier feature expansion to show that DGPs are equivalent to infinitely wide BNN with corresponding activations. 

In this work, we consider DGPs with Markov structure and, based on the Markov structure, we design an accurate and efficient sparse expansion for DGPs. \cite{siden2020deep} proposes compositions of \emph{discrete} Gaussian Markov random field on graph and called deep graphical models of this structure \emph{Deep Gaussian Markov Random Fields} (DGMRF). We must point out that our DTMGPs is essentially different from DGMRF because DTMGPs are \emph{not} deep graphical models. More specifically, { DGMRF treats every activation as a random variable and imposes a graphical Markov structure on these activations, i.e., any two non-adjacent randomly distributed activations are conditionally independent given all other activations. On the other hand, every activation in DTMGP is  one of the orthonormal basis functions of a \emph{continuous} Gaussian Markov random field and these orthonormal basis functions constitute a so-called \emph{hierarchical expansion} of  the field. }

\section{ General DGP Models} \label{sec:DGP-model}
\subsection{Deep Gaussian Processes}\label{sec:DGP}
A GP $\CalG$ is a random function characterized by its mean function $\mu$ and covariance function $k$. To be more specific, given any input $\Bx$, the output $\CalG(\Bx)$ has a Gaussian distribution with mean $\mu(\Bx)$ and variance $k(\Bx,\Bx)$, and given any pair of inputs $\Bx$ and $\Bx'$, the covariance between  $\CalG(\Bx)$ and $\CalG(\Bx')$ is $k(\Bx,\Bx')$:
\begin{equation}
\label{eq:GP}
    \begin{aligned}
   & \CalG(\Bx) \sim \CalN \big(\mu(\Bx), k(\Bx,\Bx)\big),\\
    &\textbf{Cov}\big(\CalG(\Bx),\CalG(\Bx')\big)=k(\Bx,\Bx').
\end{aligned}
\end{equation}
A $W$-variate GP is simply a $W$-vector of GPs $[\CalG_1,\CalG_2,\cdots,\CalG_W]$. Without loss of generality, we assume that any multivariate GP in the following content  has mean function $\mu=0$ and independently distributed entries. A $H$-layer DGP $\Bf^{(H)}$ is then the composition of $H$ multi-variate GPs:
\begin{equation}
    \label{eq:DGP}
    \Bf^{(H)}(\Bx^*)=\CalG^{(H)}\circ\cdots\circ \CalG^{(2)}\circ\CalG^{(1)}(\Bx^*),
\end{equation}
where $f\circ g$ denotes the composition function $f(g(\cdot))$, $\CalG^{(h)}$ denotes a $W^{(h)}$-variate GP $[\CalG_1^{(h)},\CalG_2^{(h)},\cdots,\CalG^{(h)}_{W^{(h)}}]$, and $W^{(h)}$ is called the width of layer $h$.	The formulation \eqref{eq:DGP} of DGP can be viewed as a DNN with random activations $\{\CalG_i^{(h)}\}$.

While DGP yields a non-parametric and Bayesian formulation of DNNs, computations for inference/prediction of DGPs are cumbersome. For example,   given a set of input-output pairs  $(\BX,\BY)\in\Real^{n\times d}\times \Real^n$, the conditional probability distribution $\prob(\BY\big|\BX)$ induced by \eqref{eq:DGP} is 
\begin{equation}
\label{eq:DGP-integral}
    \prob \bigl(\bold{Y}\big|\BX\bigr)=\int \prob\bigl(\bold{Y}\big|\Bf^{(H)}\bigr)\prob\bigl(\Bf^{(H-1)}\big|\Bf^{(H-2)}\bigr)\cdots\prob\bigl(\Bf^{(1)}\big|\BX\bigr)d{\Bf^{(1)}}\cdots d{\Bf^{(H)}},
\end{equation}
where $\Bf^{(h)}$ denotes the output vector from layer $h$:
\[\Bf^{(h)}=\CalG^{(h)}\circ\CalG^{(h-1)}\circ\cdots\circ\CalG^{(1)},\]
with $h=1,\cdots,H$. Evidently, the integral \eqref{eq:DGP-integral} is intractable, leading to the intractability of inference and predictions of DGP $\Bf^{(H)}$ conditioned on any observed data $\bigl(\BX,\BY\bigr)$.

\subsection{Induced Approximation of DGPs}\label{sec:inducing-variable}
The inference of GP, which acts as a single activation in DGP, is also time and space consuming. A challenge for GPs  lies in their computational complexity and storage for the inverse of covariance matrix $k(\BX,\BX)$ which are $\CalO(n^3)$ and $\CalO(n^2)$, respectively, when the covariance matrix is of size $n$. To alleviate this computational problem, current methods, such as inducing-variable approximation in \cite{hensman2014nested,dai2016variational} and random Fourier feature approximation in \cite{Cutajar17}, focus on approximating the covariance matrix $k(\BX,\BX)$ by some low rank matrices.

Instead of approximating the covariance matrix, we approximate the GP activations in a DGP directly. In the following content, we will construct a reduced-rank approximation called  induced approximation. Induced approximation can bring a more flexible and clear  deep neural network representation of DGPs.

Without loss of generality,  let $\CalG$ defined in \eqref{eq:GP} be a GP with constant mean $\mu$ and covariance function $k$. Inspired by the kriging method \citep{Matheron63,sacks1989designs}, we can approximate GP $\CalG$ defined in \eqref{eq:GP} by the following finite-rank approximation:
\begin{equation}
    \label{eq:finite-rank-GP}
   \hat{\CalG}(\Bx^*)\coloneqq \mu+ k(\Bx^*,\BU)\big[k(\BU,\BU)\big]^{-1}\CalG(\BU)
\end{equation}
where we call $\BU=\{\Bu_i\in\Real^d\}_{i=1}^m$ the inducing points and $\CalG(\BU)$ are the values of $\CalG$ on $\BU$. According to theory in \citep{ritter2000average,WangTuoWu20,ding2020sample}, { the required number of inducing points for achieving the optimal statistical error is much lower than the sample size, provided that the inducing points $\BU$ are well-chosen. Hence, \eqref{eq:finite-rank-GP} can be computed in a highly efficient way.}

 \begin{figure}[t!]
\centering
\includegraphics[width=0.5\textwidth]{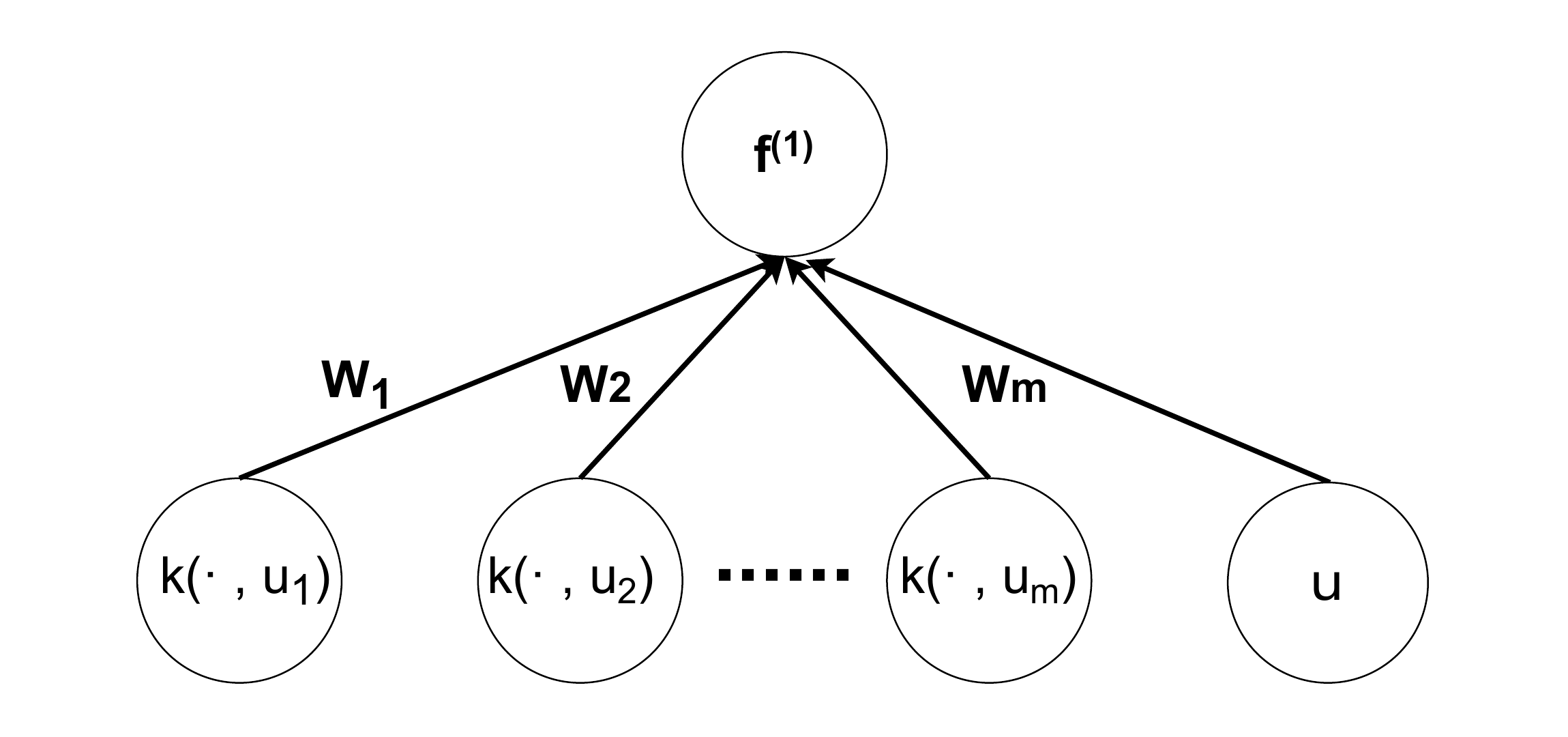}

\caption{Induced approximation $\hat{\CalG}$ can be represented as a one-layer neural network with correlated Gaussian distributed weights $\bold{w}=\big[k(\BU,\BU)\big]^{-1}\CalG(\BU)$ and bias $\mu$. \label{fig:nn-rep-GP}}
\end{figure}

Finite-rank approximation \eqref{eq:finite-rank-GP} can be represented as a one-layer neural network with correlated Gaussian distributed weights $\bold{w}=\big[k(\BU,\BU)\big]^{-1}\CalG(\BU)$ and bias $\mu$ as shown in Figure \ref{fig:nn-rep-GP}. However, the correlated Gaussian distributed weights still make training and inference difficult. In order to write the finite-rank approximation \eqref{eq:finite-rank-GP} in the form of a neural network with independently distributed weights, we can simply apply Cholesky decomposition on kernel matrix $k(\BU,\BU)$, which leads to the following equation:
\begin{equation}
\label{eq:chole-GP}
    \hat{\CalG}(\Bx^*)= \mu+k(\Bx^*,\BU)R^{-1}_{\BU}\BZ\coloneqq\mu+\Bphi^T(\Bx^*)\BZ, 
\end{equation}
where $R_\BU$ is the Cholesky decomposition of matrix $k(\BU,\BU)$ and $\BZ=[R_{\BU}^{T}]^{-1}\CalG(\BU)\in\Real^m$ are i.i.d. distributed standard Gaussian random variables. We call approximation \eqref{eq:chole-GP} induced approximation and Figure \ref{fig:nn-rep-CholGP} shows that induced approximation can be  represented  as a two-layer neural network where functions $\Bphi$ act as an extra hidden layer.
\begin{figure}[t!]
\centering
\includegraphics[width=0.5\textwidth]{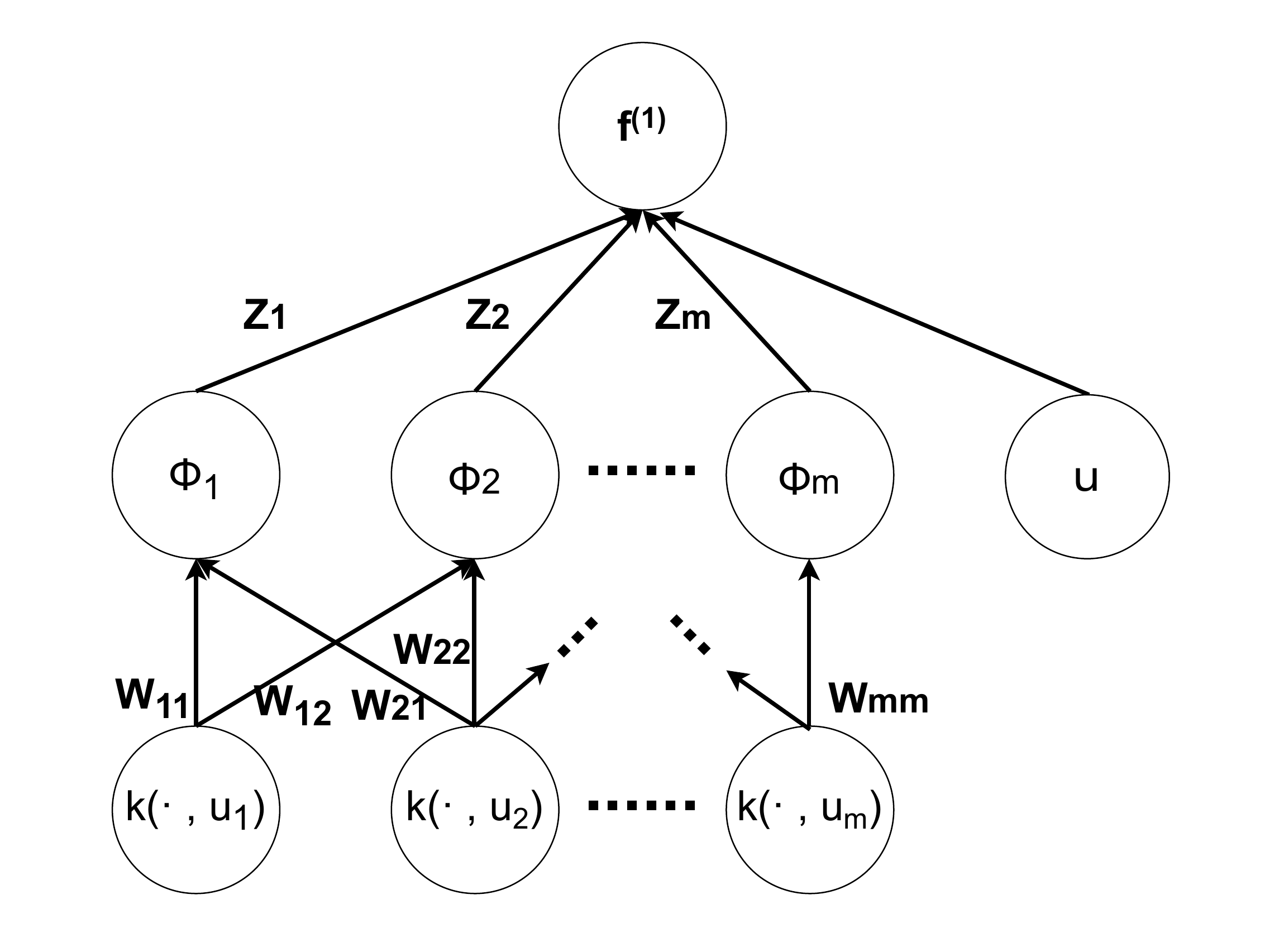}

\caption{ A two-layer neural network representation of $\hat{\CalG}$ with deterministic  weights $\bold{W}=\big[R_{\BU}^T\big]^{-1}$, i.i.d. standard Gaussian weights $\BZ$, and bias $\mu$. \label{fig:nn-rep-CholGP}}
\end{figure}
Because the approximation \eqref{eq:chole-GP} consists of only $m$ Gaussians, the  time and space complexity for its inference are then reduced to $\CalO(m^2n)$ and $\CalO(mn)$, respectively.

The induced approximation of a DGP is to approximate each random activation by corresponding induced approximation. For example, suppose that in a DGP $\Bf^{(H)}$, all activations in the same layer are characterized by the same covariance function $k^{(h)}$. Then, the induced approximation  $\Tilde{\Bf}^{(H)}$ of  $\Bf^{(H)}$ is
\begin{equation}
\label{eq:inducing-var-DGP}
    \begin{aligned}
        &\Tilde{\Bf}^{(H)}(\Bx^*)=\big[\BZ^{(H)}\Bphi^{(H)}(\cdot)+\Bmu^{(H)}\big]\circ\cdots\circ\big[\BZ^{(1)}\Bphi^{(1)}(\Bx^*)+\Bmu^{(1)}\big],\\
        &[\Bphi^{(h)}(\cdot)]^T=k^{(h)}(\cdot,\BU^{(h)})R^{-1}_{{h}}\quad \text{for}\ h=1,\cdots,H
    \end{aligned}
\end{equation}
where, for $h=1,\cdots,H$, $\BZ^{(h)}\in\Real^{W^{(h)}\times m^{(h)}}$ is a  matrix with i.i.d. standard Gaussian entries, $U^{(h)}$ are the total $m^{(h)}$ inducing variables for kernel $k^{(h)}$, $R_h\in\Real^{m^{(h)}\times m^{(h)}}$ is the Cholesky decomposition of the  covariance matrix $k^{(h)}(\BU^{(h)},\BU^{(h)})$, and $\Bmu^{(h)}\in\Real^{W^{(h)}}$ is the mean vector and can be treated as the bias of layer $h$. 

In the next section, we will show that if kernel $k$ is in a class of kernel function called tensor Markov kernel and inducing points $\BU$ are a specific design called sparse grid, then induced approximations \eqref{eq:chole-GP} and \eqref{eq:inducing-var-DGP}  can be further written as a sparse approximation.

\section{Methodology of DTMGP}\label{sec:methodology}
In this section, we first introduce the concept of TMGPs. We then introduce hierarchical expansion of TMGPs and how it leads to mutually orthogonal feature functions with hierarchical supports. Based on hierarchical expansion, we can introduce the implementation and training of DTMGPs.
\subsection{Tensor Markov GPs}
Hierarchical expansion is applied to a class of GPs called tensor Markov GP. In one dimension, a Markov GP is characterized by the following  Lemma in \cite{MarcusRosen06}:
\begin{lemma}[\cite{MarcusRosen06} Lemma 5.1.8]\label{lemma:PD}
Let $I\subseteq \Real$.
Let $\CalG$ be a zero mean GP defined on $I$ with continuous positive definite kernel $k$.
Then, $\CalG$ is a Markov GP and $k$ is Markov kernel if and only if there exist positive functions $p$ and $q$ on $I$ with $p/q$ strictly increasing such that
\begin{equation}\label{eq:Gaus-Markov}
    k(x,  x') = p(\min(x, x')) q(\max(x, x')),\quad x, x' \in I.
\end{equation}
\end{lemma}
For any Markov GP $\CalG$ we have that, conditioned on its  distribution at a point $x$, its distributions at any point $x_l<x$ and any point $x_u>x$ are independent: $\CalG(x_l)\perp \CalG(x_u)\big|\CalG(x)$, whence Markov GP.  Markov GPs  have two advantages in computations. Firstly, given any ordered inducing points $\{u_1<\cdots<u_n\}$, the distribution of $\CalG(x)$ at any  $x$, where $u_i<x<u_{i+1}$,  depends only on its left and right inducing points $\CalG(u_i)$ and $\CalG(u_{i+1})$. Secondly, for any induced approximation,  accuracy of the approximation at $x$ only depends on the distances between $x$ and its left and right neighbors, respectively. Therefore,  evenly distributed inducing variables can achieve an accurate approximation that is also cheap to compute.   

The following purely \textit{additive} and purely \textit{tensor} structure of one-dimensional Markov kernel can both extend Markov GPs to multi-dimension:
\begin{equation}
    \label{eq:multi-dim-Markov-kernel}
    k(\Bx,\Bx')=\sum_{j=1}^dk_j(x_j,x_j'),\quad k(\Bx,\Bx')=\prod_{j=1}^dk_j(x_j,x_j').
\end{equation}

We generalize the additive and tensor form in \eqref{eq:multi-dim-Markov-kernel} to the following GP, which has additive and tensor structure simultaneously:
\begin{definition}\label{def:TMGP}
A GP is called TMGP if and only if it is a zero mean GP with covariance function of the form
\begin{equation}
    \label{eq:add-tensor-Markov-kernel}
    k(\Bx,\Bx')=\sum_{l=1}^s\prod_{j\in \boldsymbol{U}_l}k_{l,j}(x_j,x_j'),
\end{equation}
where $\boldsymbol{U}_l\subseteq\{1,2,\cdots,d\}$  and $k_{l,j}$ is 1-D Markov kernel for any $l,j$. A kernel of the form \eqref{eq:tensor-Markov-kernel} is called tensor Markov kernel (TMK).
\end{definition}
A TMGP $\CalG$ with kernel \eqref{eq:add-tensor-Markov-kernel} can be  factorized as sum of $s$ mutually independent GPs: $\CalG=\sum_{l=1}^s\CalG_l$ where each $\CalG_l$ is with kernel function $\prod_{j\in \boldsymbol{U}_l}^dk_{l,j}(x_j,x_j')$. To achieve an accurate approximation of TMGP using the smallest number of inducing variables, we can only focus on TMGP with a purely tensor structure. By combining mutually independent approximations, we can arrive at a solution. Therefore, in the following content, we only study approximation for TMGP with purely tensor structure:
\begin{equation}
    \label{eq:tensor-Markov-kernel}
    k(\Bx,\Bx')=\prod_{j=1}^dk_{j}(x_j,x_j'),
\end{equation}
and the general TMGP with kernel \eqref{eq:add-tensor-Markov-kernel} is a straightforward extension that involves combining approximations of GPs with purely tensor kernels.

Commonly used TMKs with purely tensor structure include Laplace kernel $k(\Bx,\Bx')=\exp(\sum_{j=1}^d\theta_j|x_j-x_j'|)$ and Brownian sheet kernel $k(\Bx,\Bx')=\prod_{j=1}^d(1+\theta_j \min\{x_j,x_j'\})$. The challenge in extending the inducing approximation of 1-D Markov GP to multiple dimensions lies in properly defining the "left" and "right" neighbors in a multidimensional space, similar to how it's done on the real line. To address this, we introduce the hierarchical expansion method, which is explained in the following section.

\subsection{Hierarchical Expansion}
To construct the hierarchical expansion, we let inducing variables $\BU$ be an experimental design $\BX^{\rm SG}_l$ called level-l \emph{Sparse Grid} (SG) \citep{bungartz_griebel_2004, Plumlee14},  where level $l$ determines the number of inducing variables. Furthermore, the SG inducing variables must be sorted in a specific order. Detailed introduction of SGs and the required order are provided in Appendix \ref{sec:sparse-grid}, and MATLAB codes for generating SGs satisfying the requirement can be found in the \emph{Sparse Grid Designs} package \citep{SGDesign}. Examples of two-dimensional SGs are shown in Figure \ref{fig:SparseGrid}. From the examples, we can see that the incremental points from the next level of a SG exhibit a hierarchical structure -- higher level SG consists of local SGs of smaller levels.
\begin{figure}[ht]
\centering 
\includegraphics[width=0.3\textwidth]{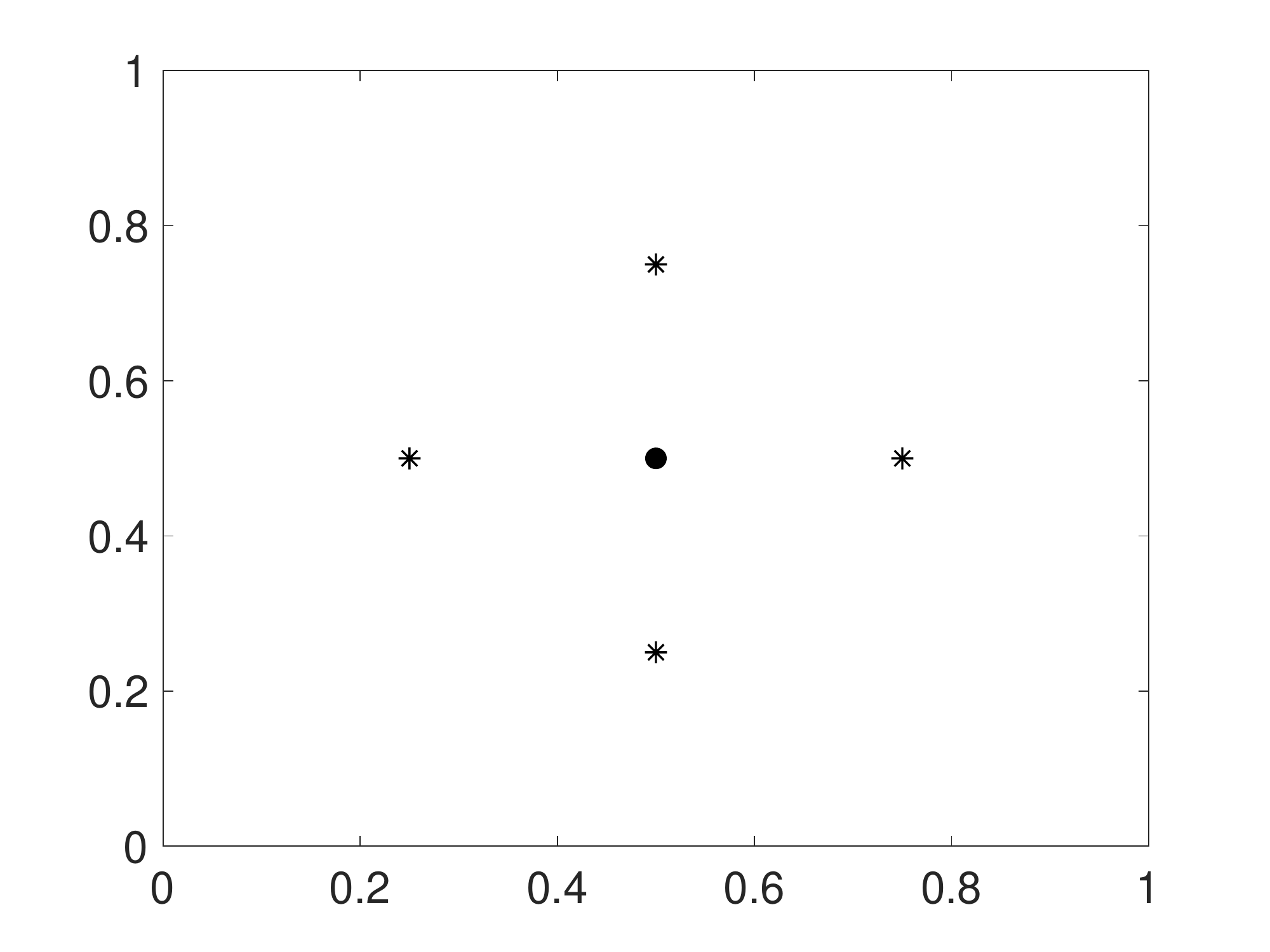} 
\includegraphics[width=0.3\textwidth]{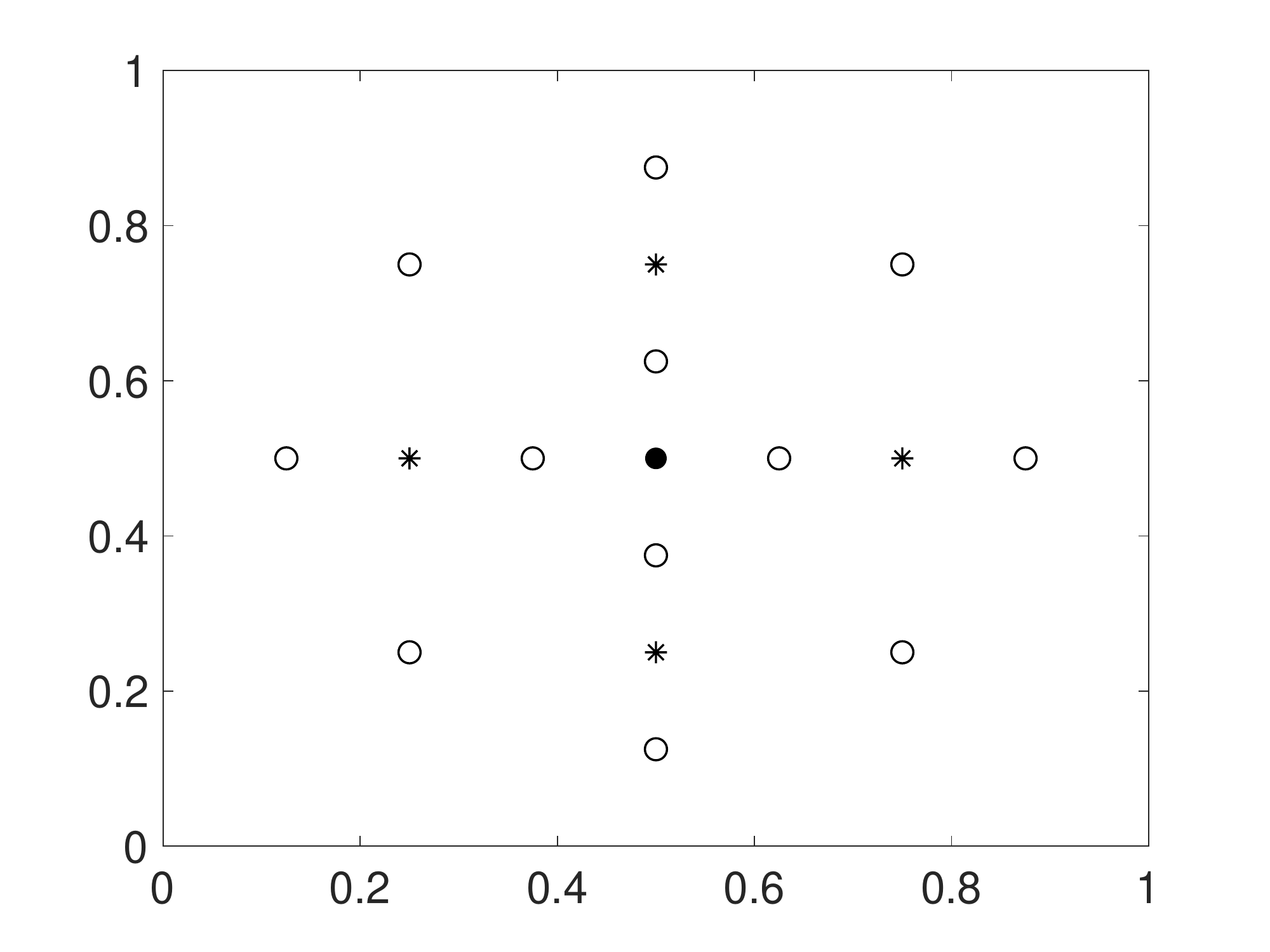} 
\includegraphics[width=0.3\textwidth]{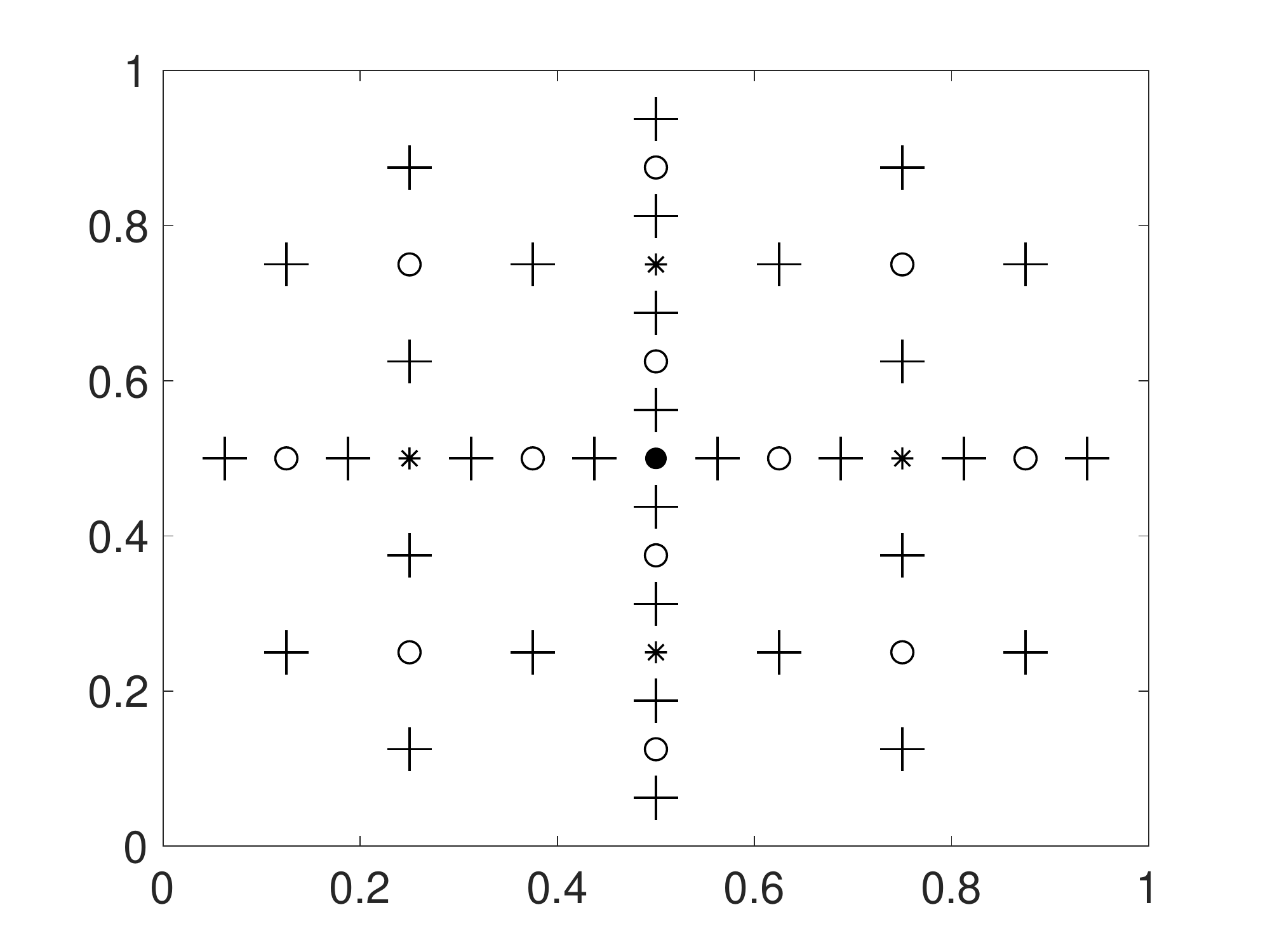}
\caption{\label{fig:SparseGrid}Two-dimensional level-2 SG (left), level-3 SG (middle) and level-4 SG (right) on cube $(0,1)^2$. The incremental points from the second level, third level and fourth level  are labeled by *, o and +, respectively.}
\end{figure}

Hierarchical expansion of a TMGP $\CalG$ with inducing points $\BX^{\rm SG}_l$ is then simply:
\begin{equation}
    \label{eq:hierarchical-expansion}
    \CalG(\Bx^*)\approx k(\Bx^*,\BX^{\rm SG}_l) R_l^{-1}\BZ\coloneqq  \Bphi^T(\Bx^*)\BZ,
\end{equation}
where $k(\Bx^*,\BX^{\rm SG}_l)=[k(\Bx^*,\Bx_1),\cdots,k(\Bx^*,\Bx_{m_l})]$ is the covariance vector, $k(\BX^{\rm SG}_l,\BX^{\rm SG}_l)=[k(\Bx_i,\Bx_j)]_{i,j\leq m_l}$ is the covariance matrix, $R_l$ is the Cholesky decomposition of  $k(\BX^{\rm SG}_l,\BX^{\rm SG}_l)$, $\BZ=[Z_1,\cdots,Z_{m_l}]$ are i.i.d. standard Gaussian random variables and $\Bphi=[\phi_i]_{i=1}^{m_l}$ are called hierarchical features. Hierarchical expansion \eqref{eq:hierarchical-expansion} yields a sparse approximation of TMGP that is easy to compute as shown in the following theorems.
\begin{theorem}
\label{thm:nnz-cholesky}
 The number of non-zero entries on $R_l^{-1}$ is $\CalO(m_l )$ and $R^{-1}_l$ can be computed in  $\CalO(m_l )$ operations. 
\end{theorem}
\begin{theorem}
\label{thm:nnz-hierarchical-feature}
 Given any input point $\Bx^*$, the number of non-zero entries on the vector of hierarchical features $\Bphi(\Bx^*)$ is $\CalO([\log m_l]^{2d-1})$.
\end{theorem}
The algorithm for computing $R^{-1}_l$ in $\CalO(m_l )$ operations is given in Appendix \ref{sec:algorithm}, and proofs for Theorems \ref{thm:nnz-cholesky} and \ref{thm:nnz-hierarchical-feature} are provided in supplementary material. Since we can compute $R^{-1}_l$ in $\CalO(m_l )$ operations and the numbers of non-zero entries on $R^{-1}_l$ is $\CalO(m_l)$,  hierarchical expansion can be computed in $\CalO(m_l )$ operations. Moreover, if matrix $R^{-1}_l$ is given, because the numbers of non-zero entries on $\Bphi(\Bx^*)$ is only $\CalO([\log m_l]^{2d-1})$, the computational time of the hierarchical expansion can be further reduced to $\CalO([\log m_l]^{2d-1})$. The sparsity of hierarchical feature vector $\Bphi$ relies on the fact that the supports of $\{\phi_i\}_{i=1}^{m_l}$ are either nested or disjoint and, hence,  form a hierarchical structure as shown in Figure \ref{fig:supports}.

\begin{figure}[ht]
\centering 
\includegraphics[width=0.24\textwidth]{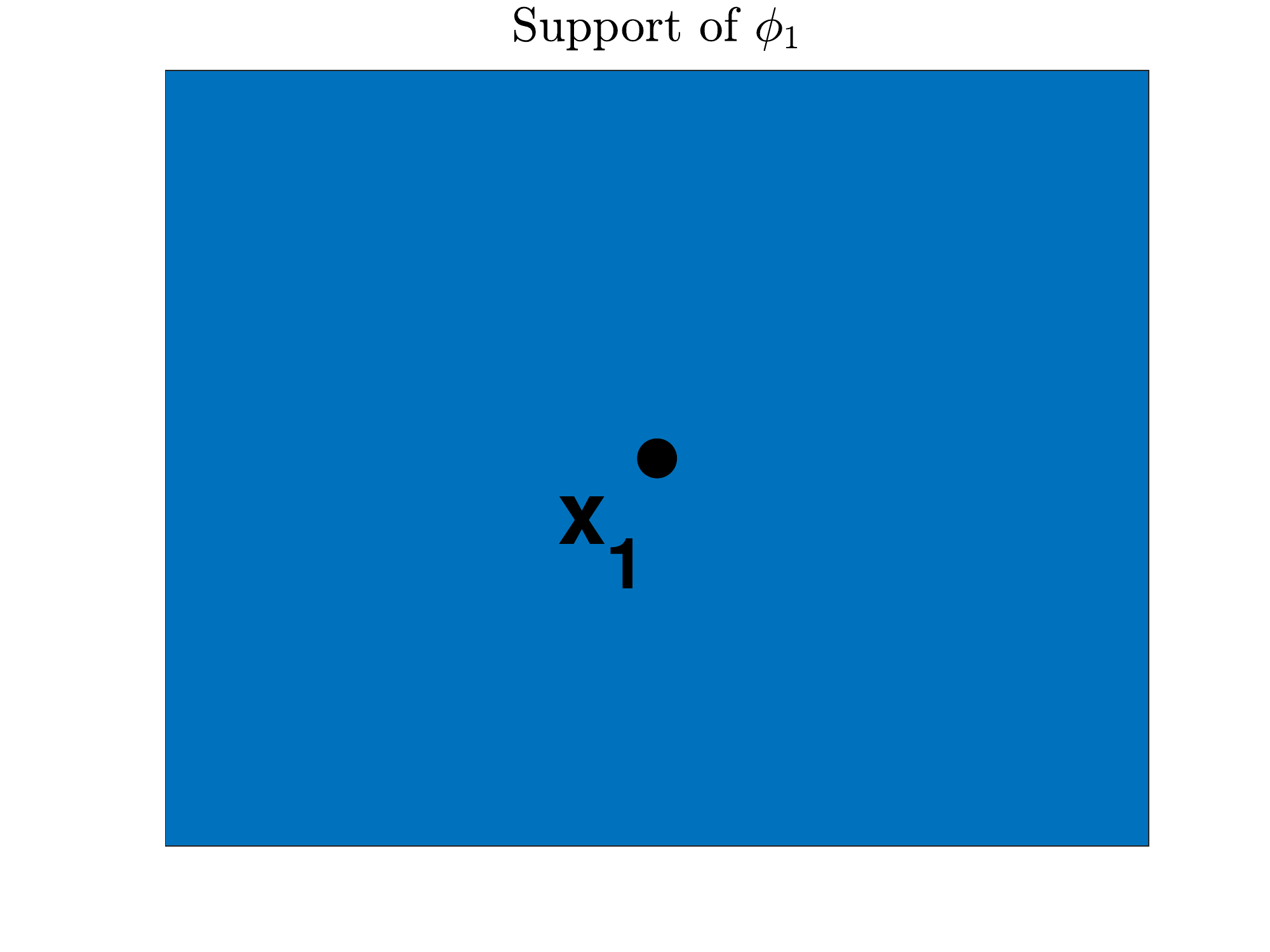} 
\includegraphics[width=0.24\textwidth]{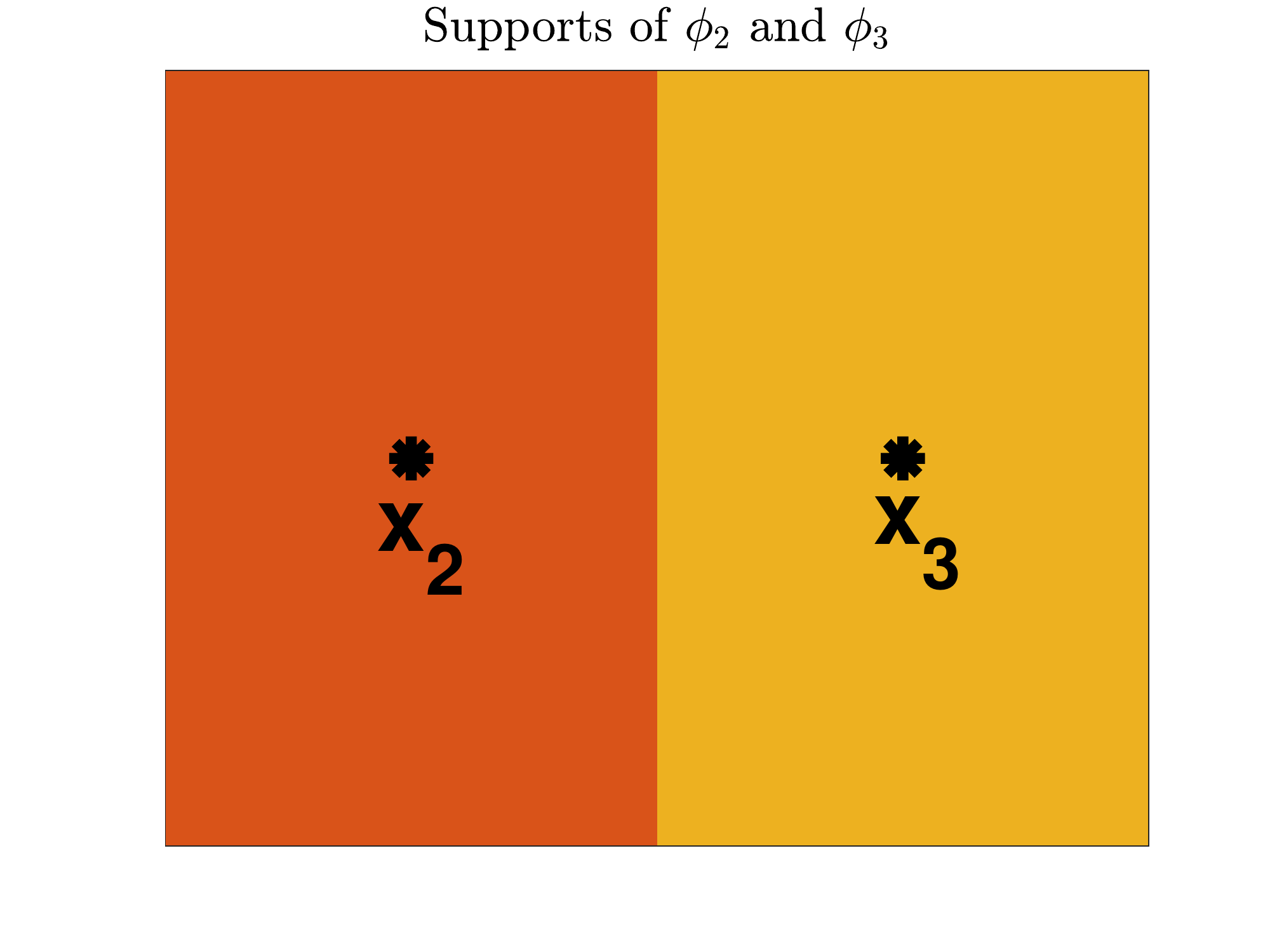} 
\includegraphics[width=0.24\textwidth]{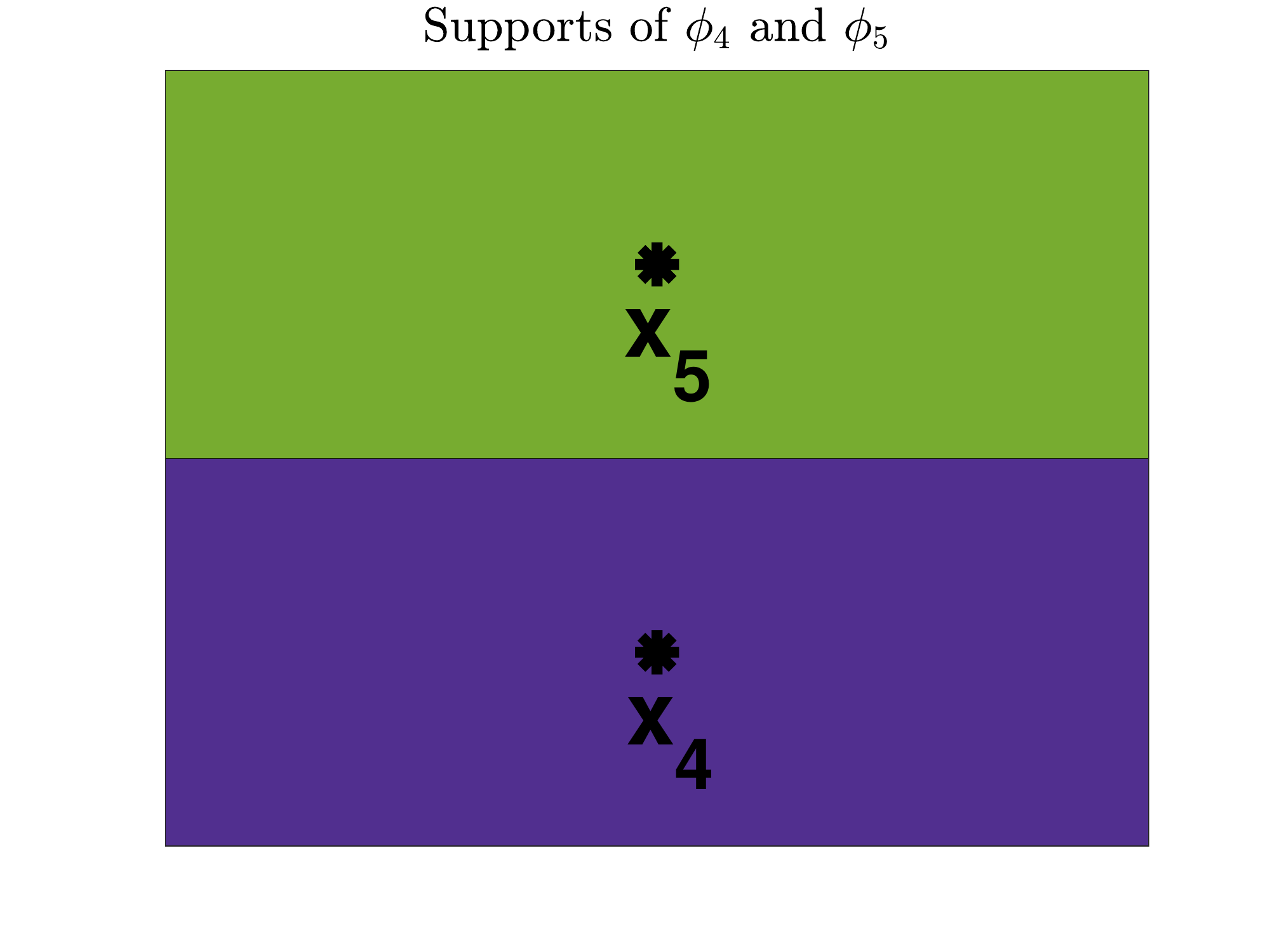}
\includegraphics[width=0.24\textwidth]{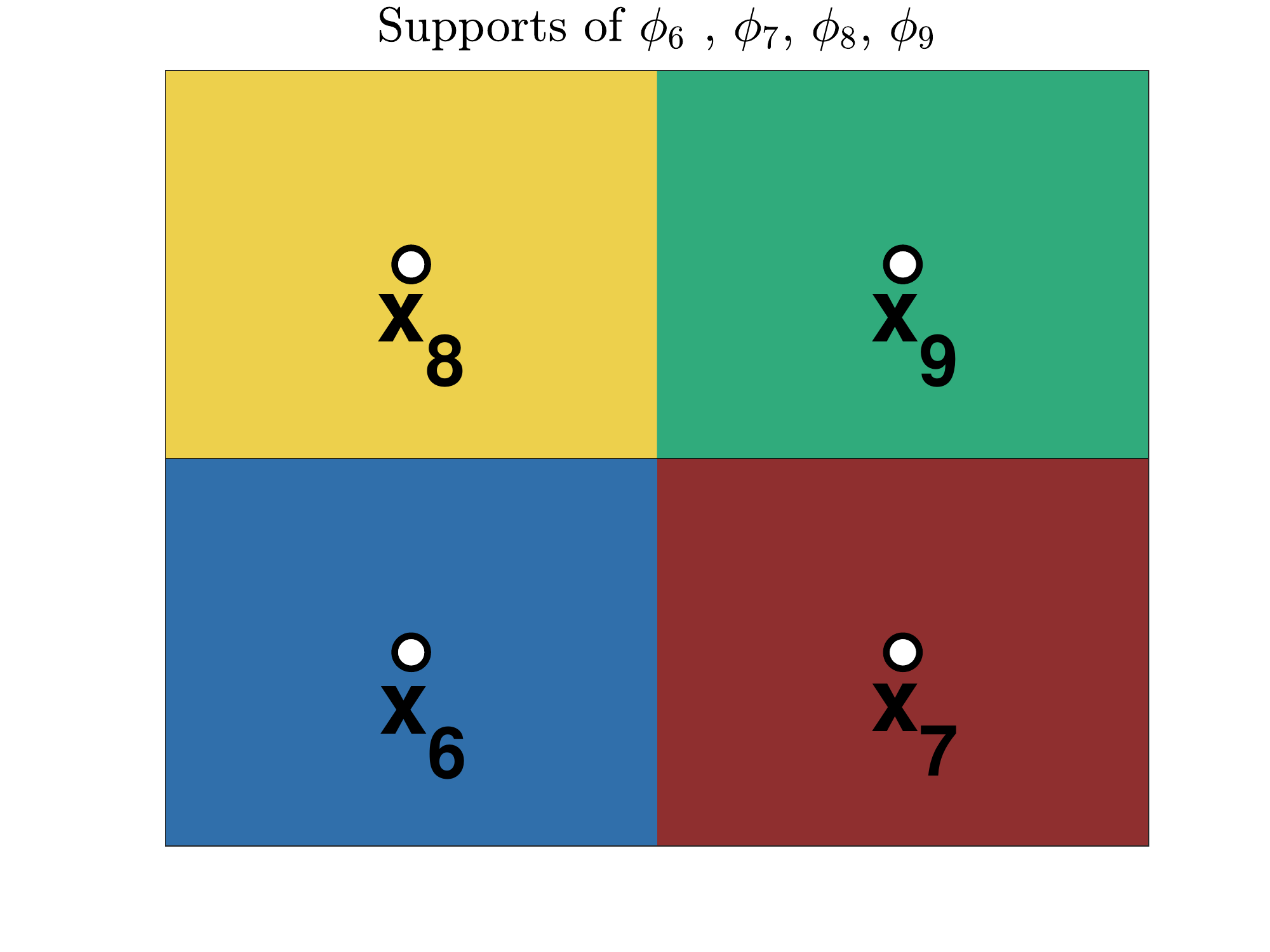}
\caption{\label{fig:supports} Support of hierarchical feature $\phi_1$ corresponding to $\Bx_1$ in level 1 (left); supports of $\{\phi_i\}_{i=2}^5$ corresponding to $\{\Bx_i\}_{i=2}^5$ (middle figures); supports of $\{\phi_i\}_{i=6}^9$ corresponding to $\{\Bx_i\}_{i=6}^9$ in level 3 (right).}
\end{figure}

$\BX^{\rm SG}_l$ is also regarded as the index set of $\Bphi$, because instead of labeling each hierarchical feature by number $i$, we can also label it by the corresponding inducing variable $\Bx_i\in\BX^{\rm SG}_l$. As shown in Figure \ref{fig:phi1andphi23}, for any $i$, $\Bx_i$ is the only point at which $\phi_i$ is not differentiable along all dimensions. Hence,  we call $\Bx_i$ the center of $\phi_i$. From this perspective, we can claim that, given any $\Bx^*$, the distribution of the approximation at  $\Bx^*$ only depends on $\CalO([\log m_l]^{2d-1})$ inducing variables. As a result, the hierarchical expansion  generalizes the one-dimensional Markov property while still remaining a high accuracy approximate. We discuss the approximation error of hierarchical expansion in supplementary material.

\begin{figure}[ht]
\centering 
\includegraphics[width=0.3\textwidth]{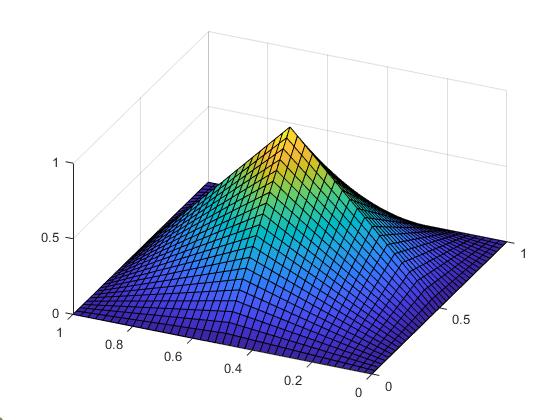} 
\includegraphics[width=0.45\textwidth]{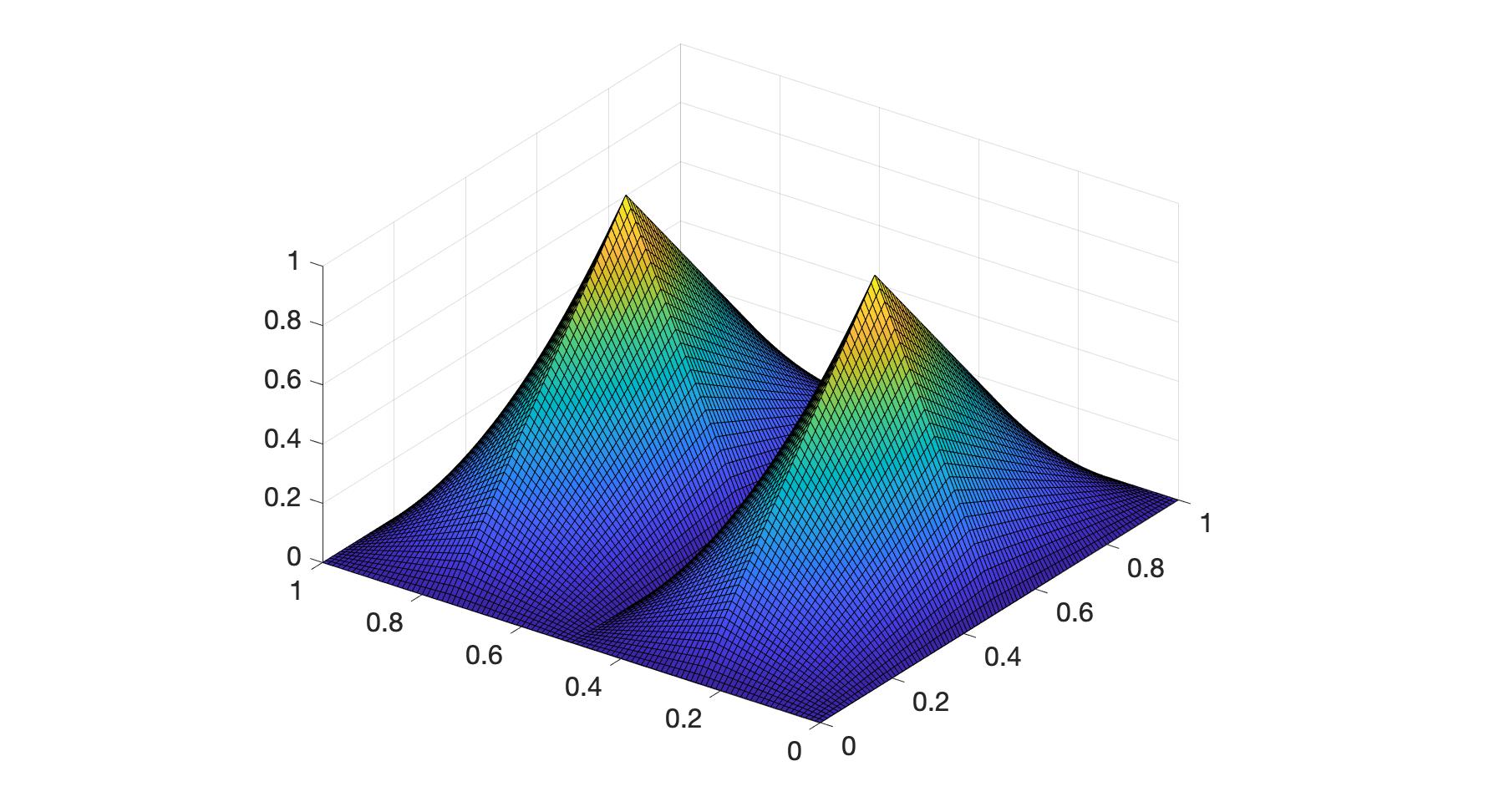} 
\caption{\label{fig:phi1andphi23} Hierarchical features $\phi_1$ (left) and $\phi_2,\phi_3$ (right) generated by the tensor Brownian Bridge kernel $k(\Bx,\Bx')=\prod_{j=1}^2\min\{x_j,x_j'\}(1-\max\{x_j,x_j'\})$. }
\end{figure}
The following numerical example  illustrates the sparsity of hierarchical expansion:
\begin{example}
Let two-dimensional TMK  $k(\Bx,\Bx')=\exp\{-\sum_{j=1}^2|x_j-x_j'|\}$. Let 
$$\BX^{\rm SG}_2=\left[\left(\frac{1}{2},\frac{1}{2}\right),\left(\frac{1}{2},\frac{1}{4}\right),\left(\frac{1}{2},\frac{3}{4}\right),\left(\frac{1}{4},\frac{1}{2}\right),\left(\frac{3}{4},\frac{1}{2}\right)\right],$$ 
be the sorted level-2 SG. Then, we numerically compute the inverse Cholesky decomposition $R_2^{-1}$ and we get
\begin{equation*}
    R^{-1}_2=
    {\tiny
    \begin{bmatrix}
     1.0000  & -1.2416 &  -1.2416 &  -1.2416 &   -1.2416\\
     0       & 1.5942  &  0.0000  &  0.0000  &   0.0000\\
     0       &  0      & 1.5942   &  0.0000  &  0.0000\\
     0       &  0      &      0   &  1.5942  &   0.0000\\
     0       &  0      &      0   &  0       &  1.5942    
    \end{bmatrix}}.
\end{equation*}
Let $\Bx^*_1=(0.8147 ,   0.9058)$, and $\Bx^*_2=(0.2785    , 0.5469)$ be two randomly chosen points.  We numerically compute the hierarchical features at $\Bx^*_1$ and $\Bx^*_2$ and we get
\begin{align*}
    &k(\Bx^*_1,\BX^{\rm SG}_2)R^{-1}_2=\big[0.4865, 0,  0.3918,    0.0000  ,    0.3918\big],\\
    &k(\Bx^*_2,\BX^{\rm SG}_2)R^{-1}_2=\big[0.7646,    0.5291 ,    0     ,    0    ,0.0933\big].
\end{align*}
As we can see,  matrix $R^{-1}_2$ is sparse and each of $k(\Bx^*_1,\BX^{\rm SG}_2)R^{-1}_2$ and $k(\Bx^*_2,\BX^{\rm SG}_2)R^{-1}_2$ has two zero elements. The sparsity is even more visible as the number of inducing points increases.
\end{example}

\subsection{Deep Tensor Markov GPs}
A $H$-layer DTMGP $\CalT^{(H)}$ is defined as the composition of $H$ hierarchical expansions of multi-variate TMGP:
\begin{equation}
    \label{eq:DTMGP}
    \CalT^{(H)}(\Bx^*)= [\BZ^{(H)}\Bphi^{(H)}(\cdot)+\Bmu^{(H)}]\circ\cdots\circ  [\BZ^{(1)}\Bphi^{(1)}(\Bx^*)+\Bmu^{(1)}],
\end{equation}
where, for $h=1,\cdots,H$, $\BZ^{(h)}$ is a $W^{(h)}$-by-$m^{(h)}$ matrix with i.i.d. standard Gaussian entries, $W^{(h)}$ is the dimension of the output (width of the neural network) of layer $h$, $m^{(h)}$ is the number of hierarchical features for approximating a $W^{(h)}$-variate TMGP, $\Bphi^{(h)}$ consists of the $m^{(h)}$ hierarchical features for approximating the TMGP, acting as the activation functions in layer $h$, and $\Bmu^{(h)}\in\Real^{W^{(h)}}$ is the parameterized mean for the TMGP in layer $h$. 

 Let us now focus on layer $h$ of $\CalT^{(H)}$ to compare DTMGP with DNN. Let $\CalT^{(h)}=(\CalT^{(h)}_1,\cdots,\CalT^{(h)}_{W^{(h)}})$ denote the $W^{(h)}$-dimensional output of layer $h$. We then have the following structure:
\begin{equation*}
    \CalT^{(h-1)}\rightarrow \underbrace{\big(\phi^{(h)}_1(\CalT^{(h-1)}),\cdots,\phi^{(h)}_{m^{(h)}}(\CalT^{(h-1)})\big)}_{\text{\# non-zero entries:}\ \CalO\big([\log m^{(h)}]^{2W^{(h-1)}-1}\big)}\rightarrow \bigg[\CalT_j^{(h)}=\mu^{(h)}_j+\sum_{i=1}^{m^{(h)}}Z_{j,i}\phi^{(h)}_i(\CalT^{(h-1)})\bigg]_{j=1}^{W^{(h)}},
\end{equation*}
where $\phi^{(h)}_i(\CalT^{(h-1)})$ is the $i^{\rm th}$ entry of vector $k^{(h)}(\CalT^{(h-1)},\BX^{\rm SG}_{l^{(h)}})[R^{(h)}]^{-1}$, $k^{(h)}$ is the TMK for layer $h$, $\BX^{\rm SG}_{l^{(h)}}$ is a level-$l^{(h)}$ SG of size $m^{(h)}$ for induced approximation, and $R^{(h)}$ is the Cholesky decomposition of covariance matrix $k^{(h)}(\BX^{\rm SG}_{l^{(h)}},\BX^{\rm SG}_{l^{(h)}})$.  This structure is also illustrated in Figure \ref{fig:DTMGP_graph}.

\begin{figure}[t!]
\centering
\includegraphics[width=.7\textwidth]{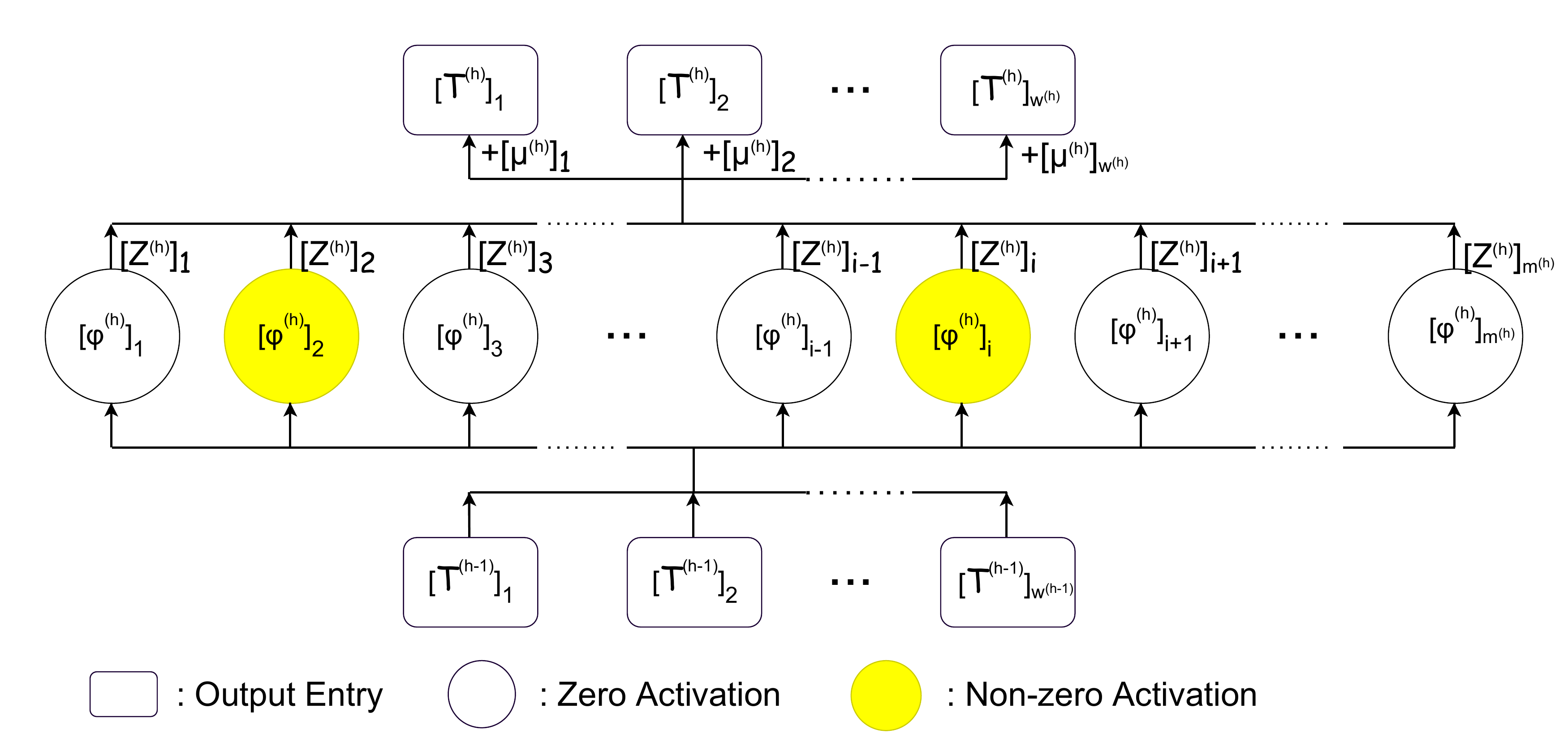}

\caption{Only $\CalO\big([\log m^{(h)}]^{2W^{(h-1)}-1}\big)$ activations are non-zero in layer $h$ of DTMGP $\CalT^{(H)}$. \label{fig:DTMGP_graph}}
\end{figure}
If we treat DTMGP as a DNN, then, for $h=1,\cdots,H$, layer $h$ is equipped with $m^{(h)}$ activations $\{k^{(h)}(\cdot,\Bx):\Bx\in\BX^{\rm SG}_{l^{(h)}}\}$, a sparse linear filter $[R^{(h)}]^{-1}$ which converts the dense input to sparse local hierarchical features, bias $\Bmu^{(h)}$, and i.i.d. standard Gaussian weights. Moreover, sparse linear filters $\{[R^{(h)}]^{-1}\}_{h=1}^H$ are fixed for a specific hierarchical expansion. Therefore,  $\{[R^{(h)}]^{-1}\}_{h=1}^H$ can be  computed and stored in advance. Due to Theorems \ref{thm:nnz-cholesky} and \ref{thm:nnz-hierarchical-feature}, the space complexity for storing $\{[R^{(h)}]^{-1}\}_{h=1}^H$ is linear in the number of inducing points and the time complexity of any operation in layer $h$ of DTMGP is $\CalO([\log m^{(h)} ]^{2W^{(h-1)}-1})$.

\subsection{Variational Inference}
{
The proposed DTMGP is mathematically equivalent to a deep neural network with biases $\{\Bmu^{(h)}\}_{h=1}^H$ and independently and normally distributed weights  $\{\BZ^{(h)}\}_{h=1}^H$. So we can parameterize DTMGP as a BNN with  independently and normally distributed weights. Specifically, we regard the variances and means of all Gaussian weights, and the biases as  parameters of the BNN. A DTMGP can be represented as a stochastic process $f_{\Btheta}$ in the parameterized family $\mathcal{F}_{\BTheta}$:
\begin{align*}
     &\mathcal{F}_{\BTheta}=\bigg\{\left[\big(\BSigma^{(H)}\odot\BZ^{(H)}+\Bm^{(H)}\big)\Bphi^{(H)}+\Bmu^{(H)}\right]\circ\cdots\circ  \left[\big(\BSigma^{(1)}\odot\BZ^{(1)}+\Bm^{(1)}\big)\Bphi^{(1)}+\Bmu^{(1)}\right]:\\
     &\quad\quad\quad  \big\{\BSigma^{(h)},\Bm^{(h)},\Bmu^{(h)}\big\}_{h=1}^H\in\BTheta \bigg\},
\end{align*}
where $\BTheta$ denotes the set of parameters
\[\BTheta\coloneqq\bigg\{\BSigma^{(h)}\in \Real_+^{W^{(h)}\times m^{(h)}},\Bm^{(h)}\in\Real^{W^{(h)}\times m^{(h)}},\Bmu^{(h)}\in\Real^{W^{(h)}},h=1,\cdots,H\bigg\},\]
and $A\odot B$ denotes the entry-wise multiplication of matrices $A$ and $B$, i.e., the Hadamard product. In other words, the family $\mathcal{F}_{\BTheta}$ consists of BNNs with activations $\{\Bphi^{(h)}\}_{h=1}^{H}$, parameterized biases, and normally distributed weights.

Given observations $(\BX,\BY)$, our goal is to search for the parameters $\Btheta^*$ that maximizes the marginal likelihood $\prob(\BY\big|\BX,\Btheta)$. The DTMGP $f_{\Btheta^*}$ can best interpret how $(\BX,\BY)$ is generated. Computing $\prob(\BY|\BX,\Btheta)$ of any DGP involves  intractable integral as shown in \eqref{eq:DGP-integral}, but we can obtain a tractable lower bound using VI. VI first assigns a prior $f_{\tilde{\Btheta}}$ with some $\tilde{\Btheta}=\{\tilde{\BSigma}^{(h)},\tilde{\Bm}^{(h)},\tilde{\Bmu}^{(h)}\}_{h=1}^H$ to the underlying stochastic process generating $(\BX,\BY)$ and then maximizes the following  \emph{evidence lower bound} (ELBO):
\begin{equation}\label{eq:ELBO_origin}
    \mathcal{E}(\Btheta)=\E_{f_{\Btheta}}\bigl[\log \Pr(\BY\big|\BX,\Btheta)\bigr]-\mathcal{D}_{KL}\bigl(f_{\Btheta}\big\|f_{\tilde{\Btheta}}\bigr),
\end{equation}
which is the lower bound of log-likelihood $\log \Pr\big(\BY\big|\BX,\Btheta\big)$.

In our framework, the term $\E_{f_{\Btheta}}\bigl[\log \Pr(\BY\big|\BX,\Btheta)\bigr]$ is called the \emph{negative energy} and can be explicitly written as 
\[\E_{f_{\Btheta}}\bigl[\log \Pr(\BY\big|\BX,\Btheta)\bigr]= \E_{\BZ^{(1)},\cdots,\BZ^{(H)}}\bigl[\log \Pr(\BY\big|\BX,\BZ^{(1)},\cdots,\BZ^{(H)},\Btheta)\bigr].\]
This term now can be efficiently estimated by stochastic optimization based on Monte Carlo methods, such as doubly-stochastic approximation \citep{hensman2014nested,Dai14,Salimbeni17}.  The term $\mathcal{D}_{KL}\bigl(f_{\Btheta}\big\|f_{\tilde{\Btheta}}\bigr)$ is the \emph{KL-divergence} between $f_{\Btheta}$ and the prior $f_{\tilde{\Btheta}}$ and it can be explicitly written
\begin{equation*}
    \mathcal{D}_{KL}\bigl(f_{\Btheta}\big\|f_{\tilde{\Btheta}}\bigr)=\frac{1}{2}\sum_{h=1}^{H}\sum_{i=1}^{W^{(h)}}\sum_{j=1}^{m^{(h)}}\frac{\big|[\Bm^{(h)}]_{i,j}-[\tilde{\Bm}^{(h)}]_{i,j}\big|^2}{[\tilde{\BSigma}^{(h)}]^2_{i,j}}+\frac{[\BSigma^{(h)}]^2_{i,j}}{[\tilde{\BSigma}^{(h)}]^2_{i,j}}-1- \log \frac{[\BSigma^{(h)}]_{i,j}^2}{[\tilde{\BSigma}^{(h)}]_{i,j}^2}. 
\end{equation*}

The computation of $\mathcal{E}$ is efficient at any $\Btheta\in\BTheta$, the training process of DTMGP is simply defined as 
\begin{equation}
    \label{eq:max-elbo}
    \Btheta^*=\arg\max_{\Btheta\in\BTheta}\left\{\E_{f_{\Btheta}}\bigl[\log \Pr(\BY\big|\BX)\bigr]-\mathcal{D}_{KL}\bigl(f_{\Btheta}\big\|f_{\tilde{\Btheta}}\bigr)\right\}.
\end{equation}
This optimization problem can also be treated as a penalized regression problem where the negative energy term is the model fitness level and the KL-divergence term is the penalty of being far from the DTMGP prior. Equation \eqref{eq:max-elbo} can be solved efficiently by using the automatic differentiation technique \citep{neidinger2010introduction,Atilim18}. 

}

\section{Simulation Studies}\label{sec:numerical}
In this section, we run simulations on systems with stochastic outputs to access DTMGP's capacity in modeling stochastic processes. More specifically, in each experiment, output $Y$ of the underlying system at each input $\Bx$ follows an \emph{unknown} distribution $F(\cdot|\Bx)$:
\[Y(\Bx)\sim F(y|\Bx).\]
A training set $(\BX_{\rm train},\BY_{\rm train})$ is independently collected where $\BY_{\rm train}=Y(\BX_{\rm train})$ is the realization of the underlying system on inputs $\BX_{\rm train}$.  We then train competing models  using training set $(\BX_{\rm train},\BY_{\rm train})$.  

Let $\widehat{Y}(\cdot)$ denote a trained model.  We first choose a test set $\BX_{\rm test}$ from the input space. For each input $\Bx\in\BX_{\rm test}$, we sample $100$ independent realizations from the true underlying system and from the trained model to get data sets  $\{{Y}_i(\Bx)\}_{i=1}^{100}$ and $\{\widehat{Y}_i(\Bx)\}_{i=1}^{100}$, respectively.  We then use the  two-sample \emph{Kolmogorov–Smirnov} (KS) statistic:
\[{D}_{\Bx}=\sup_{y}\left|\frac{1}{100}\sum_{i=1}^{100}\bold{1}_{\{Y_i(\Bx)\leq y\}}-\frac{1}{100}\sum_{i=1}^{100}\bold{1}_{\{\widehat{Y}_i(\Bx)\leq y\}}\right|\]
to quantify the similarity between $Y$ and $\widehat{Y}$ on input $\Bx$. The smaller $D_{\Bx}$ is, the closer the two distributions are. We access the overall performance of  $\widehat{Y}(\cdot)$  via the following averaged two-sample KS statistic over test set $\BX_{\rm test}$:
\[{D}=\frac{1}{|\BX_{\rm test}|}\sum_{\Bx\in\BX_{\rm test}}D_{\Bx}.\]
We repeat each experiment $R$ times and call each a macro-replication from which we obtain  ${D}_r$, $r=1,\ldots,R$.
Finally, we report the mean and standard deviation 
\begin{align*}
    &\bar{D}=\frac{1}{R}\sum_{r=1}^DD_r,\\
    &\hat{\sigma}=\biggl[\frac{1}{R}\sum_{r=1}^R \biggl({D}_r - \bar{D} \biggr)^2\biggr]^{1/2} .
\end{align*}

In all experiments, activations of DTMGP are hierarchical features generated from TMGPs with the Laplace covariance function $k(\Bx,\Bx')=\exp\{-\frac{\|\Bx-\Bx'\|_1}{d}\}$. We compare DTMGP with the following existing  models:
\begin{enumerate}
    \item \textbf{GP}: standard Gaussian process with squared exponential covariance function. 
    \item \textbf{BNN-ReLU} \citep{schmidt1992feedforward,Graves11,blundell2015weight}: ReLU deep neural networks \citep{Glorot11} with parameterized Gaussian distributed weights and biases. 
     VI is used for training BNN-ReLU.
    \item \textbf{DGP-RFF} \citep{Cutajar17}: A Random Fourier feature expansion of DGP   with each layer represented by Gaussian processes with the Gaussian kernel $\exp\{-\sum_{j=1}^dw_j|x_j-x'_j|^2\}$. 
    DGP-RFF is trained with VI approximation provided in \cite{Cutajar17}.
    \item \textbf{DGP-VEC} \citep{sauer2022vecchia} DGP where each activation is GP with Mat\'ern-$1/2$ kernel. Posterior of the model is approximated through the incorporation of the Vecchia approximation \citep{katzfuss2021general}.
\end{enumerate}

Note that  all deep models can be represented by the same architecture and the difference among them  lies in their activations. For DTMGP, BNN-ReLU, and DGP-RFF, they are BNN with different deterministic activations and random weights. For DGP-VEC, it is composition of GPs. Therefore, for each data set, we  adopt the same architectures for fair comparisons. In  other words, except for the specific activations, all deep models are with equal number of layers,  widths, weights, and biases.

We conduct experiments with two simulation models: a two-dimensional non-Gaussian random field  in Section~\ref{sec:RandomField}, and 
the expected revenue of a 13-dimensional Stochastic Activity Network problem in Section~\ref{sec:ProdLine}. All the experiments are implemented on a computer with macOS, 3.3 GHz Intel Core i5 CPU, and 8 GB of RAM (2133Mhz).
 
\subsection{A Non-Gaussian Random Field}\label{sec:RandomField}
In this section, we use the following non-Gaussian and non-stationary random field in two dimensions to access the sample and computational efficiency of the proposed methodology:
\begin{align*}
    Y(\Bx)=\frac{1}{1+\exp\{B(\Bx)\}},\quad \Bx\in[0,1]^2
\end{align*}
where $B(\Bx)$ is a \emph{Brownian sheet}, which is defined as a zero-mean GP with covariance function $\prod_{d=1}^2\big(1+\min\{x_d,x_d'\}\big)$.

\begin{figure}[ht]
\centering 
\includegraphics[width=0.35\textwidth]{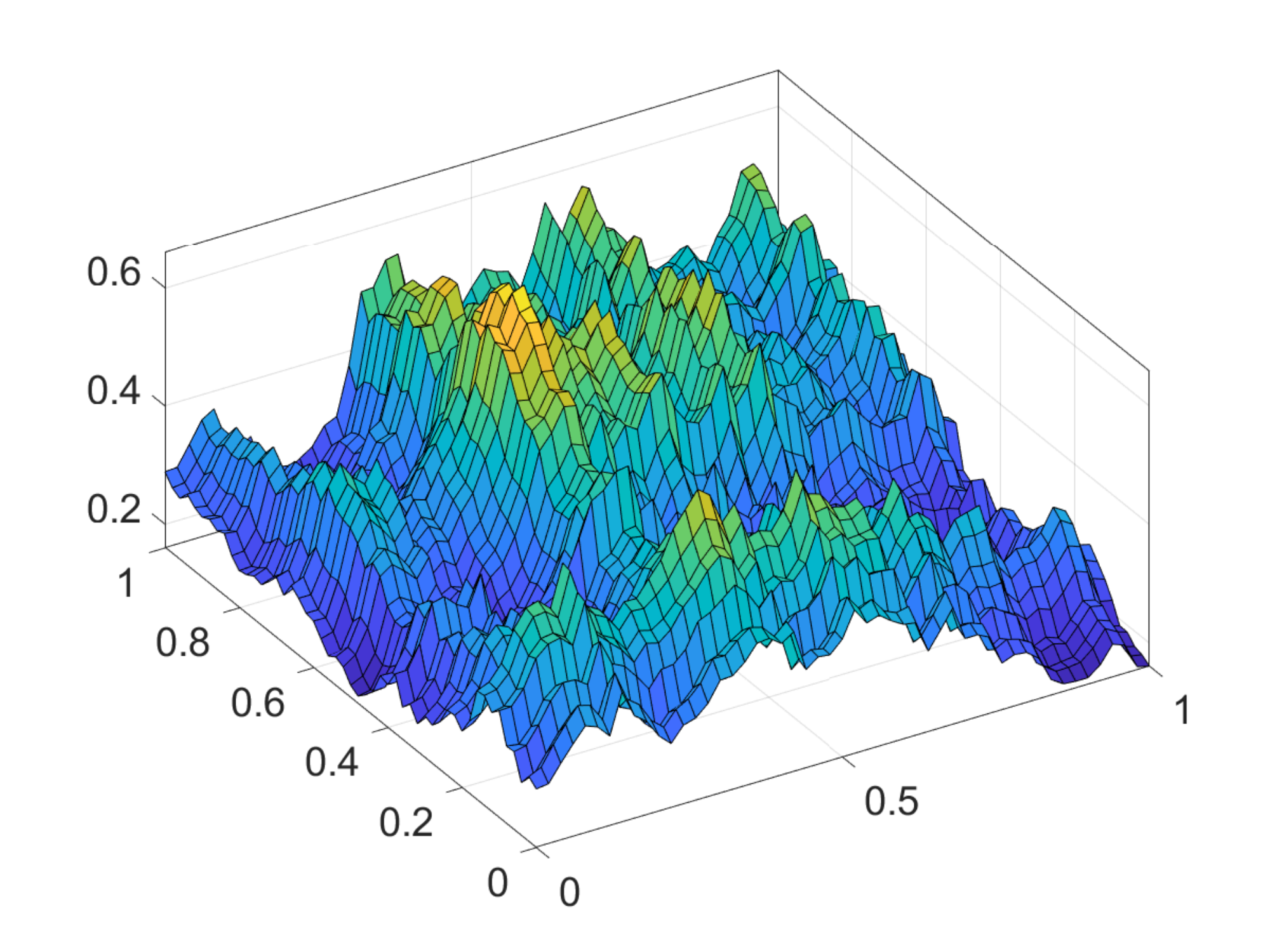} 
\includegraphics[width=0.3\textwidth]{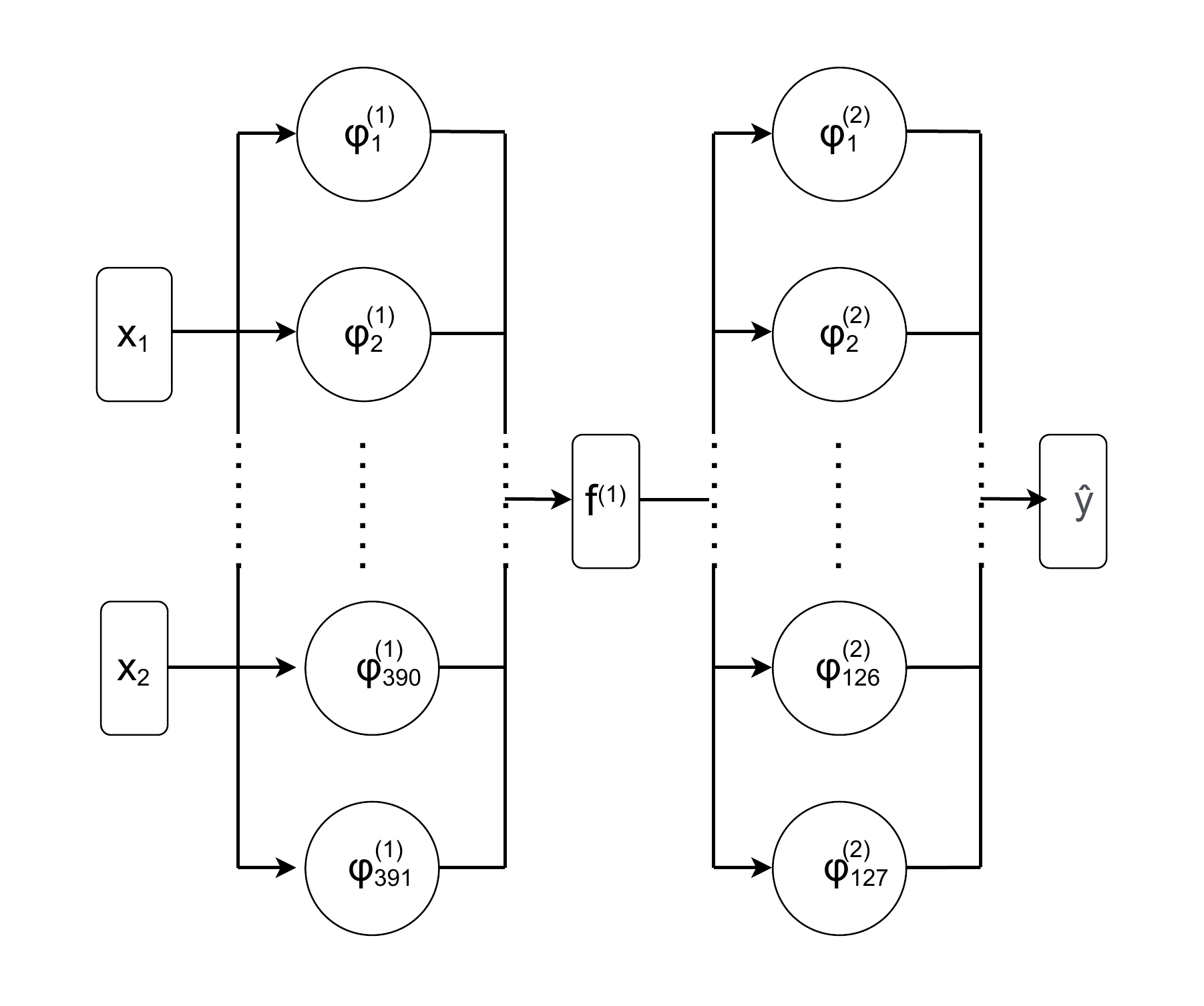}
\caption{\label{fig:RandomField} A sample path of non-Gaussian random field $\frac{1}{1+\exp\{B(\Bx)\}}$. Right: architecture shared by all deep models to capture the random field.}
\end{figure}

In this experiment, $\BX_{\rm train}$ is randomly collected from $[0,1]^2$ and $\BY_{\rm train}$ is from independently sample paths of $B(\Bx)$ on $\BX_{\rm train}$. We investigate the performance  of each method as the sizes of the training set $(\BX_{\rm train},\BY_{\rm train})$  increase.
For each training set, we randomly select $n_{\rm test}=100$ points from $[0,1]^2$ as the test set, denoted as $\BX_{\rm test}$.
We run $R=20$ macro-replications and compute $\hat{D}$ and the associated standard deviation $\hat{\sigma}$ for each method. 

All deep models are with two hidden layers where the first layers consists of 181 activations and 1 output, and the second layer consist of 127 activations and 1 output. Figure \ref{fig:RandomField} shows the realization of the random field and architecture shared by all deep models. We let the i.i.d. prior of each coefficient in DTMGP be normal distributed with mean 0 and variance 1.

Our experiments reveal that DGP-VEC has a significantly longer training time compared to all other models. Specifically, Table \ref{tab:runtime_test_func} demonstrates that while DTMGP has the fastest convergence rate to stable training error, DGP-VEC exhibits the slowest. Consequently, to ensure a fair comparison, we report our results at different training times. First, we report the KS statistics of all models when DTMGP's training error begins to converge, and then we report the KS statistics of all models when DGP-VEC's training error begins to converge.
\begin{table}[]
    \centering
\begin{tabular}{ |p{2cm}|p{1.6cm}|p{1.6cm}|p{1.6cm}|p{1.6cm}|p{1.6cm}|  }
 \hline
 model &n=40& n=80&n=120&n=160&n=200\\
 \hline
 DTMGP&28.7 sec & 53.5 sec &  89.8 sec & 138.1 sec & 166.4 sec\\
 \hline
 DGP-VEC&44.2 sec & 97.9 sec &  255.3 sec & 384.2 sec & 516.7 sec\\
 \hline
\end{tabular}
    \caption{Time Required to Convergence of Training Error}
    \label{tab:runtime_test_func}
\end{table}

\begin{figure}[ht]
\centering 
\includegraphics[width=0.45\textwidth]{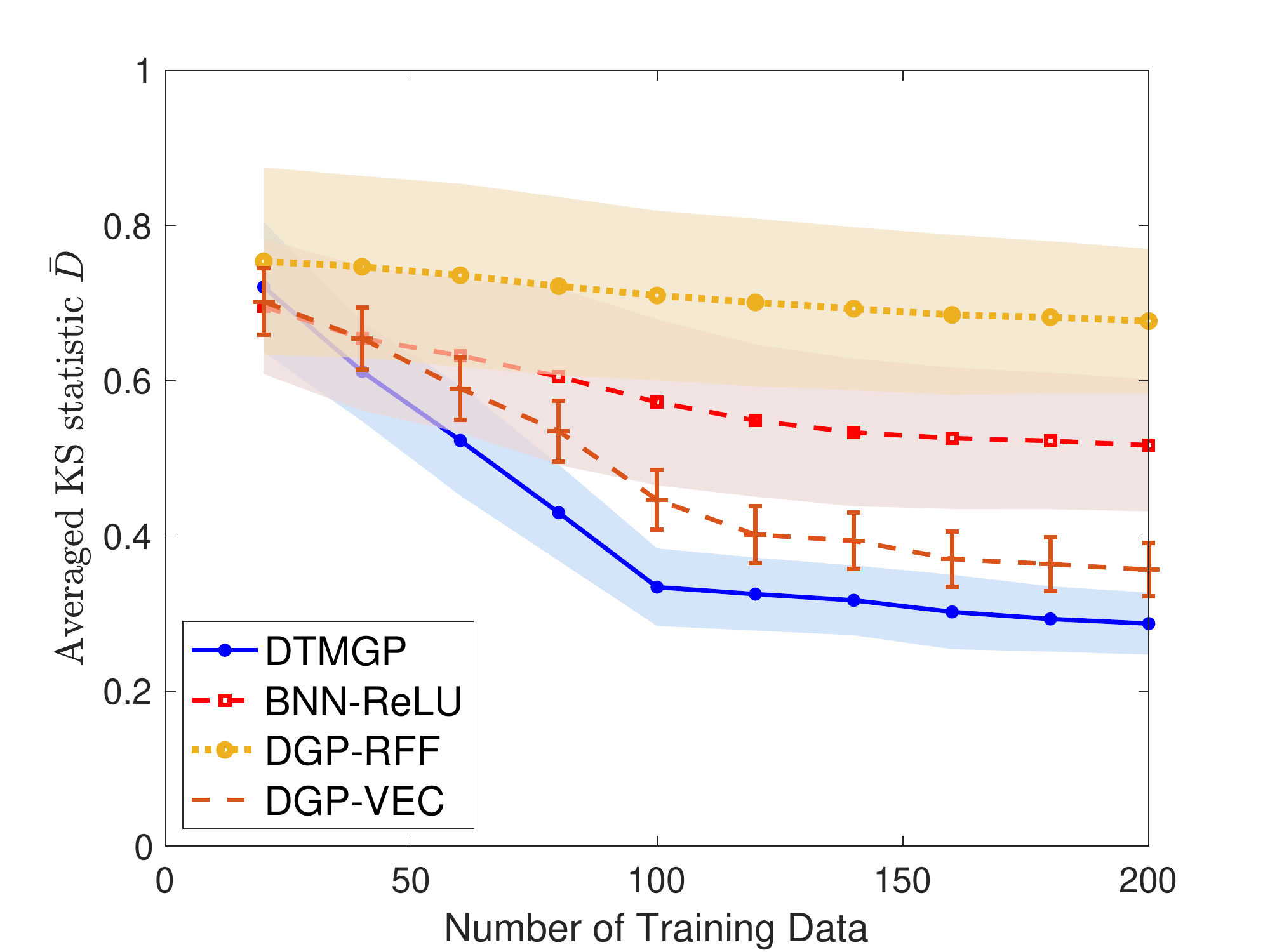} 
\includegraphics[width=0.45\textwidth]{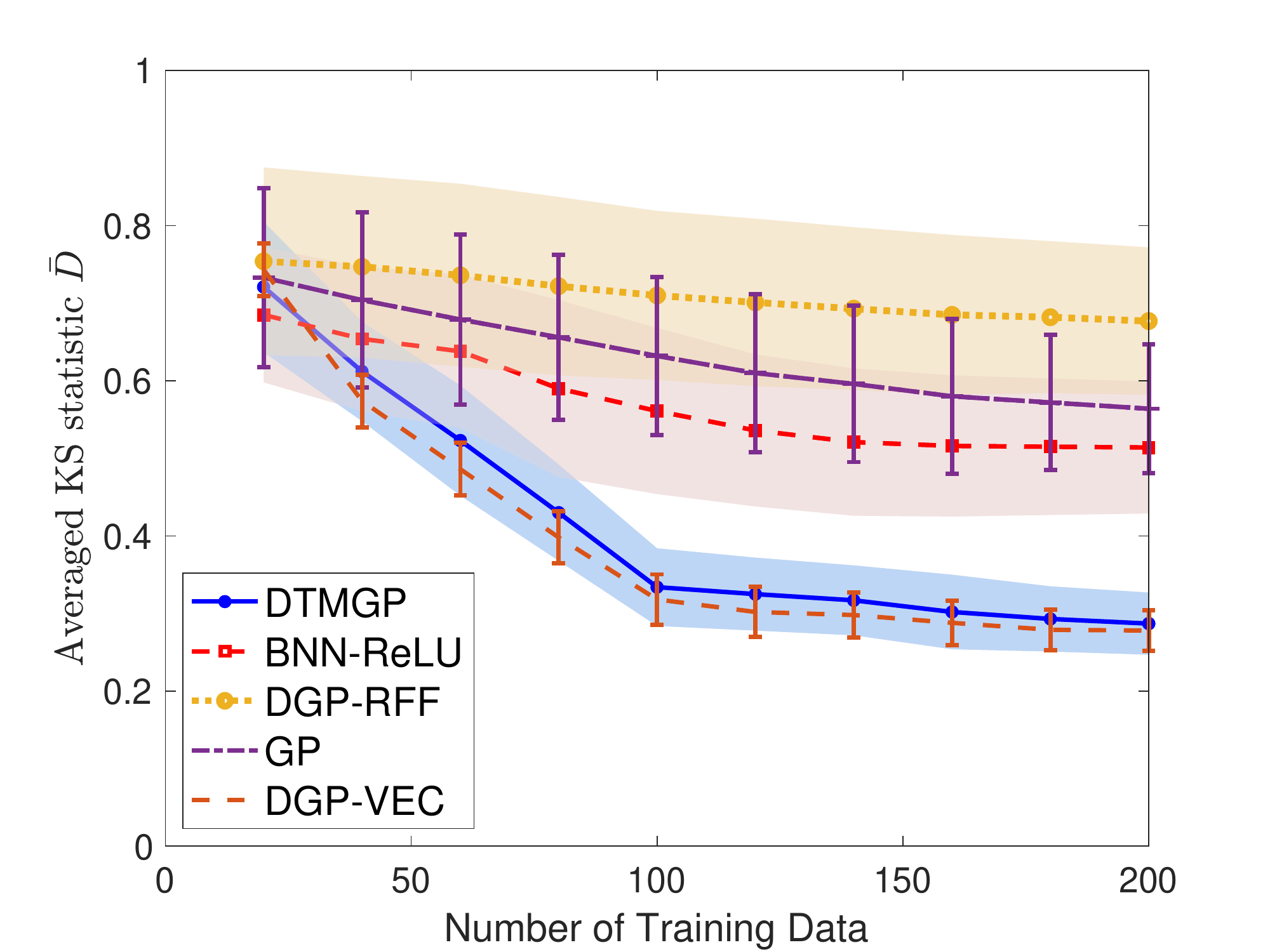} 
\caption{\label{fig:RandomField_result} Left:  $\bar{D}$ of each model, the shaded areas and error bars represent 1 std $\hat{\sigma}$ of all models. Left: results reported when KS statistics of DTMGP begins to converge. Right: results reported when KS statistics of DGP-VEC begins to converge or traing time reaches 4 hours.}
\end{figure}

As shown in Figure \ref{fig:RandomField}, when given enough training time, DGP-VEC outperforms all other models. However, it's worth noting that the performance of DTMGP is still comparable to that of DGP-VEC. If we limit the training time for all models to end once the training error of DTMGP begins to converge, we find that DTMGP - which has the fastest convergence rate - outperforms all other models, including DGP-VEC. This is despite the fact that DGP-VEC requires a training time that is three to four times longer than DTMGP, as it relies on MCMC for training. As it is widely acknowledged that VI is more efficient than MCMC, this serves as an important reminder of the tradeoffs we must consider when selecting a training method.

\subsection{Stochastic Activity Network}\label{sec:ProdLine}
In this subsection, we consider the stochastic activity network (SAN) where the arcs are
labeled from 1 through 13. The detailed explaition of SAN is available in \cite{avramidis1996integrated}. 
As shown in the left of Figure \ref{fig:nn_data} , each arc $i$ in the SAN is associated with
a task with random duration $D_i$ and task durations are mutually independent. Suppose that $D_i$ is exponentially distributed with mean $X_i$ for each $i$. Suppose that we can control $X_i > 0$ for each $i$, but there is an associated cost. In particular, the overall cost is defined as
$$C(\BX)=T(\boldsymbol{\BX}) + f(\BX)$$
where $\BX=(X_1,\cdots,X_{13})$, $T(\BX)$ is the (random) duration of the longest path from a to i, and $f(\BX)=\sum_{i=1}^{13}X_i^{-1}$. Closed form of $C(\BX)$ is unknown but \textsc{MATLAB} simulator of this problem is available in the {SimOpt} library \citep{Simopt}. 

In this experiment,  data are collected from a maximin Latin hypercube design (LHD) that maximizes the minimum distance between points \citep{van2007maximin}. The LHD consists of 5000 sample points from the cube $[0.5, 5]^{13}$. At each sample point, $m$ simulation replications are run with $m=2,4,\cdots,20$. In order to select reasonable prior for the target stochastic process, we first normalize all the output so that all output data in the training set and testing set are distributed on $[0,1]$. We then run $m=10$ replications at each design point during the training process to 
estimate the variance  at each point. Based on the sample estimates, we let the i.i.d. prior of each coefficient in DTMGP be normal distributed with mean 0 and variance 0.04. 

We examine the performance  of each method as the number of replications $m$ increase.
Given $m$, we construct the test set $\BX_{\rm test}$ of size $n_{\rm test}=100$ by random samples from $[0.5,5]^{13}$.
We compute $\hat{D}$ of each method, its associated standard deviation $\hat{\sigma}$ based on $R=20$ macro-replications.

\begin{figure}[ht]
\centering 
\includegraphics[width=0.35\textwidth]{figure/Experiments/SAN.pdf}
\includegraphics[width=0.3\textwidth]{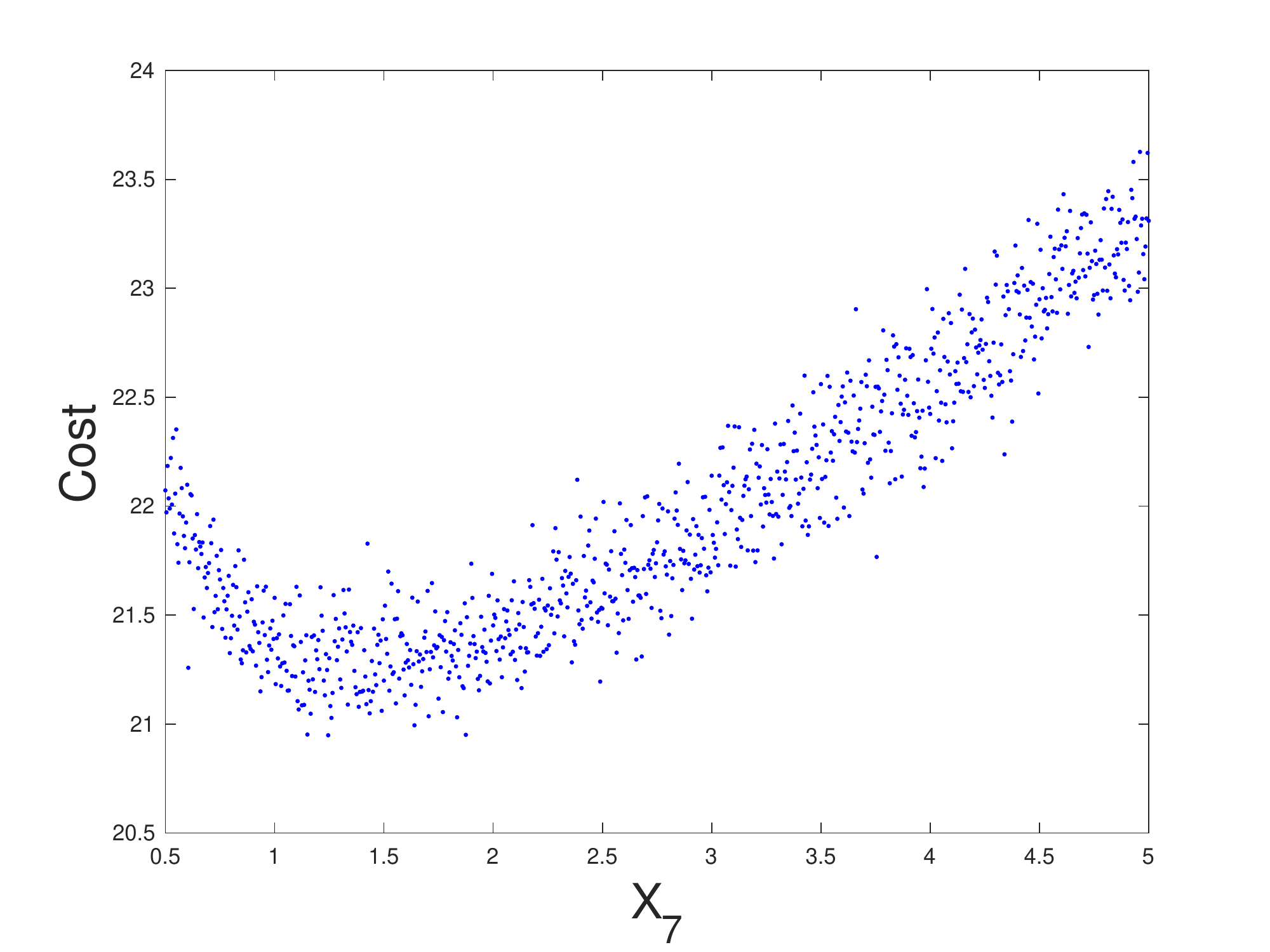} 
\includegraphics[width=0.3\textwidth]{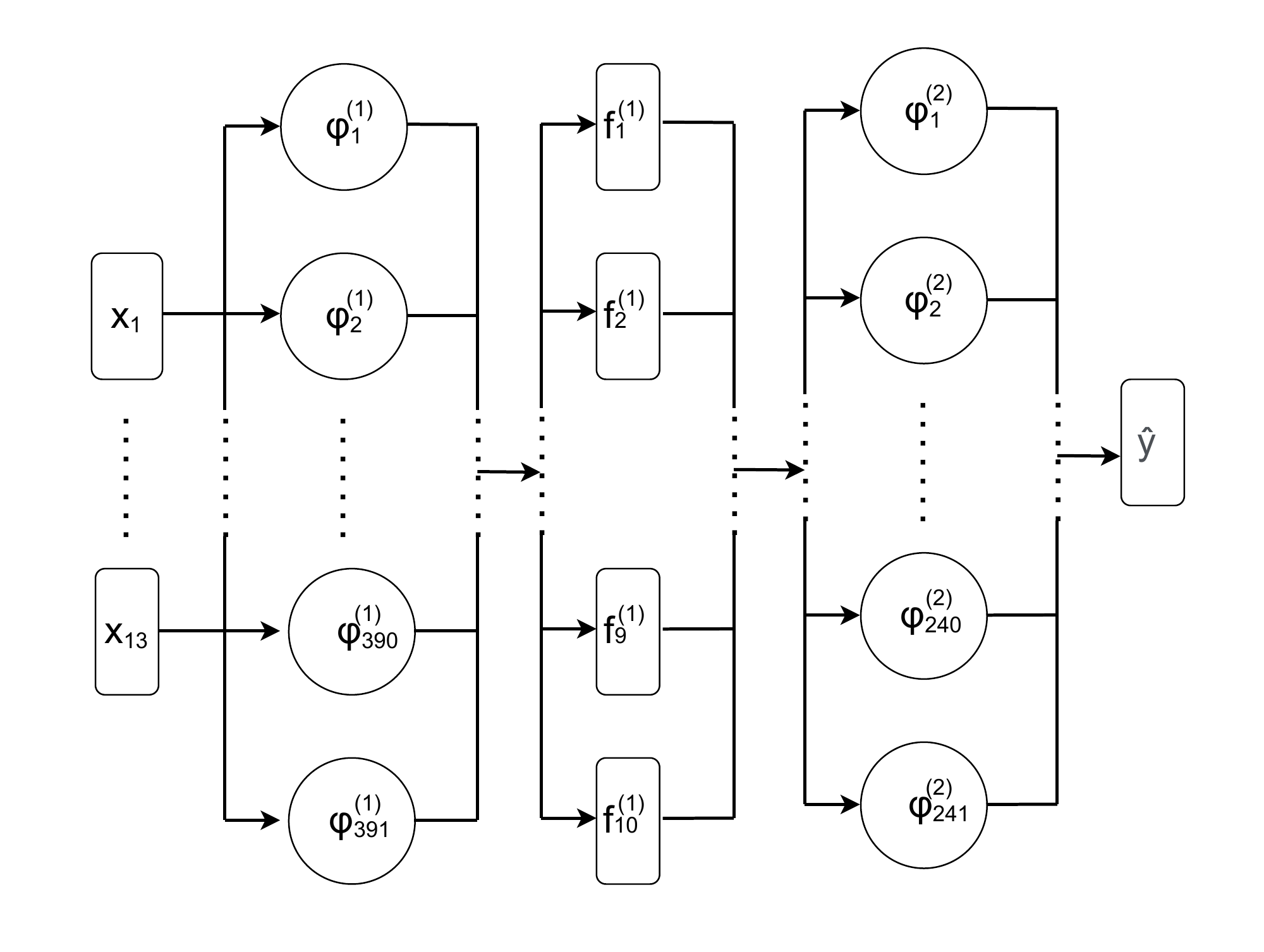} 
\caption{\label{fig:nn_data} Left: structure of SAN. Middle: i.i.d. samples of $C(\BX)$ with $X_i=2.5$, $i\neq 7$ and $X_7\in[0.5,5]$. Right: architecture shared by all competing deep GPs to capture $C(\BX)$. }
\end{figure}

Similar to section \ref{sec:RandomField}, all the competing models in this experiment share the same architecture. All deep models have two hiden layers. The first layer consists of 391 activations and its output is 10-dimensional. The second layer consists of 241 activations and its output is 1-dimensional. Such an architecture is flexible enough to capture the random process. 

\begin{table}[]
    \centering
\begin{tabular}{ |p{2cm}|p{1.6cm}|p{1.6cm}|p{1.6cm}|p{1.6cm}|p{1.6cm}|  }
 \hline
 model &m=4& m=8&m=12&m=16&m=20\\
 \hline
 DTMGP&9.5 min & 17.8 min &  24.3 min & 36.7 min & 46.2 min\\
 \hline
 DGP-VEC&1.8 hr & 3.2 hr &  $>4$ hr & $>4$ hr & $>4$ hr\\
 \hline
\end{tabular}
    \caption{Time Required to Convergence of Training Error}
    \label{tab:runtime_san}
\end{table}

\begin{figure}[ht]
\centering 
\includegraphics[width=0.45\textwidth]{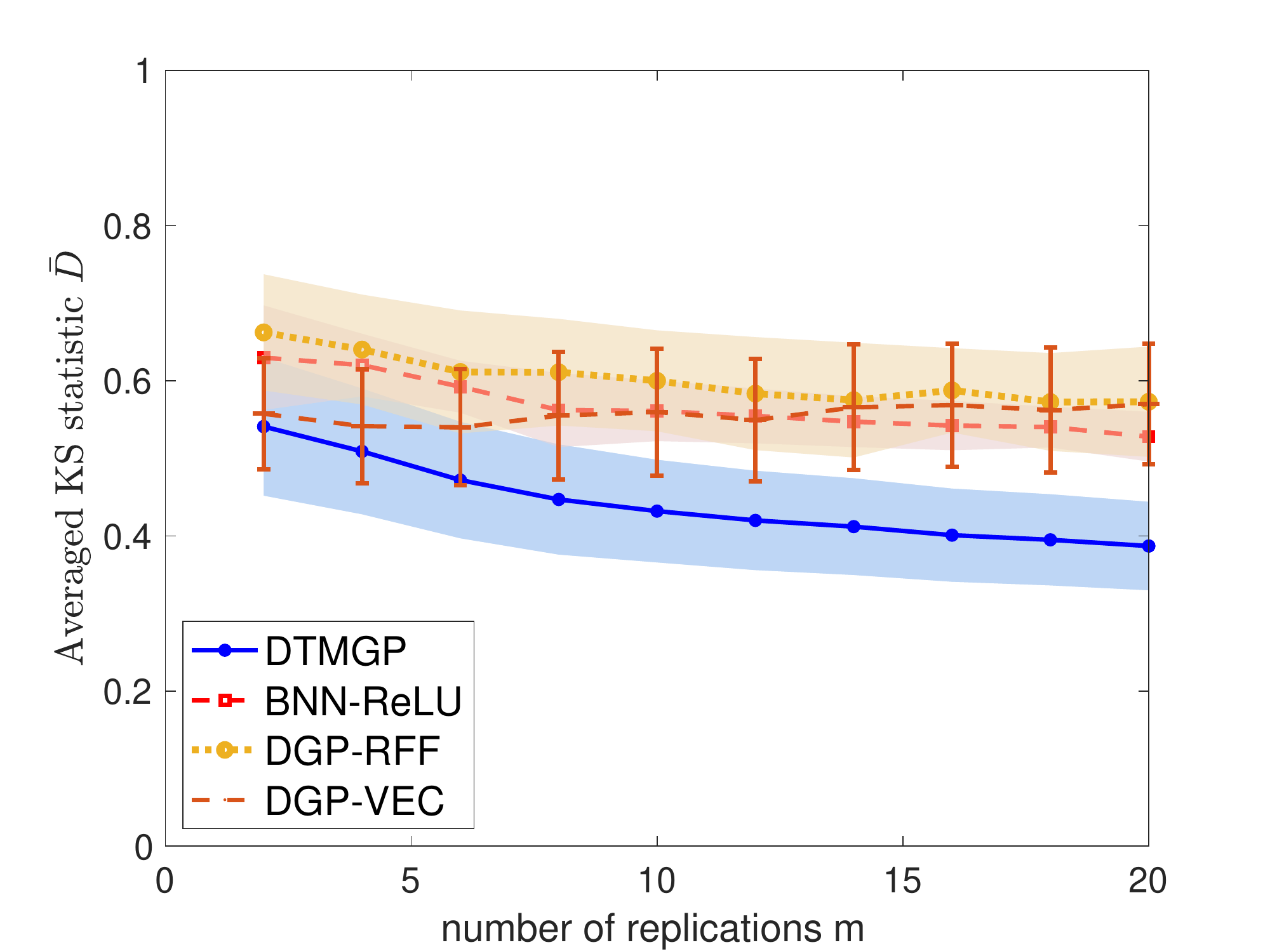} 
\includegraphics[width=0.45\textwidth]{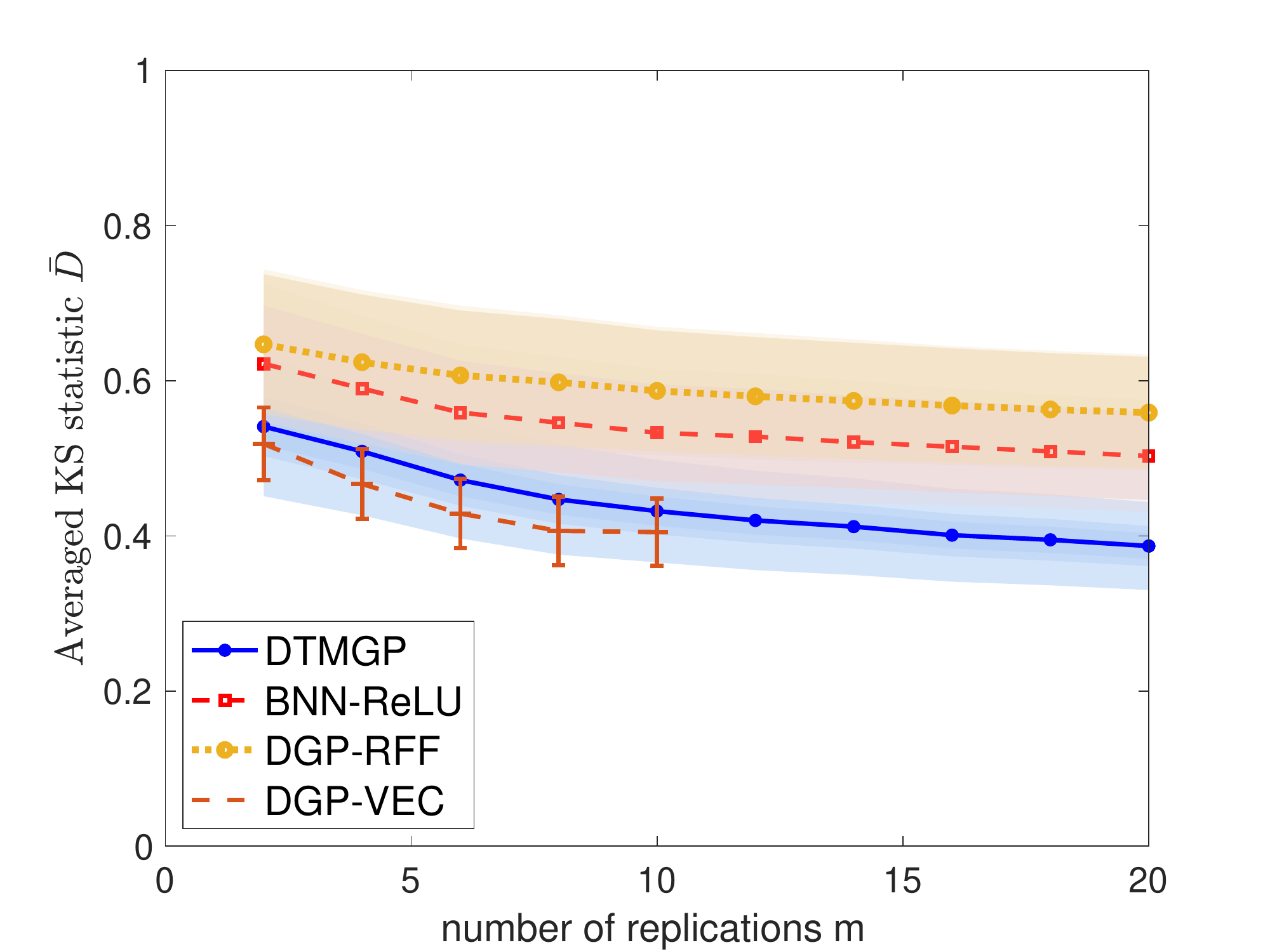} 
\caption{\label{fig:SAN_result}  $\bar{D}$ of each model, the shaded areas and error bars represent 1 std $\hat{\sigma}$ of all models. Left: results reported when KS statistics of DTMGP begins to converge. Right: results reported when KS statistics of DGP-VEC begins to converge.}
\end{figure}

As before, the training time of DGP-VEC is much longer than other models. When $m>10$, the training time required for DGP-VEC is longer than 4 hours, which is unaccetable compared with DTMGP. Therefore, to ensure a fair comparison, we also report our results at different training times.
First, we report the KS statistics of all models when DTMGP’s training error begins to
converge, and then we report the KS statistics of all models when DGP-VEC’s training
error begins to converge or the training time reaches 4 hours.

As demonstrated in Figure \ref{fig:SAN_result}, DGP-VEC outperforms all other models when given sufficient training time. However, DGP-VEC's training time exceeds 4 hours when $m>10$ or when the total training data size is greater than 50000. Therefore, we only present the results of DGP-VEC for $m\leq 10$. On the other hand, DTMGP demonstrates comparable performance, with a training time of roughly one-tenth that required by DGP-VEC. Similarly to Section \ref{sec:RandomField}, when we restrict the training time for all models to stop once the training error of DTMGP starts to stabilize, we observe that DTMGP outperforms all other models, and DGP-VEC performs even worse than BNN-ReLU. This is because the training time is not sufficient for MCMC.

\section{Experiment on Real Data}\label{sec:RC-49}

This section aims to evaluate the performance of the proposed methodology on a real dataset. To this end, we utilize the \textit{RC-49}  dataset \citep{ding2021ccgan}, which is a synthetic collection of 49 3-D chair models rendered at yaw angles ranging from 0\degree to 90\degree. The RC-49 dataset consists of $n=176400$ data pairs $(\BX,\BY)=\{(x_i,\By_i)\}_{i=1}^n$ where, for any $i=1,\ldots,n$, input $x_i \in[0,90)$ is a specific yaw angles and outputs $\By_i$ is a $64\times 64$ image of a chair rendered. The objective is to use the Bayesian formula to train a generative model, which is conditioned on the input variable $x$. When provided with input yaw angle $x$, the trained model should generate a random matrix output that is similar to the data $\BY$ at the same angle. In essence, the generative problem is a stochastic process $Y(x)\in\Real^{64\times 64}$ where $x\in[0,90)$, and our aim is to employ a deep GP model to reconstruct $Y(\cdot)$ from observed data.

  We compare DTMGP with BNN-ReLU and DGP-RFF introduced in Section \ref{sec:numerical}. Other deep GP models are omitted due to their limited capacity for handling the large data size and complex model architecture.  All competing deep GP models share the following architecture. We first apply the embedding technique on  input $x$, which uses a linear transformation $A$ to map the single variable input $x\in[0,90)$ to a 100-dimensional vector $\Bx'\in\Real^{100}$.  We then adopt a four-layer BNN architecture where $\Bx'$ is the input layer, the first hidden layer consists of 201 activations and a 256-dimensional output, the second hidden layer consists of 513 activations and a 512-dimensional output, and the third hidden layer consists of 1025 activations and the final output is a $64\times 64$ random matrix. Specifically, the architecture can be written as follows:
\begin{align*}
    &f^{(1)}_j=\sum_{i=1}^{201}Z_{i,j}^{(1)}\phi_i^{(1)}(A\Bx),\quad j=1,\cdots,256,\ Z_{i,j}^{(1)}\sim\mathcal{N}(0,1), A\in\Real^{100\times 1},\\
    &f^{(2)}_j=\sum_{i=1}^{513}Z_{i,j}^{(2)}\phi^{(2)}_{i}(f^{(1)}_1,\cdots,f^{(1)}_{256}),\quad j=1,\cdots,512,\ Z_{i,j}^{(2)}\sim\mathcal{N}(0,1),\\
    &\hat{Y}_{l,j}=\sum_{i=1}^{1025}Z_{i,l,j}^{(3)}\phi_i^{(3)}(f^{(2)}_1,\cdots,f^{(2)}_{512}),\quad l,j=1,\cdots,28,\ Z_{i,l,j}^{(3)}\sim\mathcal{N}(0,1).
\end{align*}
Figure \ref{fig:RC-49_nn} provides an illustration of the above architecture.

\begin{figure}[t!]
\centering
\includegraphics[width=.8\textwidth]{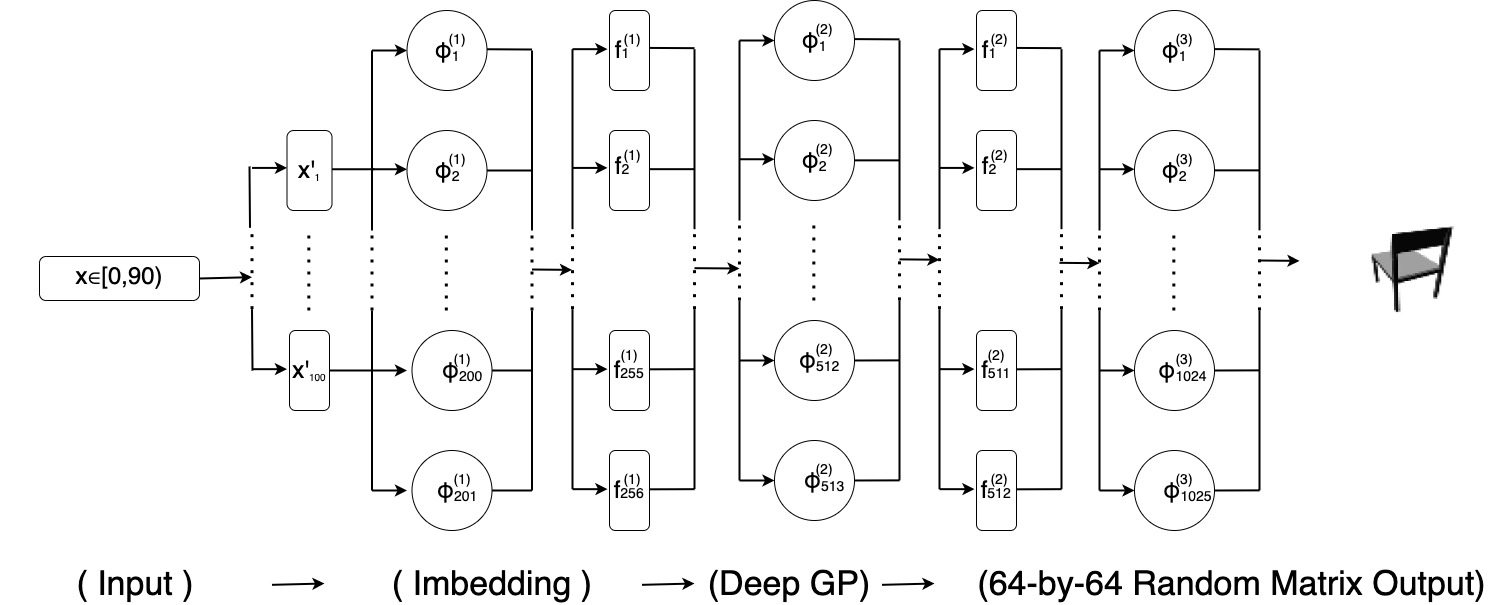}
\caption{ All competing deep GPs share the same architecture. \label{fig:RC-49_nn}}
\end{figure}

We use 80000 samples from the RC-49 dataset to train DTMGP and its competitors BNN-RELU and DGP-RFF. We train the models by mini-batch SGD and, in each epoch, we  sample their random outputs at input $x$, for $x\in[0,90)$.  Because the output in this experiment is a $64\times 64$ random matrix and we are unable to sample data from the true distribution of RC-49, the previous KS statistics cannot be used to access the accuracy. Instead, we directly examine the output images of each model to evaluate their performances.

 \begin{figure}[t!]
\centering
{\includegraphics[width=.6\textwidth]{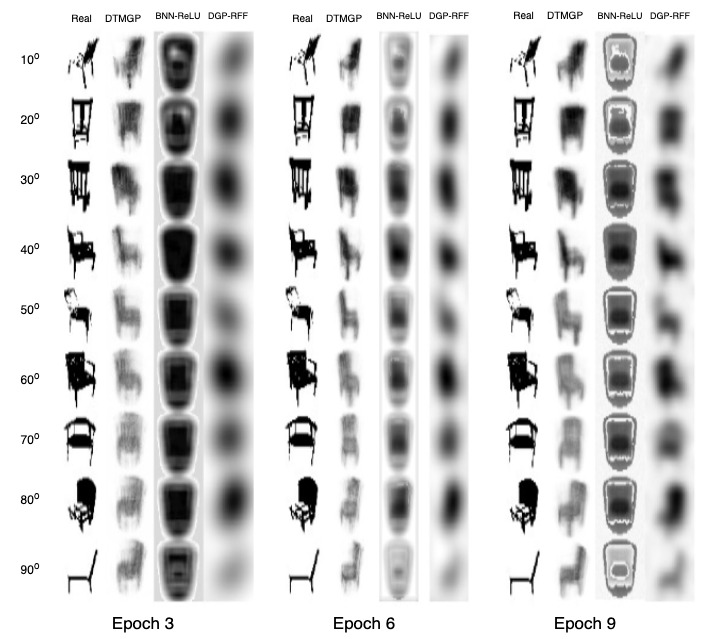}}

\caption{\label{fig:RC-49} Results at the three epoch (left), six epoch (middle), and nine epoch (right) during training with input $x=10,20,\cdots,90$.  }
\end{figure}

\begin{figure}[t!]
\centering
\includegraphics[width=.6\textwidth]{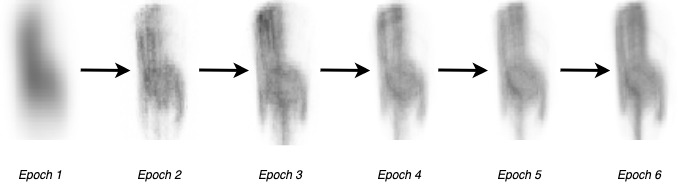}

\caption{\label{fig:RC-49_DTMGP} Generated images of DTMGP with x=60}
\end{figure}

\begin{table}[]
    \centering
\begin{tabular}{ |p{2.4cm}|p{2.4cm}|p{2.4cm}|p{2.4cm}|p{2.4cm}|  }
 \hline
 Model& 1 Epoch & 3 Epoch &6 Epoch &9 Epoch\\
 \hline
DTMGP & 1099.3 sec &  3115.2 sec&  6117.6 sec&  8953.1 sec  \\
 \hline
 BNN-ReLU & 971.5 \ sec& 2511.9 sec & 4997.2 sec  & 7301.1 sec  \\
 \hline
 DGP-RFF & 1081.4 sec  & 2987.7 sec sec & 5823.4 sec & 8682.8 sec \\
 \hline
\end{tabular}
    \caption{Training Time }
    \label{tab:RC-49_run_time}
\end{table}

Conditional samples generated by DTMGP, BNN-ReLU, and DGP-RFF at the $3^{\rm rd}$, $6^{\rm th}$ and $9^{\rm th}$ epochs are shown in  Figure \ref{fig:RC-49}. The results demonstrate that DTMGP generates samples that are closer to the actual data and performs the best, while other models fail to reconstruct the conditional samples. This is because, as observed in previous experiments, the hierarchical structure of DTMGP can effectively capture the local features while the ReLU or cosine activations of BNN-ReLU or DGP-RFF lack the ability to do so.

As depicted in Figure \ref{fig:RC-49_DTMGP}, DTMGP exhibits a faster convergence rate towards stability with more consistent generated results. The rapid convergence suggests that the model can be trained in few epochs, while the consistent performance  implies that the model is reliable. This may be attributed to the sparsity introduced by our expansion approach, which facilitates feature detection and speeds up parameter determination.

Table \ref{tab:RC-49_run_time} demonstrates that the training time per epoch of both BNN-ReLU and DGP-RFF is faster than that of DTMGP. This difference could be primarily attributed to the fact that BNN-ReLU and DGP-RFF are built using built-in functions and packages in the MATLAB, while our code for DTMGP does not take this advantage. Nonetheless, the training time per epoch of DTMGP is still comparable to that of BNN-ReLU and DGP-RFF.

\section{Conclusions and Discussion}\label{sec:conclusion}
We utilize hierarchical features to expand DGPs, which comprise TMGPs. Our expansion is known as DTMGP and is computationally efficient, resulting in a sparse representation of DGPs. The hierarchical nature of the features means that their supports are either nested or disjoint, enabling a poly-logarithmic number of activations throughout a DTMGP to have non-zero output conditioned on any input or operation. Furthermore, hierarchical features can effectively capture local features from inputs, resulting in high performance in prediction and generative problems, as demonstrated in our experiments. In comparison to existing DGP models, DTMGP's sparsity leads to efficient inference and training, while generated instances from DTMGP are also relatively accurate under the KS test.

The current paper can potentially be extended in several ways. Firstly,  how to apply hierarchical expansion on DGPs with more general structures, such as DGPs with M\'atern covariance functions, can be studied in future research. Secondly,  { it was pointed out by \cite{stein2014limitations} that, approximating the likelihood function of GP by low-rank methods, such as inducing point and random features, may have an adverse effect on the performance. How large this effect is for the proposed DGP models should be studied in further investigation. }

\section*{Acknowledgements}
The authors are grateful to the editors and referees for very helpful comments.

\section*{Funding}
Tuo's research is supported by NSF DMS-1914636 and CCF-1934904 and 2022 Texas A\&M Institute of Data Science Career Initiation Fellow Program. Part of Ding's work was conducted when he was a postdoctoral researcher at Texas A\&M University and was partially supported by the Texas A\&M Institute of Data Science (TAMIDS) Postdoctoral Fellowships.
 
\appendix
\begin{center}
    \LARGE\textbf{Appendix}
\end{center}
\section{Sparse Grid Designs with Required Order}\label{sec:sparse-grid}
The class of SG adopted in this work is called \emph{Hyperbolic-cross} SG, which is the union of full grids (FG) with dyadic structures. Without loss of generality, we can assume that all design points are collected from the hypercube $(0,1)^d$, and we start constructing SGs satisfying the required order with the one-dimensional case, which is also called dyadic point set.

A one-dimensional level-$l$ dyadic  point set $\BX_l$ with increasing order is simply defined as $\BX_l=\{2^{-l},2\cdot2^{-l},\cdots,1-2^{-l}\}$. However, this order cannot result in a sparse representation of hierarchical expansion. We need to sort the points according to the their levels. Firstly, given a level $l\in\NatInt$, we define the following set consisting only of odd numbers:
\[\Brho(l)=\{1, 3,5,\cdots,2^l-1\}.\] 
Then given a level-$l$ dyadic  point set $\BX_{l}$ and a level-$(l-1)$ dyadic  point set  $\BX_{l-1}$ , we can define the sorted incremental set 
\[\BD_l=\BX_{l}-\BX_{l-1}=\big\{i 2^{-l}: i\in\Brho(l)\big\}, \] 
with $\BD_0=\emptyset$. It is straightforward to check that
\begin{align*}
    \BX_l=\bigcup_{\ell=1}^l\BD_{\ell},\quad \BD_l\bigcap\BD_k=\emptyset,\ \text{if}\ l\neq k.
\end{align*}
For any incremental set $\BD_l$, we can label any point in $\BD_l$ by $x_{l,i}$, $i\in\rho(l)$ so that the label is unique for any point: $x_{l,i}\neq x_{l',i'}$ if and only if $(l,i)\neq (l',i')$. Now we can define the following dyadic set with the required order:
\[\BX_l^*=\big[\BD_1,\BD_2,\cdots,\BD_l\big]=[x_{\ell,i}:\ell\leq l,i\in\Brho(\ell)].\]
For example, $\BX^*_3=[\frac{1}{2},\frac{1}{4},\frac{3}{4},\frac{1}{8},\frac{3}{8},\frac{5}{8},\frac{7}{8}]$ with $\BD_1=[\frac{1}{2}]$, $\BD_2=[\frac{1}{4},\frac{3}{4}]$, and $\BD_3=[\frac{1}{8},\frac{3}{8},\frac{5}{8},\frac{7}{8}]$, respectively.

We also define the Cartesian product of one-dimensional sorted $\BX^*_l$'s for later use in Algorithms and proofs. We call this sorted  Cartesian product full grid (FG) $\BX^*_{\Bl}$ labeled by $\Bl=(l_1,\cdots,l_d)\in\NatInt^d$:
\[\BX_{\Bl}^*=\bigtimes_{j=1}^d\BX^*_{l_j}=\big[(x_{\ell_1,i_1},x_{\ell_2,i_2},\cdots,x_{\ell_d,i_d}): {\ell_j\leq l_j}, i_j\in\Brho(\ell_j),j=1,\cdots,d\big].\]
For simplicity, we define the Cartesian product $\Brho(\Bl)=\times_{j=1}^d\Brho(l_j)$ for any $\Bl\in\NatInt^d$ so that any points in $\BX^*_{\Bl}$ can be represented as $\Bx_{\Bl,\Bi}=(x_{l_1,i_1},x_{l_2,i_2},\cdots,x_{l_d,i_d})$ for any $\Bl\in\NatInt^d$ and $\Bi\in\Brho(\Bl)$. Then we have a more compact representation of $\BX^*_{\Bl}$:
\[\BX_{\Bl}^*=\big[\Bx_{\boldsymbol{\ell},\Bi}: \boldsymbol{\ell}\leq\Bl,\Bi\in\Brho(\bold{\ell})\big].\]
For example, the sorted FG $\BX^*_{(1,2)}$ is
\[\BX^*_{(1,2)}=[   (\frac{1}{2},\frac{1}{2}),(\frac{1}{2},\frac{1}{4}),(\frac{1}{2},\frac{3}{4}) ].\]

Finally, a level-$l$ SG satisfying our required order is the union of sorted FGs as follows:
\[\BX^{\rm SG}_l=\bigcup_{|\Bl|=l+d-1}\BX^*_\Bl=\big[\Bx_{\Bl,\Bi}: |\Bl|\leq l,\Bi\in\Brho(\Bl)\big],\]
where $|\Bl|$ denote the $l_1$ norm $\sum_{j=1}^d|l_j|$. To be more specific, a point $\Bx_{\Bl,\Bi}\in\BX^{\rm SG}_l$ must be ahead of any $\Bx_{\Bl'}$ with $|\Bl'|>|\Bl|$. For example, the two-dimensional sorted SG $\BX^{\rm SG}_2$ is
\[\BX^{\rm SG}_2=\big[\underbrace{(\frac{1}{2},\frac{1}{2})}_{\text{level}\ |(1,1)|=2},\underbrace{(\frac{1}{2},\frac{1}{4}),(\frac{1}{2},\frac{3}{4})}_{\text{level}\ |(1,2)|=3},\underbrace{(\frac{1}{4},\frac{1}{2}),(\frac{3}{4},\frac{1}{2}}_{\text{level}\ |(2,1)|=3}) \big].\]

Unlike FG, whose size increases exponentially in dimension $d$, the size of SGs increases relatively mildly in $d$. Lemma~3.6 in \cite{bungartz_griebel_2004} stipulates that the sample size of a $d$-dimensional level-$l$ $\BX_{l}^{\rm SG}$ is given by
\begin{equation}\label{eq:SSGNumpt}
    |\BX_{l}^{\rm SG}|=\sum_{\ell=0}^{l-1}2^\ell \binom{\ell+d-1}{d-1} = 2^l\cdot\left(\frac{l^{d-1}}{(d-1)!}+\CalO(l^{d-2})\right)=\CalO( 2^l l^{d-1}).
\end{equation}

\section{Algorithms}\label{sec:algorithm}
We first introduce how to compute $R^{-1}$, with $R$ the Cholesky decomposition of $k(\BX^*_l,\BX^*_l)$, $k$ is a one-dimensional Markov kernel and $\BX^*_l$ is a one-dimensional sorted level-$l$ dyadic point set. To present the algorithm for computing the inverse of Cholesky decomposition, we also label the entries on matrices by points in $\BX^{*}_l$. For example, $[R^{-1}]_{x,x'}$ represents entry with row index corresponding to point $x\in\BX^*_l $ and column index corresponding to point $x'\in\BX^*_l$.

\begin{algorithm}[ht]
\SetAlgoLined
\DontPrintSemicolon
\SetKwInOut{Input}{Input}
\SetKwInOut{Output}{Output}
\Input{Markov kernel $k$,  level-$l$ dyadic points  $\BX_l^*$ }
\Output{$R^{-1}$}
\BlankLine
Initialize $R^{-1}\leftarrow \bold{0} \in\Real^{(2^l-1)\times (2^l-1)}$, $\bold{N}=\{-\infty,\infty\}$, define $k(\pm\infty,x)=0,\forall x$\;
\For{$\ell \leftarrow 1$ \KwTo $l$}{
    \For{$i \in\Brho(\ell) $}{
        search the closest left neighbor $x_{left}$ and right neighbor $x_{right}$ of $x_{\ell,i}$ in $\bold{N}$\;
        Solve $c_1,c_2$, and $c_3$ for the following system:
        \vspace{-2ex}
        \begin{equation}
            \label{eq:1-D-cholesky-system}
        \begin{aligned}
            &c_1k(x_{left},x_{left})+c_2k(x_{\ell,i},x_{left})+c_3k(x_{right},x_{left})=0,\\
            &c_1k(x_{left},x_{right})+c_2k(x_{\ell,i},x_{right})+c_3k(x_{right},x_{right})=0,\\
            &\mathbb{E}\big[\big(c_1\CalG(x_{left})+c_2\CalG(x_{\ell,i})+c_3\CalG(x_{right})\big)^2\big]=1
        \end{aligned}
        \end{equation}\;
        \vspace{-6ex}
        \textbf{If} $x_{left}\neq -\infty$, $[R^{-1}]_{x_{left},x_{\ell,i}}=c_1$\; 
        \textbf{If} $x_{right}\neq \infty$, $[R^{-1}]_{x_{right},x_{\ell,i}}=c_3$\;
        Let $[R^{-1}]_{x_{\ell,i},x_{\ell,i}}=c_2$, $x_{\ell,i}\rightarrow \bold{N} $

    }
}
Return $R^{-1}$ \;
\caption{Computing $R^{-1}$ for Markov kernel $k$ and point set $\BX_l^*$}\label{alg:Cholesky-1D}
\end{algorithm}

Now let $k=\prod_{j=1}^dk_j$ be a $d$-dimensional TMK and $\BX^*_\Bl=\bigtimes_{j=1}^d \BX^*_{l_j}$ be a sorted FG. Then, $R^{-1}_{\Bl}$, with $R_{\Bl}$ the Cholesky decomposition of $k(\BX^*_\Bl,\BX^*_\Bl)$ can be directly calculated as
\begin{equation}
    \label{eq:cholesky-tensor}
    R_{\Bl}^{-1}=\bigotimes_{j=1}^dR^{-1}_{l_j},
\end{equation}
where $\bigotimes$ denotes the Kronecker product between matrices and $R_{l_j}^{-1}$ is the inverse Cholesky decomposition yielded by Algorithm \ref{alg:Cholesky-1D} with input $k_j$ and $\BX^*_{l_j}$.

 Let $R_{\BX^*_\Bl,\BX^*_\Bl}$ represent a sub-matrix of $R$ consisting of entries $R_{\Bx,\By}$, $\Bx,\By\in \BX^*_\Bl\subset \BX^{\rm SG}_l$. Now we can present the algorithm of constructing $R_l^{-1}$
\begin{algorithm}[ht]
\SetAlgoLined
\DontPrintSemicolon
\SetKwInOut{Input}{Input}
\SetKwInOut{Output}{Output}
\Input{TMK $k$,  level-$l$ SG  $\BX_l^{\rm{SG}}$ }
\Output{$R_l^{-1}$}
\BlankLine
Initialize $R^{-1}\leftarrow \bold{0} \in\Real^{m_l\times m_l}$\;
\For{all $\Bl\in\NatInt^d$ with $ l\leq |\Bl| \leq l+d-1$ }{
Compute $R_{\Bl}^{-1}$ associated  to $(k,\BX^*_\Bl)$  via Algorithm~\ref{alg:Cholesky-1D} and \eqref{eq:cholesky-tensor}\;
Update $R_l^{-1}$ via
\vspace{-2ex}
\begin{align}
[R_{l}^{-1}]_{\BX^*_{\Bl},\BX^*_{\Bl}} \leftarrow{}& [R_{l}^{-1}]_{\BX^*_{\Bl},\BX^*_{\Bl}} + (-1)^{l+d-1-|{\Bl}|} \binom{d-1}{l+d-1-|{\Bl}|} R_{\Bl}^{-1} \label{eq:update-Rl}
\end{align}
\vspace{-2ex}
}
Return $R_l^{-1}$ \;
\caption{Computing $R_l^{-1}$ for TMK $k$ and level-$l$ SG $\BX^{\rm SG}_l$\label{alg:Cholesky-SG}}
\end{algorithm}

We will prove the correctness of Algorithm~\ref{alg:Cholesky-1D} and  Algorithm~\ref{alg:Cholesky-SG} in the supplementary material.

\bibliographystyle{chicago}
\spacingset{1}
\bibliography{IISE-Trans}

\begin{thebibliography}{}

\bibitem[\protect\citeauthoryear{Ankenman, Nelson, and Staum}{Ankenman
  et~al.}{2010}]{AnkenmanNelsonStaum10}
Ankenman, B., B.~L. Nelson, and J.~Staum (2010).
\newblock Stochastic kriging for simulation metamodeling.
\newblock {\em Oper. Res.\/}~{\em 58\/}(2), 371--382.

\bibitem[\protect\citeauthoryear{Avramidis and Wilson}{Avramidis and
  Wilson}{1996}]{avramidis1996integrated}
Avramidis, A.~N. and J.~R. Wilson (1996).
\newblock Integrated variance reduction strategies for simulation.
\newblock {\em Operations Research\/}~{\em 44\/}(2), 327--346.

\bibitem[\protect\citeauthoryear{Banerjee, Gelfand, Finley, and Sang}{Banerjee
  et~al.}{2008}]{banerjee2008gaussian}
Banerjee, S., A.~E. Gelfand, A.~O. Finley, and H.~Sang (2008).
\newblock Gaussian predictive process models for large spatial data sets.
\newblock {\em Journal of the Royal Statistical Society: Series B (Statistical
  Methodology)\/}~{\em 70\/}(4), 825--848.

\bibitem[\protect\citeauthoryear{Baydin, Pearlmutter, Radul, and
  Siskind}{Baydin et~al.}{2018}]{Atilim18}
Baydin, A.~G., B.~A. Pearlmutter, A.~A. Radul, and J.~M. Siskind (2018).
\newblock Automatic differentiation in machine learning: {A} survey.
\newblock {\em Journal of Machine Learning Research\/}~{\em 18\/}(153), 1--43.

\bibitem[\protect\citeauthoryear{Blei, Kucukelbir, and McAuliffe}{Blei
  et~al.}{2017}]{Blei_2017}
Blei, D.~M., A.~Kucukelbir, and J.~D. McAuliffe (2017, Apr).
\newblock Variational inference: A review for statisticians.
\newblock {\em Journal of the American Statistical Association\/}~{\em
  112\/}(518), 859–877.

\bibitem[\protect\citeauthoryear{Blundell, Cornebise, Kavukcuoglu, and
  Wierstra}{Blundell et~al.}{2015}]{blundell2015weight}
Blundell, C., J.~Cornebise, K.~Kavukcuoglu, and D.~Wierstra (2015).
\newblock Weight uncertainty in neural network.
\newblock In {\em International Conference on Machine Learning}, pp.\
  1613--1622. PMLR.

\bibitem[\protect\citeauthoryear{Bui, Hernandez-Lobato, Hernandez-Lobato, Li,
  and Turner}{Bui et~al.}{2016}]{Bui16}
Bui, T., D.~Hernandez-Lobato, J.~Hernandez-Lobato, Y.~Li, and R.~Turner (2016,
  20--22 Jun).
\newblock Deep {Gaussian} processes for regression using approximate
  expectation propagation.
\newblock Volume~48 of {\em Proceedings of Machine Learning Research}, New
  York, New York, USA, pp.\  1472--1481. PMLR.

\bibitem[\protect\citeauthoryear{Bungartz and Griebel}{Bungartz and
  Griebel}{2004}]{bungartz_griebel_2004}
Bungartz, H.-J. and M.~Griebel (2004).
\newblock Sparse grids.
\newblock {\em Acta Numerica\/}~{\em 13}, 147–269.

\bibitem[\protect\citeauthoryear{Cressie and Johannesson}{Cressie and
  Johannesson}{2008}]{cressie2008fixed}
Cressie, N. and G.~Johannesson (2008).
\newblock Fixed rank kriging for very large spatial data sets.
\newblock {\em Journal of the Royal Statistical Society: Series B (Statistical
  Methodology)\/}~{\em 70\/}(1), 209--226.

\bibitem[\protect\citeauthoryear{Cutajar, Bonilla, Michiardi, and
  Filippone}{Cutajar et~al.}{2017}]{Cutajar17}
Cutajar, K., E.~V. Bonilla, P.~Michiardi, and M.~Filippone (2017, 06--11 Aug).
\newblock Random feature expansions for deep {G}aussian processes.
\newblock Volume~70 of {\em Proceedings of Machine Learning Research},
  International Convention Centre, Sydney, Australia, pp.\  884--893. PMLR.

\bibitem[\protect\citeauthoryear{Dai, Xie, He, Liang, Raj, Balcan, and
  Song}{Dai et~al.}{2014}]{Dai14}
Dai, B., B.~Xie, N.~He, Y.~Liang, A.~Raj, M.-F.~F. Balcan, and L.~Song (2014).
\newblock Scalable kernel methods via doubly stochastic gradients.
\newblock In Z.~Ghahramani, M.~Welling, C.~Cortes, N.~D. Lawrence, and K.~Q.
  Weinberger (Eds.), {\em Advances in Neural Information Processing Systems
  27}, pp.\  3041--3049. Curran Associates, Inc.

\bibitem[\protect\citeauthoryear{Dai, Damianou, González, and Lawrence}{Dai
  et~al.}{2016}]{dai2016variational}
Dai, Z., A.~Damianou, J.~González, and N.~Lawrence (2016).
\newblock Variational auto-encoded deep {Gaussian} processes.
\newblock {\em Proceedings of the International Conference on Learning
  Representations (ICLR) 2016\/}.

\bibitem[\protect\citeauthoryear{Damianou and Lawrence}{Damianou and
  Lawrence}{2013}]{Damianou13}
Damianou, A. and N.~Lawrence (2013).
\newblock Deep {Gaussian} processes.
\newblock Volume~31 of {\em Proceedings of Machine Learning Research},
  Scottsdale, Arizona, USA, pp.\  207--215. PMLR.

\bibitem[\protect\citeauthoryear{Ding and Zhang}{Ding and
  Zhang}{2020}]{ding2020sample}
Ding, L. and X.~Zhang (2020).
\newblock Sample and computationally efficient simulation metamodeling in high
  dimensions.
\newblock {\em arXiv preprint arXiv:2010.06802\/}.

\bibitem[\protect\citeauthoryear{Ding, Wang, Xu, Welch, and Wang}{Ding
  et~al.}{2021}]{ding2021ccgan}
Ding, X., Y.~Wang, Z.~Xu, W.~J. Welch, and Z.~J. Wang (2021).
\newblock Ccgan: Continuous conditional generative adversarial networks for
  image generation.
\newblock In {\em International conference on learning representations}.

\bibitem[\protect\citeauthoryear{Fei, Zhao, Sun, and Liu}{Fei
  et~al.}{2018}]{fei2018active}
Fei, J., J.~Zhao, S.~Sun, and Y.~Liu (2018).
\newblock Active learning methods with deep {Gaussian} processes.
\newblock In {\em International Conference on Neural Information Processing},
  pp.\  473--483. Springer.

\bibitem[\protect\citeauthoryear{Glorot, Bordes, and Bengio}{Glorot
  et~al.}{2011}]{Glorot11}
Glorot, X., A.~Bordes, and Y.~Bengio (2011, 11--13 Apr).
\newblock Deep sparse rectifier neural networks.
\newblock Volume~15 of {\em Proceedings of Machine Learning Research}, Fort
  Lauderdale, FL, USA, pp.\  315--323. JMLR Workshop and Conference
  Proceedings.

\bibitem[\protect\citeauthoryear{Graves}{Graves}{2011}]{Graves11}
Graves, A. (2011).
\newblock Practical variational inference for neural networks.
\newblock In J.~Shawe-Taylor, R.~S. Zemel, P.~L. Bartlett, F.~Pereira, and
  K.~Q. Weinberger (Eds.), {\em Advances in Neural Information Processing
  Systems 24}, pp.\  2348--2356. Curran Associates, Inc.

\bibitem[\protect\citeauthoryear{Havasi, Hern\'{a}ndez-Lobato, and
  Murillo-Fuentes}{Havasi et~al.}{2018}]{Havasi18}
Havasi, M., J.~M. Hern\'{a}ndez-Lobato, and J.~J. Murillo-Fuentes (2018).
\newblock Inference in deep {Gaussian} processes using {Stochastic Gradient
  Hamiltonian Monte Carlo}.
\newblock In S.~Bengio, H.~Wallach, H.~Larochelle, K.~Grauman, N.~Cesa-Bianchi,
  and R.~Garnett (Eds.), {\em Advances in Neural Information Processing
  Systems}, Volume~31. Curran Associates, Inc.

\bibitem[\protect\citeauthoryear{Hebbal, Brevault, Balesdent, Talbi, and
  Melab}{Hebbal et~al.}{2021}]{hebbal2021bayesian}
Hebbal, A., L.~Brevault, M.~Balesdent, E.-G. Talbi, and N.~Melab (2021).
\newblock {Bayesian} optimization using deep {Gaussian} processes with
  applications to aerospace system design.
\newblock {\em Optimization and Engineering\/}~{\em 22\/}(1), 321--361.

\bibitem[\protect\citeauthoryear{Hensman and Lawrence}{Hensman and
  Lawrence}{2014}]{hensman2014nested}
Hensman, J. and N.~D. Lawrence (2014).
\newblock Nested variational compression in deep {Gaussian} processes.

\bibitem[\protect\citeauthoryear{Hinton}{Hinton}{2010}]{Hinton10}
Hinton, G. (2010).
\newblock A practical guide to training restricted {Boltzmann} machines
  (version 1).
\newblock {\em Technical Report UTML TR 2010-003, University of Toronto\/}~{\em
  9}.

\bibitem[\protect\citeauthoryear{Hinton, Osindero, and Teh}{Hinton
  et~al.}{2006}]{Hinton06}
Hinton, G., S.~Osindero, and Y.-W. Teh (2006, 08).
\newblock A fast learning algorithm for deep belief nets.
\newblock {\em Neural Computation\/}~{\em 18}, 1527--54.

\bibitem[\protect\citeauthoryear{Katzfuss and Guinness}{Katzfuss and
  Guinness}{2021}]{katzfuss2021general}
Katzfuss, M. and J.~Guinness (2021).
\newblock A general framework for vecchia approximations of gaussian processes.
\newblock {\em Statistical Science\/}~{\em 36\/}(1), 124--141.

\bibitem[\protect\citeauthoryear{Ko and Kim}{Ko and Kim}{2021}]{ko2021deep}
Ko, J. and H.~Kim (2021).
\newblock Deep {Gaussian} process models for integrating multifidelity
  experiments with non-stationary relationships.
\newblock {\em IISE Transactions\/}~(just-accepted), 1--28.

\bibitem[\protect\citeauthoryear{Li, Rao, Hassaine, Ramakrishnan, Canoy,
  Salimi-Khorshidi, Mamouei, Lukasiewicz, and Rahimi}{Li
  et~al.}{2021}]{li2021deep}
Li, Y., S.~Rao, A.~Hassaine, R.~Ramakrishnan, D.~Canoy, G.~Salimi-Khorshidi,
  M.~Mamouei, T.~Lukasiewicz, and K.~Rahimi (2021).
\newblock Deep {Bayesian} {Gaussian} processes for uncertainty estimation in
  electronic health records.
\newblock {\em Scientific Reports\/}~{\em 11\/}(1), 1--13.

\bibitem[\protect\citeauthoryear{MacKay}{MacKay}{1992}]{mackay1992practical}
MacKay, D.~J. (1992).
\newblock A practical {Bayesian} framework for backpropagation networks.
\newblock {\em Neural Computation\/}~{\em 4\/}(3), 448--472.

\bibitem[\protect\citeauthoryear{Marcus and Rosen}{Marcus and
  Rosen}{2006}]{MarcusRosen06}
Marcus, M.~B. and J.~Rosen (2006).
\newblock {\em {Markov} Processes, {Gaussian} Processes, and Local Times}.
\newblock Cambridge University Press.

\bibitem[\protect\citeauthoryear{Marmin and Filippone}{Marmin and
  Filippone}{2022}]{marmin2022deep}
Marmin, S. and M.~Filippone (2022).
\newblock Deep gaussian processes for calibration of computer models.
\newblock {\em Bayesian Analysis\/}~{\em 1\/}(1), 1--30.

\bibitem[\protect\citeauthoryear{Matheron}{Matheron}{1963}]{Matheron63}
Matheron, G. (1963).
\newblock Principles of geostatistics.
\newblock {\em Econ. Geol.\/}~{\em 58\/}(8), 1246--1266.

\bibitem[\protect\citeauthoryear{Neal}{Neal}{1996}]{neal1996bayesian}
Neal, R.~M. (1996).
\newblock {\em {Bayesian} learning for neural networks}, Volume 118.
\newblock Springer Science \& Business Media.

\bibitem[\protect\citeauthoryear{Neidinger}{Neidinger}{2010}]{neidinger2010introduction}
Neidinger, R.~D. (2010).
\newblock Introduction to automatic differentiation and matlab object-oriented
  programming.
\newblock {\em SIAM review\/}~{\em 52\/}(3), 545--563.

\bibitem[\protect\citeauthoryear{Pasupathy and Henderson}{Pasupathy and
  Henderson}{2011}]{Simopt}
Pasupathy, R. and S.~G. Henderson (2011).
\newblock Simopt: A library of simulation optimization problems.
\newblock In {\em Proceedings of the 2011 Winter Simulation Conference (WSC)},
  pp.\  4075--4085.

\bibitem[\protect\citeauthoryear{Plumlee}{Plumlee}{2014}]{Plumlee14}
Plumlee, M. (2014).
\newblock Fast prediction of deterministic functions using sparse grid
  experimental designs.
\newblock {\em J. Amer. Statist. Assoc.\/}~{\em 109\/}(508), 1581--1591.

\bibitem[\protect\citeauthoryear{Plumlee}{Plumlee}{2021}]{SGDesign}
Plumlee, M. (2021).
\newblock Sparse grid designs.
\newblock
  \url{https://www.mathworks.com/matlabcentral/fileexchange/45668-sparse-grid-designs}.
\newblock MATLAB Central File Exchange.

\bibitem[\protect\citeauthoryear{Plumlee and Tuo}{Plumlee and
  Tuo}{2014}]{plumlee2014building}
Plumlee, M. and R.~Tuo (2014).
\newblock Building accurate emulators for stochastic simulations via quantile
  kriging.
\newblock {\em Technometrics\/}~{\em 56\/}(4), 466--473.

\bibitem[\protect\citeauthoryear{Radaideh and Kozlowski}{Radaideh and
  Kozlowski}{2020}]{radaideh2020surrogate}
Radaideh, M.~I. and T.~Kozlowski (2020).
\newblock Surrogate modeling of advanced computer simulations using deep
  {Gaussian} processes.
\newblock {\em Reliability Engineering \& System Safety\/}~{\em 195}, 106731.

\bibitem[\protect\citeauthoryear{Rahimi and Recht}{Rahimi and
  Recht}{2008}]{rahimi2008random}
Rahimi, A. and B.~Recht (2008).
\newblock Random features for large-scale kernel machines.
\newblock In {\em Advances in Neural Information Processing Systems}, pp.\
  1177--1184.

\bibitem[\protect\citeauthoryear{Ritter}{Ritter}{2000}]{ritter2000average}
Ritter, K. (2000).
\newblock {\em Average-case analysis of numerical problems}.
\newblock Number 1733. Springer Science \& Business Media.

\bibitem[\protect\citeauthoryear{Sacks, Schiller, and Welch}{Sacks
  et~al.}{1989}]{sacks1989designs}
Sacks, J., S.~B. Schiller, and W.~J. Welch (1989).
\newblock Designs for computer experiments.
\newblock {\em Technometrics\/}~{\em 31\/}(1), 41--47.

\bibitem[\protect\citeauthoryear{Salimbeni and Deisenroth}{Salimbeni and
  Deisenroth}{2017}]{Salimbeni17}
Salimbeni, H. and M.~Deisenroth (2017).
\newblock Doubly stochastic variational inference for deep {Gaussian}
  processes.
\newblock In I.~Guyon, U.~V. Luxburg, S.~Bengio, H.~Wallach, R.~Fergus,
  S.~Vishwanathan, and R.~Garnett (Eds.), {\em Advances in Neural Information
  Processing Systems}, Volume~30. Curran Associates, Inc.

\bibitem[\protect\citeauthoryear{Sauer, Cooper, and Gramacy}{Sauer
  et~al.}{2022a}]{VecchiaDGP}
Sauer, A., A.~Cooper, and R.~B. Gramacy (2022a).
\newblock Vecchia-approximated deep gaussian processes for computer
  experiments.

\bibitem[\protect\citeauthoryear{Sauer, Cooper, and Gramacy}{Sauer
  et~al.}{2022b}]{sauer2022vecchia}
Sauer, A., A.~Cooper, and R.~B. Gramacy (2022b).
\newblock Vecchia-approximated deep gaussian processes for computer
  experiments.
\newblock {\em Journal of Computational and Graphical Statistics\/}, 1--14.

\bibitem[\protect\citeauthoryear{Sauer, Gramacy, and Higdon}{Sauer
  et~al.}{2022}]{sauer2020active}
Sauer, A., R.~B. Gramacy, and D.~Higdon (2022).
\newblock Active learning for deep gaussian process surrogates.
\newblock {\em Technometrics\/}, 1--15.

\bibitem[\protect\citeauthoryear{Schmidt, Kraaijveld, and Duin}{Schmidt
  et~al.}{1992}]{schmidt1992feedforward}
Schmidt, W.~F., M.~A. Kraaijveld, and R.~P. Duin (1992).
\newblock Feedforward neural networks with random weights.
\newblock In {\em Pattern Recognition, 1992. Vol. II. Conference B: Pattern
  Recognition Methodology and Systems, Proceedings., 11th IAPR International
  Conference on}, pp.\  1--4. IEEE.

\bibitem[\protect\citeauthoryear{Sid{\'e}n and Lindsten}{Sid{\'e}n and
  Lindsten}{2020}]{siden2020deep}
Sid{\'e}n, P. and F.~Lindsten (2020).
\newblock Deep {Gaussian} {Markov} random fields.
\newblock In {\em International Conference on Machine Learning}, pp.\
  8916--8926. PMLR.

\bibitem[\protect\citeauthoryear{Stein}{Stein}{2014}]{stein2014limitations}
Stein, M.~L. (2014).
\newblock Limitations on low rank approximations for covariance matrices of
  spatial data.
\newblock {\em Spatial Statistics\/}~{\em 8}, 1--19.

\bibitem[\protect\citeauthoryear{Tran, Ranganath, and Blei}{Tran
  et~al.}{2016}]{Tran16}
Tran, D., R.~Ranganath, and D.~M. Blei (2016).
\newblock The variational {Gaussian} process.
\newblock {\em Proceedings of the International Conference on Learning
  Representations (ICLR) 2016\/}.

\bibitem[\protect\citeauthoryear{van Dam, Husslage, den Hertog, and
  Melissen}{van Dam et~al.}{2007}]{van2007maximin}
van Dam, E.~R., B.~Husslage, D.~den Hertog, and H.~Melissen (2007).
\newblock Maximin {Latin} hypercube designs in two dimensions.
\newblock {\em Oper. Res.\/}~{\em 55\/}(1), 158--169.

\bibitem[\protect\citeauthoryear{Vecchia}{Vecchia}{1988}]{vecchia1988estimation}
Vecchia, A.~V. (1988).
\newblock Estimation and model identification for continuous spatial processes.
\newblock {\em Journal of the Royal Statistical Society: Series B
  (Methodological)\/}~{\em 50\/}(2), 297--312.

\bibitem[\protect\citeauthoryear{Wang, Tuo, and Wu}{Wang
  et~al.}{2020}]{WangTuoWu20}
Wang, W., R.~Tuo, and C.~F.~J. Wu (2020).
\newblock On prediction properties of kriging: Uniform error bounds and
  robustness.
\newblock {\em J. Amer. Statist. Assoc.\/}~{\em 115\/}(530), 920--930.

\bibitem[\protect\citeauthoryear{Yang and Klabjan}{Yang and
  Klabjan}{2021}]{Yang21}
Yang, J. and D.~Klabjan (2021).
\newblock {Bayesian} active learning for choice models with deep {Gaussian}
  processes.
\newblock {\em IEEE Transactions on Intelligent Transportation Systems\/}~{\em
  22\/}(2), 1080--1092.

\end{thebibliography}

\end{document}


		\def\spacingset#1{\renewcommand{\baselinestretch}%
			{#1}\small\normalsize} \spacingset{1}
		
		\if0\blind
		{
			\title{\bf Supplementary Material for ``A Sparse Expansion For Deep Gaussian Processes''}
			\author{Liang Ding $^{a1}$, Rui Tuo $^{a2}$ and Shahin Shahrampour$^b$ \\
			$^a$ Industrial \& Systems Engineering, Texas A\&M University, College Station, TX \\
             $^b$ Mechanical \& Industrial Engineering, Northeastern University, Boston, MA }
			\date{}
			\maketitle
		} \fi
		
		\if1\blind
		{
            \title{\bf Supplementary Material for ``A Sparse Expansion For Deep Gaussian Processes''}

			\date{}
			\maketitle
			
		} \fi
		\bigskip

In the first and second sections, we provide detailed discussions on Algorithm 1 and Algorithm 2, respectively. We then present the proofs for Theorem 1 and Theorem 2 in Section \ref{sec:thm1_pf} and Section \ref{sec:thm2_pf}, respectively. At the end, we discuss the approximation capacity by introducing two more theorems regarding the convergence rates of DTMGP in Section \ref{sec:HE_approx} and Section \ref{sec:DTMGP_approx}, respectively. 

\section{Discussion on Algorithm 1}
Before presenting the essential idea of Algorithm 1 in Appendix B, we first need to prove the following lemma and theorem.
\begin{lemma}\label{lem:compact_supp}
Let $k$ be a TMK defined on interval $U$ and $\CalG$ be the GP generated by $k$. For any $x<y<z$ in $U$, suppose $c_1,c_2,c_3$ satisfies the following conditions:
 \begin{align*}
            &c_1k(x,x)+c_2k(y,x)+c_3k(z,x)=0,\\
            &c_1k(x,z)+c_2k(y,z)+c_3k(z,z)=0,\\
            &\mathbb{E}\big[\big(c_1\CalG(x)+c_2\CalG(y)+c_3\CalG(z)\big)^2\big]=1.
        \end{align*}
Then function $\phi(\cdot)=c_1k(\cdot,x)+c_2k(\cdot,y)+c_3k(\cdot,z)$ is compactly supported on $[x,z]$ and  $\langle \phi,\phi\rangle_{k}=1$ where $\langle\cdot,\cdot\rangle_k$ is the inner product generated by $k$. 
\end{lemma}
\proof
From the definition of TMK in Lemma 1 in the main paper, the first two conditions can be written as
\begin{align*}
    &p(x)\big(c_1q(x)+c_2q(y)+c_3q(z)\big)=0\Rightarrow \quad c_1q(x)+c_2q(y)+c_3q(z) =0\\
    &q(z)\big(c_1p(x)+c_2p(y)+c_3p(z)\big)=0\Rightarrow \quad c_1p(x)+c_2p(y)+c_3p(z)=0.
\end{align*}
So  for any $s<x$
\begin{align*}
     \phi(s)&=c_1p(x\wedge s)q(x\vee s)+c_2p(y\wedge s)q(y\vee s)+c_3p(z\wedge s)q(z\vee s)\\
     &=p(s)\big(c_1q(x)+c_2q(y)+c_3q(z)\big)\\
     &=0,
\end{align*}
and for any $s>z$
\begin{align*}
    \phi(s)=q(s)\big(c_1p(x)+c_2p(y)+c_3p(z)\big)=0.
\end{align*}
This proves that the support of $\phi$ is $[x,z]$.

Let $\bold{C}^T$ denote vector $[c_1,c_2,c_3]$, $\bold{G}^T$ denote vector $[\CalG(x),\CalG(y),\CalG(z))]$, $\boldsymbol{\gamma}^T$ denote vector $[k(\cdot,x),k(\cdot,y),k(\cdot,z)]$ and $\bold{K}$ denote the covariance matrix on $x,y,z$. Then the third condition reads
\begin{align*}
    \mathbb{E}\big[\bold{C}^T\bold{G}\bold{G}^T\bold{C}\big]=\bold{C}^T\bold{K}\bold{C}=1.
\end{align*}
Therefore, we have 
\begin{align*}
    \langle\phi,\phi\rangle_k&=\langle \bold{C}^T\boldsymbol{\gamma}, \boldsymbol{\gamma}^T\bold{C}\rangle_k=\bold{C}^T\bold{K}\bold{C}=1.
\end{align*}
\endproof

We can notice that equation (14) in Algorithm 1 is exactly the three conditions listed in Lemma \ref{lem:compact_supp}. The purpose of Algorithm 1 is, in fact, generate compactly supported orthonormal basis functions in the reproducing kernel Hilbert space \citep{statlearnbook} generated by $k$ as we show in the following theorem:
\begin{theorem}\label{thm:orthor_basis}
The hierarchical features $$\boldsymbol{\phi}=R^{-1}k(\BX^*_l,\cdot)=\big[\phi_{\ell,i}\big]_{x_{\ell,i}\in\BX^*_l}$$ 
where $R^{-1}$ generated by Algorithm 1 are orthonormal under inner product $\langle\cdot,\cdot\rangle_k$:
\[\langle\phi_{\ell,i},\phi_{\ell',i'}\rangle_k=\bold{1}_{\{(\ell,i)=(\ell',i')\}}.\]
\end{theorem}
\proof
This theorem is equivalent to Lemma EC. 4 in \cite{Ding2020sample}. Here, we prove it in a different method which has less calculation. The theorem is a consequence of the way we sort the dyadic point set $\BX^*_l$ and Lemma \ref{lem:compact_supp}. Remind that $\BX^*_l$ has the following multi-resolution structure:
\[\BX^*_l=\bigcup_{\ell\leq l}\bold{D}_\ell=\bigcup_{\ell\leq l}\big\{i 2^{-\ell}: i\in\boldsymbol{\rho}(\ell)\big\}=\bigcup_{\ell\leq l}\big\{x_{\ell,i}: i\in\boldsymbol{\rho}(\ell)\big\}.\]
We first prove the following claim:
\begin{claim}
    \textbf{Supports of functions $\{\phi_{\ell,i}: i\in\boldsymbol{\rho}(\ell)\}$ are mutually disjoint for any fixed $\ell$.}
\end{claim}

 In Algorithm 1, for any pair of points $x_{\ell,i}$ and $x_{\ell',i'}$ with $\ell>\ell'$, $x_{\ell',i'}$ must be  processed ahead of $x_{\ell,i}$  because  the outer for loop will run through all the points in $\bold{D}_{\ell'}$ before processing any point in $D_{\ell}$. Moreover, the left and right neighbors of any $x_{\ell,i}$ in the iteration for processing $x_{\ell,i}$ must be in $\bold{D}_{\ell'}$ and $\bold{D}_{\ell''}$ for some $\ell',\ell''<\ell$. 

In iteration $(\ell,i)$, Lemma \ref{lem:compact_supp} tells that the support of $\phi_{\ell,i}$ is $[x_{left},x_{right}]$ because
\[\phi_{\ell,i}=c_1k(\cdot,x_{left})+c_2k(\cdot,x_{\ell,i})+c_3k(\cdot,x_{right})\]
 where $c_1,c_2,c_3$ solve equation (14) in the iteration for processing $x_{\ell,i}$.  When the outer loop is in the $\ell$ iteration, all points in $\BX^*_{\ell-1}$ have already been added to $\bold{N}$. Remind that if we sort points in $\BX^*_{\ell-1}$ in increasing order, then $\BX^*_{\ell-1}=\{2^{-(\ell-1)},2\cdot2^{-(\ell-1)},\cdots,1-2^{-(\ell-1)}\}$. So we have the following relation regarding the distance between $x_{\ell,i}=i2^{-\ell}$ and  $\BX^*_{\ell-1}$:
 \[\{i2^{-\ell}-2^{-\ell},i2^{+\ell}-2^{-\ell}\}=\arg\min_{x\in\BX^*_{\ell-1}}|x-x_{\ell,i}|=\arg\min_{x\in\BX^*_{\ell-1}}|x-i2^{-\ell}|.\]
 On the other hand, for any other point in $\bold{D}_{\ell}$ , we have:
 \[\min_{x_{\ell,i'}\in\bold{D}_{\ell}}|x_{\ell,i}-x_{\ell,i'}|=\min_{i'\in\Brho(\ell)}|i2^{-\ell}-i'2^{-\ell}|= 2\cdot2^{-\ell}.\]
So both $x_{left}$ and $x_{right}$ are  must be from  $\bold{X}^*_{\ell-1}$ such that
 \[x_{left}=i2^{-\ell}-2^{-\ell},\quad x_{right}=i2^{-\ell}+2^{-\ell}. \]
Therefore, the supports of hierarchical features $\{\phi_{\ell,i}: i\in\boldsymbol{\rho}(\ell)\}$ are
\[\big\{[i2^{-\ell}-2^{-\ell},i2^{-\ell}+2^{-\ell}]:i=1,3,\cdots,2^{\ell}-1\big\}\]
and they are mutually disjoint. We finish the proof of  claim 1.

Disjoint supports of hierarchical features with the same index $\ell$ indicates that they are also mutually orthogonal under $\langle\cdot,\cdot\rangle_k$.

We now prove $\phi_{\ell',i'}$ and $\phi_{\ell,i}$ are also orthogonal for any $\ell'<\ell$. Without loss of generality, we can assume that $x_{\ell,i}<x_{\ell',i'}$. As a direct consequence of the fact that support of $\phi_{\ell,i}$ is $[x_{\ell,i}-2^{-\ell},x_{\ell,i}+2^{-\ell}]$ and support of $\phi_{\ell',i'}$ is $[x_{\ell',i'}-2^{-\ell'},x_{\ell',i'}+2^{-\ell'}]$, we have the following claim:
\begin{claim}
    \textbf{For any $x_{\ell,i}<x_{\ell',i'}$ with $\ell'<\ell$ and $i\in\Brho(\ell)$, $i'\in\Brho(\ell')$, there are only two possible cases regarding the supports of $\phi_{\ell',i'}$ and $\phi_{\ell,i}$:
    \begin{enumerate}[]
        \item[(a)] $[x_{\ell,i}-2^{-\ell},x_{\ell,i}+2^{-\ell}]$ and $[x_{\ell',i'}-2^{-\ell'},x_{\ell',i'}+2^{-\ell'}]$ are disjoint;\label{itr:a}
        \item[(b)] $[x_{\ell,i}-2^{-\ell},x_{\ell,i}+2^{-\ell}]\subset [x_{\ell',i'}-2^{-\ell'},x_{\ell',i'}]$. 
    \end{enumerate}
    }
\end{claim}

For case \textbf{(\textit{a})} supports of $\phi_{\ell',i'}$ and $\phi_{\ell,i}$ are disjoint,  the inner product $\langle\phi_{\ell',i'},\phi_{\ell,i}\rangle_k=0$.

To analyze case \textbf{(\textit{b})}, we first use the following identity:
\[\langle \phi_{\ell',i'},  \phi_{\ell,i}\rangle_k=\bold{C}^T_{\ell',i'}k(\BX_{\ell',i'},\BX_{\ell,i})\bold{C}_{\ell,i}.\]
In the above equation, point set $\BX_{\ell,i}$ consists of the left and right neighbors of $x_{\ell,i}$, and $x_{\ell,i}$ itself. $\bold{C}_{\ell,i} = [c_1,c_2,c_3]^T$  where $c_1,c_2,c_3$ are the solution of equation (14) in the iteration for process $\phi_{\ell,i}$. $\BX_{\ell',i'}$ and $\bold{C}_{\ell',i'}$ are defined in the same manner. 

Because $[x_{\ell,i}-2^{-\ell},x_{\ell,i}+2^{-\ell}]\subset [x_{\ell',i'}-2^{-\ell'},x_{\ell',i'}]$,  and TMK $k(x,y)=p(x\wedge y)q(x\vee y)$, we have:
 \begin{align*}
     &\quad\langle \phi_{\ell,i},  \phi_{\ell',i'}\rangle_k\\
     &=\bold{C}^T_{\ell,i}k(\BX_{\ell,i},\BX_{\ell',i'})\bold{C}_{\ell',i'}\\
     &=\bold{C}^T_{\ell,i}
     \footnotesize{
     \begin{bmatrix}
     p(x_{\ell',i'}-2^{-\ell'})q(x_{\ell,i}-2^{-\ell}) &p(x_{\ell,i}-2^{-\ell})q(x_{\ell',i'}) &p(x_{\ell,i}-2^{-\ell})q(x_{\ell',i'}+2^{-\ell'}) \\
     p(x_{\ell',i'}-2^{-\ell'})q(x_{\ell,i}) &p(x_{\ell,i})q(x_{\ell',i'}) &p(x_{\ell,i})q(x_{\ell',i'}+2^{-\ell'}) \\
     p(x_{\ell',i'}-2^{-\ell'})q(x_{\ell,i}+2^{-\ell}) &p(x_{\ell,i}+2^{-\ell})q(x_{\ell',i'}) & p(x_{\ell,i}+2^{-\ell})q(x_{\ell',i'}+2^{-\ell'})
     \end{bmatrix}
     }
    \bold{C}_{\ell',i'}\\
    &=0
 \end{align*}
where the last line can be derived directly from the first two conditions in Lemma \ref{lem:compact_supp} imposed on $\bold{C}_{\ell,i}$ and $\bold{C}_{\ell',i'}$.

\endproof

Lemma \ref{lem:compact_supp} and Theorem \ref{thm:orthor_basis} show that the matrix $R$ generated by Algorithm 1 is indeed the Cholesky decomposition of  the covariance matrix:
\begin{corollary}
$R^{-1}$  is a upper triangular matrix and $R^{-T}k(\BX^*_l,\BX^*_l)R^{-1}=\bold{I}$.
\end{corollary}
\proof
From our analysis in Theorem \ref{thm:orthor_basis}, we know that in the iteration for processing any $x_{\ell,i}$, $x_{left}$ and $x_{right}$ are from $\bold{D}_{\ell'}$ and $\bold{D}_{\ell''}$ with some $\ell',\ell''<\ell$. According to how we sort $\BX^*_l$, entries $[R^{-1}]_{x_{left},x_{\ell,i}}$, $[R^{-1}]_{x_{\ell,i},x_{\ell,i}}$, and $[R^{-1}]_{x_{right},x_{\ell,i}}$ are in the upper triangular part of $R^{-1}$. In Algorithm 1, operation in iteration for processing $x_{\ell,i}$ is assigning values to three entries of $R^{-1}$, namely, $[R^{-1}]_{x_{left},x_{\ell,i}}$, $[R^{-1}]_{x_{\ell,i},x_{\ell,i}}$, and $[R^{-1}]_{x_{right},x_{\ell,i}}$ and let any other value on column $[R^{-1}]_{:,x_{\ell,i}}$ equal 0. Therefore, the returned matrix $R^{-1}$ is a upper triangular matrix.

From Theorem \ref{thm:orthor_basis}, we know that $\boldsymbol{\phi}=k(\cdot,\BX^*_l)R^{-1}$ are orthonormal basis functions under $\langle\cdot,\cdot\rangle_k$. Therefore,
\[\bold{I}=\langle\boldsymbol{\phi},\boldsymbol{\phi}\rangle_k=R^{-T}k(\BX^*_l,\BX^*_l)R^{-1}.\]
\endproof

\section{Discussion on Algorithm 2}

Algorithm 2 in our main paper is a special case of \cite{Plumlee14}[Algorithm 2].  \cite{Plumlee14}[Algorithm 2] is for computing the matrix $\bold{A}\big[k(\BX_l^{\rm{SG}},\BX_l^{\rm{SG}})\big]^{-1}$  where  $k$ is any kernel in tensor product form $k(\Bx,\By)=\prod_{j=1}^dk_j(x_j,y_j)$ and $\bold{A}$ is any matrix. \cite{Plumlee14}[Algorithm 2]  is written as follows:

\begin{algorithm}[ht]
\SetAlgoLined
\DontPrintSemicolon
\SetKwInOut{Input}{Input}
\SetKwInOut{Output}{Output}
\Input{tensor kernel $k$,  level-$l$ SG  $\BX_l^{\rm{SG}}$, matrix $\bold{A}$ }
\Output{$\bold{A}\big[k(\BX_l^{\rm{SG}},\BX_l^{\rm{SG}})\big]^{-1}$}
\BlankLine
Initialize $\tilde{\bold{A}}\leftarrow \bold{0} \in\Real^{m_l\times m_l}$\;
\For{all $\Bl\in\NatInt^d$ with $ l\leq |\Bl| \leq l+d-1$ }{

Update $\tilde{\bold{A}}$ via
\vspace{-2ex}
\begin{align}
&[\tilde{\bold{A}}]_{:,\BX^*_{\Bl}} \leftarrow{} [\tilde{\bold{A}}]_{:,\BX^*_{\Bl}} + (-1)^{l+d-1-|{\Bl}|} \binom{d-1}{l+d-1-|{\Bl}|} \bold{A}_{:,\BX^*_l}\big[k(\BX^*_{\Bl},\BX^*_{\Bl})\big]^{-1}
\label{eq:plumlee}\tag{P}\\
&\text{\footnotesize{($[\tilde{\bold{A}}]_{:,\BX^*_{\Bl}}$ denotes the matrix with columns that correspond to $\BX^*_{\Bl}$ and all rows of $\bold{A}$)}}\nonumber
\end{align}
\vspace{-2ex}
}
Return $\tilde{\bold{A}}$ \;
\caption{Computing $\bold{A}\big[k(\BX_l^{\rm{SG}},\BX_l^{\rm{SG}})\big]^{-1}$ for tensor kernel $k$, level-$l$ SG $\BX^{\rm SG}_l$, and any matrix $\bold{A}$}
\end{algorithm}

If we let the argument $\bold{A}$ be $R^{-1}_lk(\BX_l^{\rm{SG}},\BX_l^{\rm{SG}})$ then, obviously, the output $\tilde{\bold{A}}$ must be $R^{-1}_l$. To show that Algorithm 2 in our paper is correct, we only need to show that the matrix $\bold{A}_{:,\BX^*_l}\big[k(\BX^*_{\Bl},\BX^*_{\Bl})\big]^{-1}$ in \eqref{eq:plumlee} equals to $R^{-1}_{\Bl}$ on rows associated to $\BX^*_{\Bl}$ and equals 0 on any other row.

If $\bold{A} =R^{-1}_lk(\BX_l^{\rm{SG}},\BX_l^{\rm{SG}})$, then $\bold{A}_{:,\BX^*_{\Bl}}\big[k(\BX^*_{\Bl},\BX^*_{\Bl})\big]^{-1}=R^{-1}_lk(\BX_l^{\rm{SG}},\BX^*_{\Bl})\big[k(\BX^*_{\Bl},\BX^*_{\Bl})\big]^{-1}$. From the definition of hierarchical features (9), we have $$\bold{A}_{:,\BX^*_{\Bl}}\big[k(\BX^*_{\Bl},\BX^*_{\Bl})\big]^{-1}=R^{-1}_lk(\BX_l^{\rm{SG}},\BX^*_{\Bl})\big[k(\BX^*_{\Bl},\BX^*_{\Bl})\big]^{-1}=\Bphi(\BX^*_{\Bl})\big[k(\BX^*_{\Bl},\BX^*_{\Bl})\big]^{-1}.$$ 
Based on the this identity, we can use the following theorem to immediately get the  result:  
\begin{theorem}
$\Bphi(\BX^*_{\Bl})$ equals $R^{-1}_{\Bl}k(\BX^*_{\Bl},\BX^*_{\Bl})$ on rows associated to $\BX^*_{\Bl}$ and equals 0 on any other row.
\end{theorem}
\proof
Remind that a level-$l$ sparse grid is defined as $\BX^{\rm SG}_l=\{\Bx_{\Bl,\Bi}:|\Bl|\leq l,\Bi\in\Brho(\Bl)\}$,  we can label hierarchical feature by its associated design point $\Bx_{\Bl',\Bi}$ in the sparse grid $\BX^{\rm SG}_l$: $[\phi_{\Bl',\Bi}]_{\Bx_{\Bl',\Bi}\in\BX^{\rm SG}_l}$. \cite{Ding2020sample}[Lemma EC.1] states that
\[\phi_{\Bl',\Bi}=\prod_{j=1}^d\phi_{\ell_j',i_j}\]
where, for $j=1,\cdots,d$, $\phi_{\ell'_j,i_j}$ is the one-dimensional hierarchical feature associated to point $x_{\ell'_j,i_j}$ in one-dimensional set $\bold{D}_{\ell'_j}$ as we defined in Lemma \ref{lem:compact_supp} and Theorem \ref{thm:orthor_basis}.

Similarly, we can use Kronecker product technique to derive that
\begin{align*}
    R^{-1}_{\Bl}k(\BX^*_{\Bl},\cdot)&=\bigg[\bigotimes_{j=1}^dR^{-1}_{l_j}\bigg]\bigg[\bigotimes_{j=1}^d k_j(\BX^*_{\ell_j},\BX^*_{\ell_j})\bigg]=\bigotimes_{j=1}^d\big[\phi_{\ell'_j,i_j}\big]_{x_{\ell'_j,i_j}\in\BX^*_{\ell_j}}=\big[\phi_{\boldsymbol{\ell}',\bold{i}}\big]_{\Bx_{\boldsymbol{\ell}',\Bi}\in\BX^*_{\Bl}}.
\end{align*}

We can immediately see that $\big[\phi_{\boldsymbol{\ell}',\bold{i}}\big]_{\Bx_{\boldsymbol{\ell}',\Bi}\in\BX^*_{\Bl}}$ is a sub-vector of  $[\phi_{\Bl',\Bi}]_{\Bx_{\Bl',\Bi}\in\BX^{\rm SG}_l}$ with rows associated to $\BX^*_{\Bl}$.  As a result, it is obvious that $\big[\phi_{\boldsymbol{\ell}',\bold{i}}(\BX^*_{\Bl})\big]_{\Bx_{\boldsymbol{\ell}',\Bi}\in\BX^*_{\Bl}}$ equals  $[\phi_{\Bl',\Bi}(\BX^*_{\Bl})]_{\Bx_{\Bl',\Bi}\in\BX^{\rm SG}_l}$ on rows associated to  $\BX^*_{\Bl}$ because these two vectors of functions have same entries of functions on these rows.

Let $\bold{D}=\BX^{\rm SG}_l-\BX^*_{\Bl}$ denote the difference set. The only thing left is to prove  $[\phi_{\Bl',\Bi}(\BX^*_{\Bl})]_{\Bx_{\Bl',\Bi}\in\bold{D}}$ equals zero. For any point $\Bx_{\Bl',\Bi}\in \bold{D}$, there muse be some dimension $j$ such that $\ell_j'>\ell_j$ for  $\Bx_{\Bl',\Bi}\in \BX^*_{\Bl}$ otherwise. From our discussion in Theorem \ref{thm:orthor_basis}, we know that the support of $\phi_{\ell_j',i_j}$ is $[i_j2^{-\ell'}-2^{-\ell'},i_j2^{-\ell'}+2^{-\ell'}]$ where $i_j\in\Brho(\ell'_j)$ is odd. As a result, for any point $\Bx_{\Bl'',\Bi''}\in\BX^*_{\Bl}$, its $j^{\rm th}$ entry $x_{\ell''_j,i_j''}$ must satisfy $\ell''_j<\ell'_j$ and $i_j''\in\Brho(\ell''_j)$ is odd. Therefore, it is straightforward to derive that
\[x_{\ell'_j,i_j''}=i_j''2^{-\ell''}\not\in (i_j2^{-\ell'}-2^{-\ell'},i_j2^{-\ell'}+2^{-\ell'})\quad \text{for any $i_j\in\Brho(\ell'_j)$ odd}.\]
This leads to $\phi_{\ell'_j,i_j}(x_{\ell_j'',i''_j})=0$ and, hence, $\phi_{\Bl',\Bi}(\Bx_{\ell'',\Bi''})=0$.
\endproof

\section{Proof of Theorem 1:}\label{sec:thm1_pf}
\proof
We first count the number of nuo-zero entries on $R^{-1}_l$. From Appendix B, Algorithm 2, we can see that $R^{-1}_l$ is the summation of matrices of the form $R^{-1}_{\Bl}=\bigotimes_{j=1}^d R^{-1}_{l_j}$. So we first count the number of non-zero entries on these matrices. Because,  $R^{-1}_{l_j}$ is a $(2^{l_j}-1)$-by-$(2^{l_j}-1)$ sparse matrix with $\CalO(2^{l_j})$ non-zero entries as shown in Appendix B-Algorithm 1, the total number of non-zero entries on $R^{-1}_{\Bl}$ is $\CalO(2^{|\Bl|})$. Then the total number of non-zero entries of $R^{-1}_l$ in Appendix B, Algorithm 2 should be bounded by the following sum:
\begin{align*}
    \sum_{l\leq|\Bl|\leq l+d-1}\CalO(2^{|\Bl|})=\CalO(\sum_{\ell=0}^{l-1}2^\ell \binom{\ell+d-1}{d-1})=\CalO(|\BX^{\rm SG}_l|)=\CalO(m_l),
\end{align*}
where the last equality is exactly equation (13) in Appendix.

We now count the time complexity of Appendix B, Algorithm 2. In each iteration, the time complexity for constructing $R^{-1}_{\Bl}$ is
\[\big(\sum_{j=1}^d \CalO(2^{l_j})\big)+ \CalO(2^{|\Bl|})=\CalO(2^{|\Bl|}),\]
where the first term on the left-hand side is the time complexity for constructing each component $R^{-1}_{l_j}$, $j=1,\cdots,d$, in the Kronecker product and the second term is the time complexity for computing the Kronecker product of upper triangular matrices. 

Similar to our counting for its non-zero entries, the total time complexity is then
\begin{align*}
    \CalO(\sum_{l\leq|\Bl|\leq l+d-1}2^{|\Bl|})=\CalO(m_l).
\end{align*}
\endproof

\section{Proof of Theorem 2:}\label{sec:thm2_pf}
\begin{proof} To prove this theorem, we need two more algorithms for computing hierarchical features. The first algorithm is developed from Appendix B-Algorithm 1:
\begin{algorithm}[ht] 
\SetAlgoLined
\DontPrintSemicolon
\SetKwInOut{Input}{Input}
\SetKwInOut{Output}{Output}
\Input{Markov kernel $k$,  level-$l$ dyadic points  $\BX_l^*$, and input $x^*$ }
\Output{$[R^{T}]^{-1}k(\BX_\Bl^*,x^*)$}
\BlankLine
Initialize $\Bphi \leftarrow \bold{0} \in\Real^{(2^l-1)}$, $\bold{N}=\{-\infty,\infty\}$, define $k(\pm\infty,x)=0,\forall x$\;
\For{$\ell \leftarrow 1$ \KwTo $l$}{
    \For{$i \in\Brho(\ell) $}{
        search the closest left neighbor $x_{left}$ and right neighbor $x_{right}$ of $x_{\ell,i}$ in $\bold{N}$\;
        Solve $c_1,c_2$, and $c_3$ associated to $x_{left}$, $x_{\ell,i}$ ,and $x_{right}$ in equation (14) in Appendix B\;
        Update $[\Bphi]_{x_{\ell,i}}$ via
        \vspace{-2ex}
        \begin{equation} \tag{s.1}
            \label{eq:update-feature}
            [\Bphi]_{x_{\ell,i}}=c_1k(x_{left},x^*)+c_2k(x_{\ell,i},x^*)+c_3k(x_{right},x^*)
        \end{equation}\;
        \vspace{-6ex}
        $x_{\ell,i}\rightarrow\bold{N}$
    }
}
Return $[R^{T}]^{-1}k(\BX_\Bl^*,x^*)=\Bphi$ \;
\caption{Computing $[R^{T}]^{-1}k(\BX_\Bl^*,x^*)$ for Markov kernel $k$ and point set $\BX_l^*$}\label{alg:compute-phi-1D}
\end{algorithm}

We first count the number of non-zero entries on the output of Algorithm \ref{alg:compute-phi-1D}. Given any $x^*$, we first fix a $\ell$ in the outer for-loop. From Lemma \ref{lem:compact_supp}, we can derive that  the only non-zero $[\Bphi]_{x_{\ell,i}}$ is the $({\ell,i})$ satisfying the requirement $x_{left}<x^*<x_{right}$. 

Because of the way we sort $\BX^*_l=[\BD_1 \BD_2 \cdots \BD_l]$, all the points $\{x_{\ell,i},i\in\Brho(\ell)\}$ in the inner for-loop are in $\BD_\ell$ and none of their closet neighbor is in $\BD_\ell$. As a result, for any fixed $\ell$, there is only one $(\ell,i)$ satisfies $x_{left}<x^*<x_{right}$ in the inner-loop and, hence, there is only one non-zero entry is added to $\Bphi$ in Algorithm \ref{alg:compute-phi-1D} in each otter iteration. This shows that the number of non-zero on the output of Algorithm \ref{alg:compute-phi-1D} is $l$.

\begin{algorithm}[ht]
\SetAlgoLined
\DontPrintSemicolon
\SetKwInOut{Input}{Input}
\SetKwInOut{Output}{Output}
\Input{TMK $k$,  level-$l$ SG  $\BX_l^{\rm{SG}}$, input $\Bx^*$ }
\Output{$\Bphi(\Bx^*)$}
\BlankLine
Initialize $\Bphi(\Bx^*)\leftarrow \bold{0} \in\Real^{ m_l}$\;
\For{all $\Bl\in\NatInt^d$ with $ l\leq |\Bl| \leq l+d-1$ }{
Compute $\{[R^{T}]^{-1}_{l_j}k_j(\BX^*_{l_j},x_j^*)\}_{j=1}^d$ via Algorithm \ref{alg:compute-phi-1D} \;
Compute $\Bphi_{\Bl}(\Bx^*)=[R^{T}]^{-1}_{\Bl}k(\BX^*_{\Bl},\Bx^*)$ via
\vspace{-2ex}
\begin{equation}\tag{s.2}
    \label{eq:tensor-phi}
    \Bphi_{\Bl}(\Bx^*)=\bigotimes_{j=1}^{d} [R^{T}]^{-1}_{l_j}k_j(\BX^*_{l_j},x_j^*).
\end{equation}\;
\vspace{-6ex}
Update $\Bphi(\Bx^*)$ via
\vspace{-2ex}
\begin{align}\tag{s.3}
[\Bphi(\Bx^*)]_{\BX^*_{\Bl}} \leftarrow{}& [\Bphi(\Bx^*)]_{\BX^*_{\Bl}} + (-1)^{l+d-1-|{\Bl}|} \binom{d-1}{l+d-1-|{\Bl}|} \Bphi_{\Bl}(\Bx^*) \label{eq:update-Rl}
\end{align}
\vspace{-6ex}
}
Return $\Bphi(\Bx^*)$ \;
\caption{Computing $\Bphi(\Bx^*)$ for TMK $k$, level-$l$ SG $\BX^{\rm SG}_l$ and input $\Bx^*$\label{alg:compute-phi}}
\end{algorithm}
We then propose an algorithm to compute $\Bphi(\Bx^*)=R_l^{-T}k(\BX^{\rm SG}_l,\Bx^*)$ as shown in Algorithm \ref{alg:compute-phi}. Algorithm \ref{alg:compute-phi} is a simple extension of Appendix B-Algorithm 2. Based on this algorithm, we can estimate the number of non-zero entries on $\Bphi(\Bx^*)$.

We first estimate the number of non-zero entries for $\Bphi_{\Bl}(\Bx^*)$ in each iteration. 
Because the number of non-zero entries on each $[R^{T}]^{-1}_{l_j}k_j(\BX^*_{l_j},x_j^*)$ is $l_j$, the total number of non-zero entries on  tensor $\Bphi_{\Bl}(\Bx^*)$ is $\CalO(\prod_{j=1}^dl_j)$.

Because only a constant amount of entries are flipped from zero to non-zero in each iteration in Algorithm \ref{alg:compute-phi}, the total number of non-zero entries of $\Bphi(\Bx^*)$ is 
\[\CalO(\sum_{l\leq |\Bl|\leq l+d-1}\prod_{j=1}^dl_j) =\CalO(\sum_{\ell=0}^{l-1} \ell^d \binom{\ell+d-1}{d-1})=\CalO(l^{2d-1})=\CalO([\log m_l]^{2d-1}).\]
\end{proof}

\section{Approximation Error of Hierarchical Expansion}\label{sec:HE_approx}
Compared with the original DGP proposed in \cite{Damianou13}, every output unit in a DTMGP is, in fact, the finite-rank hierarchical expansion of a tensor Markov GP given that output unit is linear combination of finitely many random hierarchical features. Therefore, we first analyze how well hierarchical expansion can approximate a tensor Markov GP. This leads to the following theorem:

\begin{theorem}\label{thm:convergence}\label{sec:DTMGP_approx}
Let $\CalG$ be a TMGP with TMK $k$. Given any level-$l$ SG $\BX^{\rm SG}_l$, we have the following $L^2$ distance between $\CalG$ and its hierarchical expansion with inducing variables $\BX^{\rm SG}_l$
\[\max_{\Bx\in (0,1)^D}\mathbb{E}\bigl[\bigl|\CalG(\Bx)-\Bphi^T(\Bx)\BZ\bigr|^2\bigr]=\CalO\left(\frac{l^{2d-2}}{2^l}\right)=\CalO\left(\frac{[\log m_l]^{2d-2}}{m_l}\right),\]
where $\BZ=[R_l^{T}]^{-1}\CalG(\BX^{\rm SG}_l)$ is a vector of i.i.d. standard Gaussian random variables and  $\CalG(\BX^{\rm SG}_l)$ are samples of $\CalG$ on $\BX^{\rm SG}_l$.
\end{theorem}
\begin{proof}[Proof of Theorem \ref{thm:convergence}]
We prove by induction on dimension $D$. For the base case $D=1$, notice that sparse grid $\BX^{\rm SG}_{l}$ is then reduced to a dyadic grid $\{2^{-l},2\cdot2^{-l},\cdot,1-2^{-l}\}$ and TMGP $\CalG$ is a Markov process. From the Markov property of $\CalG$, we have
\begin{align*}
    \max_{x\in (0,1)}\mathbb{E}\bigl[\bigl|\CalG(x)-\Bphi^T(x)\BZ\bigr|^2\bigr]&=\max_{x\in (0,1)}\mathbb{E}\bigl[\bigl|\CalG(x)-k(x,\BX^{\rm SG}_l)\bold{K}^{-1}\CalG(\BX^{\rm SG}_l)\bigr|^2\bigr]\\
    &=\max_{i=1,\cdots,2^{l}}\max_{x\in\big((i-1)2^{-l},i2^{-l}\big)}\mathbb{E}\bigl[\bigl|\CalG(x)-k(x,\BX_i)\bold{K}_i^{-1}\CalG(\BX_i)\bigr|^2\bigr]
\end{align*}
where $\bold{K}=k(\BX^{\rm SG}_l,\BX^{\rm SG}_l)$, $\BX_i=\{(i-1)2^{-l},i2^{-l}\}$ for $1<i<2^l$, $\BX_1=\{2^{-l}\}$,  $\BX_{2^{l}}=\{1-2^{-l}\}$ , and $\bold{K}_i=k(\BX_i,\BX_i)$. The last line of the above equations is from the following reasoning. $\CalG(x)-k(x,\BX^{\rm SG}_l)\bold{K}^{-1}\CalG(\BX^{\rm SG}_l)$ can be treated as a Markov GP which equals 0 on $\BX^{\rm SG}_l$ and, therefore,  the distribution of $\CalG(x)-k(x,\BX^{\rm SG}_l)\bold{K}^{-1}\CalG(\BX^{\rm SG}_l)$ is independent of $\CalG(y)-k(y,\BX^{\rm SG}_l)\bold{K}^{-1}\CalG(\BX^{\rm SG}_l)$ for any pair of $x,y$ not sharing the same left and right neighboring points in $\BX^{\rm SG}_l$.

Now for any $i$ and $x\in\big((i-1)2^{-l},i2^{-l}\big)$, the posterior $M_i(x)\coloneqq \CalG(x)-k(x,\BX_i)\bold{K}_i^{-1}\CalG(\BX_i)$ can be treated again as a Markov GP with fixed boundary at $x_{i-1}$ and $x_i$. Because $M_i$ is a Markov GP defined on interval of length $2^{-l}$ with fixed boundary condition, it is straightforward to derive that $\max_{x}\mathbb{E}\big[|M_i(x)|^2\big]= \CalO(2^{-l})$. This holds for any $i$ so we can get when dimension $D=1$, 
\[ \max_{x\in (0,1)}\mathbb{E}\bigl[\bigl|\CalG(x)-\Bphi^T(x)\BZ\bigr|^2\bigr]= \max_{i=1,\cdots,2^{l}}\max_{x\in\mathcal{I}_i}\mathbb{E}\bigl[\bigl|\CalG(x)-k(x,\BX_i)\bold{K}_i^{-1}\CalG(\BX_i)\bigr|^2\bigr]=\CalO(2^{-l}).\]
where $\mathcal{I}_i$ denote the inverval $\big((i-1)2^{-l},i2^{-l}\big)$. 

Now suppose our theorem holds for $D=d-1$. We now prove  the case $D=d$. Let $\prod_{j=1}^dk_{j}$ denote TMK $k(\Bx,\Bx)$. For any $\Bx$, the $L^2$ distance has the following identity
\begin{align*}
    &\quad\mathbb{E}\bigl[\bigl|\CalG(\Bx)-\Bphi^T(\Bx)\BZ\bigr|^2\bigr]\\
    &=\mathbb{E}\bigl[\bigl|\CalG(\Bx)-k(\Bx,\BX^{\rm SG}_l)\bold{K}^{-1}\CalG(\BX^{\rm SG}_l)\bigr|^2\bigr]\\
    &=k(\Bx,\Bx)-k(\Bx,\BX^{\rm SG}_l)\bold{K}^{-1}k(\BX^{\rm SG}_l,\Bx)\\
    &=\prod_{j=1}^dk_j-\Bphi^T(\Bx)\Bphi(\Bx).
\end{align*}
 Let   $\hat{k}^{(l)}_j$ denote $k_j(x,\BX^*_{l_j})[k_j(\BX^*_{l_j},\BX^*_{l_j})]^{-1}k(\BX^*_{l_j},x)$ where $\BX^*_{l_j}$ is the dyadic point set defined in the proof of theorem 2. As a direct result of Algorithm \ref{alg:compute-phi}, we have
\begin{align*}
    &\quad\Bphi^T(\Bx)\Bphi(\Bx)\\
    &=\sum_{|\Bl|=l}^{l+d-1}(-1)^{l+d-1-|\Bl|} \binom{d-1}{l+d-1-|{\Bl}|}\Bphi^T_{\Bl}(\Bx)\Bphi_{\Bl}(\Bx)\\
    &=\sum_{|\Bl|=l}^{l+d-1}(-1)^{l+d-1-|\Bl|} \binom{d-1}{l+d-1-|{\Bl}|}\bigg[\bigotimes_{j=1}^{d} [R^{T}]^{-1}_{l_j}k_j(\BX^*_{l_j},x_j^*)\bigg]^T\bigg[\bigotimes_{j=1}^{d} [R^{T}]^{-1}_{l_j}k_j(\BX^*_{l_j},x_j^*)\bigg]\\
    &=\sum_{|\Bl|=l}^{l+d-1}(-1)^{l+d-1-|\Bl|} \binom{d-1}{l+d-1-|{\Bl}|}\prod_{j=1}^d\hat{k}^{(l_j)}_j.
\end{align*}
The last line of the above equations can be derive by induction on the following identity of Kronecker product:
\begin{align*}
    \big[\big(\bold{M}_2v_2\big)\bigotimes\big(\bold{M}_1v_1\big)\big]^T\big[\big(\bold{M}_2v_2\big)\bigotimes\big(\bold{M}_1v_1\big)\big]=\bigg(v_2^T\bold{M}_2^T\bold{M}_2v_2\bigg)\bigg(v_1^T\bold{M}_1^T\bold{M}_1v_1\bigg)
\end{align*}
where $\bold{M}_i$ is any $m$-by-$m$ matrix and $v_i$ is any $m$-by-1 vector for $i=1,2$. 

Let $\Delta_j^{(l)}$ denote $\hat{k}^{(l)}_j-\hat{k}^{(l-1)}_j$ and $\Bl_{-d}$ denote the vector $(l_1,\cdots,l_{d-1})$. We have the following identity:
\begin{align*}
    &\quad \sum_{|\Bl|=l}^{l+d-1}(-1)^{l+d-1-|\Bl|} \binom{d-1}{l+d-1-|{\Bl}|}\prod_{j=1}^d\hat{k}^{(l_j)}_j\\
    &=\sum_{|\Bl|\leq l+d-1}\prod_{j=1}^d\Delta^{(l_j)}_j\\
    &=\sum_{i=d}^{l+d-1}\sum_{|\Bl|=i}\prod_{j=1}^d\Delta^{(l_j)}_j\\
    &=\sum_{i=d}^{l+d-1}\sum_{|\Bl_{-d}|=d-1}^{i-1}(\hat{k}^{(i-|\Bl_{-d}|)}_d-\hat{k}^{(i-1-|\Bl_{-d}|)}_d)\prod_{j=1}^{d-1}\Delta^{(l_j)}_j\\
    &=\sum_{|\Bl_{-d}|=d-1}^{l+d-2}\hat{k}^{(l+d-1-|\Bl_{-d}|)}_d\prod_{j=1}^{d-1}\Delta^{(l_j)}_j.
\end{align*}
where the last line is from telescoping sum. Hence, the $L^2$ distance becomes
\begin{align}
    &\quad \prod_{j=1}^dk_j-\sum_{|\Bl|=l}^{l+d-1}(-1)^{l+d-1-|\Bl|} \binom{d-1}{l+d-1-|{\Bl}|}\prod_{j=1}^d\hat{k}^{(l_j)}_j\nonumber\\
    &=\prod_{j=1}^dk_j-\sum_{|\Bl_{-d}|=d-1}^{l+d-2}\hat{k}^{(l+d-1-|\Bl_{-d}|)}_d\prod_{j=1}^{d-1}\Delta^{(l_j)}_j\nonumber\\
    &=k_d\big(\prod_{j=1}^{d-1}k_j-\sum_{|\Bl_{-d}|=d-1}^{l+d-2}\prod_{j=1}^{d-1}\Delta^{(l_j)}_j\big)+\sum_{|\Bl_{-d}|=d-1}^{l+d-2}\big(k_d-\hat{k}^{(l+d-1-|\Bl_{-d}|)}_d\big)\prod_{j=1}^{d-1}\Delta^{(l_j)}_j\nonumber\\
    &=\CalO(\frac{l^{2(d-2)}}{2^l})+\CalO\big(\sum_{|\Bl_{-d}|=d-1}^{l+d-2} 2^{-l} \big)\label{eq:TMGP_rate_l}\tag{s.4}
\end{align}
where the last line is from our induction hypothesis. From direct calculation, we then get
\begin{align}\label{eq:TMGP_rate_sum_l}\tag{s.5}
    \sum_{|\Bl_{-d}|=d-1}^{l+d-2} 2^{-l}&=2^{-l}\sum_{i=d-1}^{l+d-2}\binom{i-2}{d-2}=\CalO\big(\frac{l^{2(d-1)}}{2^l}\big).
\end{align}
Lemma~3.6 in \cite{bungartz_griebel_2004} stipulates that the sample size of $\BX_{l}^{\rm{SG}}$ is given by
\begin{equation}\label{eq:SG_size}\tag{s.6}
    m_l=\bigg|{\BX_{l}^{\rm {SG}}}\bigg|=\sum_{\ell=0}^{l-1}2^\ell \binom{\ell+d-1}{d-1} = 2^l\cdot\left(\frac{l^{d-1}}{(d-1)!}+\CalO(l^{d-2})\right)=\CalO( 2^l l^{d-1}).
\end{equation}
Putting \eqref{eq:TMGP_rate_l},\eqref{eq:TMGP_rate_sum_l},\eqref{eq:SG_size}, we can get the final result:
\[ \prod_{j=1}^dk_j-\sum_{|\Bl|=l}^{l+d-1}(-1)^{l+d-1-|\Bl|} \binom{d-1}{l+d-1-|{\Bl}|}\prod_{j=1}^d\hat{k}^{(l_j)}_j=\CalO\big(\frac{l^{2(d-1)}}{2^l}\big)=\CalO\left(\frac{[\log m_l]^{2d-2}}{m_l}\right).\]

\end{proof}

\section{Approximation Error of DTMGP}
The sparse property of hierarchical features in each layer yields efficient training and inference of DTMGP. Furthermore, hierarchical features are nearly the optimal choice  of features in recovering the underlying TMGP as shown in \cite{Ding20}. The sparsity and optimality of $\{\Bphi^{(h)}\}_{h=1}^H$ leads to the following  theorem regarding the approximation capacity of DTMGP:
\begin{theorem}
\label{thm:DTMGP-convergence}
Let $f^{(H)}:\Real^{W^{(0)}}\to\Real^{W^{(H)}}$ be a layer-$H$ DGP such that layer $h$ of $f^{(H)}$ consists of a $W^{(h)}$-variate TMGP, for $h=1,\cdots,d$. Let DTMGP $\CalT^{(H)}$ be the hierarchical expansion of $f^{(H)}$ such that $\CalT^{(H)}$ uses $m^{(h)}$ hierarchical features to approximate the $W^{(h)}$-variate TMGP in layer $h$ of $f^{(H)}$. We then have the following $L^2$ distance between $\CalT^{(H)}$ and $f^{(H)}$:
\begin{align*}
    & \quad  \max_{\Bx}\bigg(\E\big[\big\|f^{(H)}(\Bx)-\CalT^{(H)}(\Bx)\big\|^2\big]\bigg)^{\frac{1}{2}}\\
    &\leq C\sum_{h=1}^H \bigg({\frac{W^{(h)}(\log m^{(h)})^{2W^{(h-1)}-2}}{m^{(h)}}}\bigg)^{\frac{1}{2^{H-h+1}}}\prod_{\gamma=h+1}^{H}\bigg({W^{(\gamma)}}\bigg)^{\frac{1}{2^{H-\gamma+1}}},
\end{align*}
where $C$ is some constant independent of $\{m^{(h)}\}_{h=1}^H$.
\end{theorem}
\begin{proof}[Proof of Theorem \ref{thm:DTMGP-convergence}]
Let $\CalG^{(h)}$ denote the TMGP in layer $h$ of $f^{(H)}$ and let $\Tilde{\CalG}^{(h)}$ denote the hierarchical expansion of $\CalG^{(h)}$. We prove the statement by induction in the number of layers $H$. When $H=1$, the statement is obviously true because it is equivalent to Theorem \ref{thm:convergence}. Now suppose the statement holds for the case $H-1$, then for $H$ and any $\Bx$, we have
\begin{align}
    &\quad \bigg(\E\big[\big\|f^{(H)}(\Bx)-\CalT^{(H)}(\Bx)\big\|^2\big]\bigg)^{\frac{1}{2}}\nonumber\\
    &= \bigg(\E\big[\big\|\CalG^{(H)}\circ f^{(H-1)}(\Bx)-\Tilde{\CalG}^{(H)}\circ\CalT^{(H-1)}(\Bx)\big\|^2\big]\bigg)^{\frac{1}{2}}\nonumber\\
    &=\bigg(\E\big[\big\|\CalG^{(H)}\circ f^{(H-1)}(\Bx)-\Tilde{\CalG}^{(H)}\circ f^{(H-1)}(\Bx)+\Tilde{\CalG}^{(H)}\circ f^{(H-1)}(\Bx)-\Tilde{\CalG}^{(H)}\circ\CalT^{(H-1)}(\Bx)\big\|^2\big]\bigg)^{\frac{1}{2}}\nonumber\\
    &\leq \bigg(\E\big[\big\|\CalG^{(H)}\circ f^{(H-1)}(\Bx)-\Tilde{\CalG}^{(H)}\circ f^{(H-1)}(\Bx)\big\|^2\big]\bigg)^{\frac{1}{2}}\label{eq:termA} \tag{s.7}\\
    &\quad +\bigg(\E\big[\big\|\Tilde{\CalG}^{(H)}\circ f^{(H-1)}(\Bx)-\Tilde{\CalG}^{(H)}\circ\CalT^{(H-1)}(\Bx)\big\|^2\big]\bigg)^{\frac{1}{2}}\label{eq:termB}\tag{s.8},
\end{align}
where the fourth line is from triangular inequality. 

According to theorem \ref{thm:convergence}, we can directly derive the following upper bound of \eqref{eq:termA}
\begin{equation}\tag{s.9}
    \label{eq:upper-bound-A}
    \bigg(\E\big[\big\|\CalG^{(H)}\circ f^{(H-1)}(\Bx)-\Tilde{\CalG}^{(H)}\circ f^{(H-1)}(\Bx)\big\|^2\big]\bigg)^{\frac{1}{2}}\leq C \sqrt{\frac{W^{(H)}(\log m^{(H)})^{2W^{(H-1)}-2}}{m^{(H)}}},
\end{equation}
for some $C$ independent of $m^{(H)}$ and $W^{(H)}$.

Now we estimate the upper bound of \eqref{eq:termB}. Because hierarchical expansion is an induced approximation, it does not change the H\"older condition of the Gaussian process, we can have:
\[\E\big[\big\|\Tilde{\CalG}^{(H)}\circ f^{(H-1)}(\Bx)-\Tilde{\CalG}^{(H)}\circ\CalT^{(H-1)}(\Bx)\big\|^2\big]\leq C' \E\big[\big\|{\CalG}^{(H)}\circ f^{(H-1)}(\Bx)-{\CalG}^{(H)}\circ\CalT^{(H-1)}(\Bx)\big\|^2\big],\]
for some $C'$ independent of $f^{(H-1)}$ and $\CalT^{(H-1)}$. Now we estimate the H\"older condition of $\CalG$ along all dimensions of its input $\Bf=(f_1,\cdots,f_{W^{(h-1)}})$:
\begin{align}
    &\quad \frac{\E\big[\big\|{\CalG}^{(H)}(\Bf)-{\CalG}^{(H)}(\Bf +h\bold{e}_j)\big\|^2\big]}{h}\nonumber\\
    &=W^{(H)}\frac{k(f_j,f_j)+k(f_j+h,f_j+h)-2k(f_j,f_j+h)}{h}\nonumber\\
    &=W^{(H)}\frac{p(f_j)q(f_j)+p(f_j)q(f_j+h)-2p(f_j)q(f_j+h)}{h}\nonumber\\
    &=C''W^{(H)},\quad \text{for}\ j=1,\cdots,W^{(h-1)} \label{eq:TMK-holder}\tag{s.10},
\end{align}
where $C''$ is some constant independent of $\Bf$, $h$ is any small positive constant, $\bold{e_j}$ is a zero vector with the $j^{\rm th}$ entry equals $1$, $k$ is the TMK of $\CalG^{(H)}$, and the third line is from definition of TMK $k$, the last line is from Lemma 1 in the main paper.

Then,  we have the following estimate of \eqref{eq:termB}:
\begin{align}
&\quad \bigg(\E\big[\big\|\Tilde{\CalG}^{(H)}\circ f^{(H-1)}-\Tilde{\CalG}^{(H)}\circ\CalT^{(H-1)}\big\|^2\big]\bigg)^{\frac{1}{2}}\nonumber\\
&\leq C'''\sqrt{W^{(H)}\bigg(\E\|f^{(H-1)}-\CalT^{(H-1)}\big\|^2\bigg)^{\frac{1}{2}}}\nonumber\\
&\leq C'''' \sqrt{W^{(H)}\sum_{h=1}^{H-1} \bigg({\frac{W^{(h)}(\log m^{(h)})^{2W^{(h-1)}-2}}{m^{(h)}}}\bigg)^{\frac{1}{2^{H-h}}}\prod_{\gamma=h+1}^{H-1}\bigg({W^{(\gamma)}}\bigg)^{\frac{1}{2^{H-\gamma}}}}\nonumber\\
&\leq  C'''' \sqrt{W^{(H)}}\sum_{h=1}^{H-1}\sqrt{ \bigg({\frac{W^{(h)}(\log m^{(h)})^{2W^{(h-1)}-2}}{m^{(h)}}}\bigg)^{\frac{1}{2^{H-h}}}\prod_{\gamma=h+1}^{H-1}\bigg({W^{(\gamma)}}\bigg)^{\frac{1}{2^{H-\gamma}}}}\nonumber\\
&=C''''  \sum_{h=1}^{H-1} \bigg({\frac{W^{(h)}(\log m^{(h)})^{2W^{(h-1)}-2}}{m^{(h)}}}\bigg)^{\frac{1}{2^{H-h+1}}}\prod_{\gamma=h+1}^{H}\bigg({W^{(\gamma)}}\bigg)^{\frac{1}{2^{H-\gamma+1}}},\label{eq:upper-bound-B}\tag{s.11}
\end{align}
where  $C'''$ and $C''''$ are some number independent of $\{m^{(h)}\}_{h=1}^{H-1}$, the second line is from the H\"older condition \eqref{eq:TMK-holder}, the third line is from the induction assumption.

Putting equations \eqref{eq:upper-bound-A} and \eqref{eq:upper-bound-B} together, we can have the final result.
\end{proof}

From theorem \eqref{thm:DTMGP-convergence}, we can see that if a DTMGP is not tremendously deep and the number of inducing points dominates the dimension of output from each layer, then the distance between $f^{(H)}$ and $\CalT^{(H)}$ is relatively close. DGP $f^{(H)}$ is, in fact, an infinite-dimensional non-parametric model while DTMGP $\CalT^{(H)}$ is a parametrized model with finitely many parameters. Therefore, theorem \ref{thm:DTMGP-convergence} partially reflects that DTMGP can also well approximate  other complex stochastic systems.

\bibliographystyle{chicago}
\spacingset{1}
\bibliography{IISE-Trans}